\def\DPRMArxivVersion{1}
\PassOptionsToPackage{dvipsnames,svgnames,table}{xcolor}
\documentclass[10pt,journal,compsoc]{IEEEtran}

\usepackage[utf8]{inputenc}
\usepackage[T1]{fontenc}
\usepackage{microtype}
\usepackage{graphicx}
\usepackage{subcaption}
\usepackage{booktabs}
\usepackage{amsmath,amssymb,amsfonts,mathtools,amsthm}
\usepackage{bm}
\usepackage{enumerate}
\usepackage{enumitem}
\usepackage{pifont}
\usepackage{placeins}
\usepackage{multirow}
\usepackage[normalem]{ulem}
\usepackage[most]{tcolorbox}
\usepackage{colortbl}
\usepackage{wrapfig}
\usepackage{float}
\usepackage{stfloats}
\usepackage{xurl}
\usepackage{xcolor}
\usepackage{algorithm}
\usepackage{algorithmic}
\usepackage[numbers,sort&compress]{natbib}
\usepackage[hidelinks]{hyperref}
\usepackage[capitalize,noabbrev]{cleveref}

\graphicspath{{figs/}}
\DeclareGraphicsExtensions{.pdf,.png,.jpg,.jpeg}

\providecommand{\theHalgorithm}{\arabic{algorithm}}
\theoremstyle{plain}
\newtheorem{theorem}{Theorem}[section]
\newtheorem{proposition}[theorem]{Proposition}
\newtheorem{lemma}[theorem]{Lemma}
\newtheorem{corollary}[theorem]{Corollary}
\theoremstyle{definition}

\newtheorem{assumption}[theorem]{Assumption}
\theoremstyle{remark}
\newtheorem{remark}[theorem]{Remark}

\crefname{equation}{}{}
\crefname{figure}{Fig.}{Figs.}
\crefname{section}{Sec.}{Secs.}
\crefname{appendix}{App.}{Apps.}
\crefname{table}{Tab.}{Tabs.}
\crefname{theorem}{Thm.}{Thms.}
\crefname{lemma}{Lem.}{Lems.}
\crefname{assumption}{Assump.}{Assumps.}
\crefname{proposition}{Prop.}{Props.}
\crefname{corollary}{Cor.}{Cors.}
\crefname{definition}{Def.}{Defs.}
\crefname{remark}{Rmk.}{Rmks.}
\crefname{algorithm}{Alg.}{Algs.}

\definecolor{myred}{RGB}{190,30,45}
\definecolor{myblue}{RGB}{28,76,170}
\definecolor{mygreen}{RGB}{20,120,60}
\definecolor{myorange}{RGB}{210,120,20}
\definecolor{mygray}{RGB}{70,70,70}
\definecolor{boxborder}{RGB}{40,55,90}
\definecolor{boxbg}{RGB}{248,248,248}
\definecolor{mypurple}{RGB}{120,55,160}
\definecolor{lightblue}{RGB}{232,241,255}
\definecolor{lightgreen}{RGB}{232,247,238}
\definecolor{lightpurple}{RGB}{242,235,250}
\definecolor{LightCyan}{rgb}{0.88,1,1}

\newcommand{\bbE}{\mathbb{E}}
\newcommand{\KL}{\mathrm{KL}}
\newcommand{\Bias}{\operatorname{Bias}}
\newcommand{\rad}{\operatorname{rad}}
\newcommand{\drift}{\operatorname{drift}}
\newcommand{\calK}{\mathcal K}
\newcommand{\calO}{\mathcal O}
\newcommand{\calE}{\mathcal E}
\newcommand{\calR}{\mathcal R}
\newcommand{\tO}{\widetilde O}
\newcommand{\tOm}{\widetilde \Omega}
\newcommand{\AH}[1]{}
\newcommand{\cmark}{\ding{51}}
\newcommand{\xmark}{\ding{55}}
\newcommand{\bestcell}[1]{\cellcolor{LightCyan}\textbf{#1}}
\newcommand{\omnimatchedrows}{%
Random & 0.17939 & 0.23381 & 1 \\
Omni default & 0.18661 & 0.23836 & 1 \\
Uniform order & 0.18566 & 0.23796 & 5 \\
DPRM & \bestcell{0.21125} & \bestcell{0.24854} & 5 \\
}

\newcommand{\multimodalaggregaterow}{%
Two-host mean & 0.92561 & \bestcell{1.00000} \\
}

\newcommand{\dmpomatchedrows}{%
Progressive DMPO & 50.4 & 44.3 & 53.4 & 29.6 \\
DMPO-DPRM & \bestcell{52.5} & \bestcell{47.9} & \bestcell{55.0} & \bestcell{33.4} \\%
}

\newcommand{\mask}{\mathsf{M}}

\title{DPRM: A Plug-in Token-Ordering Module for Discrete Diffusion Models}

\author{Dake Bu, Wei Huang, Andi Han, Si Wu, Hau-San Wong, Qingfu Zhang, Taiji Suzuki, and Atsushi Nitanda%
\IEEEcompsocitemizethanks{%
\IEEEcompsocthanksitem Dake Bu, Hau-San Wong and Qingfu Zhang are with City University of Hong Kong.
\IEEEcompsocthanksitem Dake Bu and Atsushi Nitanda are with CFAR and IHPC, A*STAR.
\IEEEcompsocthanksitem Wei Huang is with RIKEN AIP and The Institute of Statistical Mathematics.
\IEEEcompsocthanksitem Andi Han is with the University of Sydney.
\IEEEcompsocthanksitem Si Wu is with South China University of Technology.
\IEEEcompsocthanksitem Taiji Suzuki is with The University of Tokyo and RIKEN AIP.
\IEEEcompsocthanksitem Atsushi Nitanda is also with Nanyang Technological University.
\IEEEcompsocthanksitem Corresponding authors: Atsushi Nitanda (atsushi\_nitanda@a-star.edu.sg) and Hau-San Wong (cshswong@cityu.edu.hk).}%
}

\ifdefined\DPRMArxivVersion
\else
\fi

\begin{document}

\IEEEtitleabstractindextext{%
\begin{abstract}
Discrete diffusion models admit many token orders, yet most systems rely on confidence-based decoding. Confidence is a strong and efficient heuristic, but it can be myopic because local certainty does not measure a position's effect on terminal quality. We introduce DPRM (Doob $h$-transform Process Reward Model), a plug-in token-ordering module induced by a terminal-reward-tilted trajectory distribution. The Doob transform converts terminal reward into a process reward for each candidate position. A compact estimator indexed by generation progress, confidence, and optional positional or task state makes this correction reusable across partial states. DPRM changes only token order; the host architecture, denoising objective, supervision, and token sampler remain fixed. We derive the reward-tilted law for deterministic host updates, bound the marginal shortlist approximation to its Gibbs policy, and quantify online score error. Across nine recent open-source hosts spanning language, multimodal generation, and scientific sequences, DPRM improves reasoning, numeric VQA, visual-codebook ordering, and preference-conditioned generation. Relative to confidence, it improves PUMA GSM8K by $14.6\%$, Omni-Diffusion CLIP-L/14 by $13.2\%$, Prism voted accuracy from $82.41\%$ to $83.85\%$, and RealWorldQA numeric/count accuracy by $8.97$ points. Matched scientific gains include $8.3\%$ in DCM nonzero recovery, $15.0\%$ in GenMol QED, and $53.3\%$ in SDPO-DNA total utility. Entropy and random-order controls show that the gains do not come from favoring uncertain tokens indiscriminately. Token-order traces and shared-state interventions instead show how terminal reward can correct a locally confident but globally poor choice.
\end{abstract}
}

\maketitle
\IEEEdisplaynontitleabstractindextext
\IEEEpeerreviewmaketitle

\section{Introduction}
\label{sec:intro}

Autoregressive LLMs have achieved remarkable performance in mathematics, coding, and agentic problem solving \citep{jaech2024openai,anthropic2025claude4,guo2025deepseek,kimiteam2025kimi}, but left-to-right generation is a modeling choice rather than a universal property of the data. Many scientific objects are governed by global constraints and bidirectional dependencies: discrete-diffusion protein models outperform same-scale autoregressive protein LMs on function, localization, interaction, annotation, and foldability tasks \citep{wang2024diffusion}; masked diffusion improves regulatory DNA generation by modeling DNA bidirectionally \citep{yang2026d3lm}; and masked discrete diffusion is a natural fit for sparse, non-sequential single-cell transcriptomes \citep{zhang2026lingshucell}. These results suggest that scientific generation may need global, flexible-order denoising rather than a fixed serial order. The same shift is visible in language and multimodal modeling, where data types and conditioning signals are increasingly combined rather than handled by one fixed sequential interface \citep{zhu2024visionx}. Large discrete diffusion models such as LLaDA \citep{nie2025llada} and Dream \citep{ye2025dream} show that autoregression is no longer the only viable default, while Mercury \citep{inception2025mercury} and Gemini Diffusion \citep{deepmind_gemini_diffusion} suggest that flexible-order generation can also offer practical throughput advantages \citep{ye2025diffusion,zhu2025dmpo}.

This flexibility, however, exposes a new algorithmic question that autoregressive models largely hide: 

\emph{how should a discrete diffusion model order its tokens?} 

Related non-diffusion work reaches a similar conclusion for graphs and molecules, where learned-order autoregressive models treat generation order as a state-dependent policy \citep{wang2025learningorder}; adaptive token optimization is also an explicit efficiency variable in recent vision transformers \citep{li2026prance}. Recent masked-diffusion work likewise shows that token order is not a minor implementation detail. \citet{kim2025train} improve masked diffusion models through adaptive inference order, and \citet{kim2026puma} align training and decoding with confidence-progressive masks. Confidence or entropy ranking is also the native rule in masked image generators and in several recent language, multimodal, and search hosts, including Omni-Diffusion, LLaDA-V, and Prism \citep{chang2022maskgit,li2026omnidiffusion,you2025llada,bai2026prism}. Confidence is cheap, model-internal, and effective, making it the strongest common heuristic across our nine hosts.

Confidence is not, however, an objective for terminal quality. It can make myopic choices, suppress useful exploration, and reveal end-of-text tokens too early when the best completion begins with a less-confident position \citep{fang2026locally,park2026confidence}. Recent learned unmasking policies also treat confidence order as a baseline to improve upon: a small transformer maps token confidence, mask status, time, and position-aware attention to token-order decisions \citep{jazbec2026unmasking}. Work on language reasoning further shows that high-entropy minority tokens can carry disproportionate learning signal \citep{wang2025entropyminority}, and bypassing uncertain tokens can reduce reasoning coverage in diffusion models \citep{ni2026flexibility}. These findings identify where confidence may fail, but they do not imply that the least-confident position should be revealed first. Our entropy and random-order controls confirm this distinction: uncertainty alone does not recover the gains of reward-guided order.


This raises the central question of the paper: 

\emph{can token ordering use task reward information beyond local confidence without changing the host diffusion model?}

We answer yes by deriving token order from terminal reward. Let the host decoder define a base distribution over complete trajectories, and tilt that distribution by \(\exp(\beta R(X_T))\). For deterministic host updates, its Doob \(h\)-transform gives the exact token-order law whose local correction is the conditional log-moment of terminal reward after revealing each candidate position. With stochastic token updates, this correction instead defines an order-policy target; retaining the host token kernel need not realize the full terminal tilt. DPRM starts from confidence order and introduces this correction only when it is supported by observed trajectories. It therefore preserves the host model, loss, data, and token sampler while replacing a heuristic order with an approximation to a well-defined reward-tilted distribution \citep{doob1957conditional,leonard2014a,chen2021stochastic,bu2026distributionalbiasesposttrainingmarkovian}.

The approximation is deliberately lightweight. It shares only the reward correction through cells indexed by generation phase, confidence, and, when needed, position, EOT status, or a task-specific variable; the full partial state still determines the host logits and provisional tokens. Confidence is available without another model call, while phase and position retain structure known to matter for token order \citep{park2026confidence,jazbec2026unmasking}. Counts and readiness gates prevent sparse cells from overriding confidence. This plug-in design is complementary to diffusion self-guidance methods that also control generation without retraining the base model \citep{li2026selfguidance}.

With exact process rewards and fresh i.i.d. candidate multisets at every step, the marginal stagewise Soft-BoN law approaches the exact Gibbs-policy target at rate \(O(T/N)\) in terminal forward KL. The online estimator has an empirical-Bernstein sampling term plus explicit approximation, warmup, and drift terms. The resulting sample-complexity analysis predicts an easy-to-hard curriculum: confidence reduces noise early, while reward guidance later recovers useful orders that confidence rarely visits.

We evaluate nine recent public hosts spanning pretraining, post-training, test-time search, vision-language decoding, visual generation, proteins, cells, molecules, and regulatory DNA. Each comparison keeps the host model, token-value rule, objective, and evaluator fixed. Confidence is the common reference, with the host's native or random order included when relevant. DPRM improves PUMA GSM8K by \(14.6\%\), Omni-Diffusion CLIP-L/14 by \(13.2\%\), Prism voted accuracy from \(82.41\%\) to \(83.85\%\), and RealWorldQA numeric/count accuracy by \(8.97\) percentage points. Scientific gains include \(8.3\%\) in DCM nonzero recovery, \(15.0\%\) in GenMol QED, and \(53.3\%\) in SDPO-DNA total utility.

For seven hosts, training or calibration supplies a reusable table, so evaluation adds only scalar lookup and reranking and makes no additional denoiser calls. Prism and Omni require online values for different, experimentally verified reasons. Prism is a test-time scaling method: its verifier reward exists only after search begins, and there is no training-stage trajectory stream with which to prefill the table. Omni poses a transfer problem instead. A phase/confidence table learned from earlier prompts failed grouped validation and almost always reproduced confidence on new prompts, so we estimate token-order values from continuations of the current canvas. The added NFE and path costs of these two settings are reported explicitly.

We also test why the order changes matter. Complete PUMA traces show that DPRM reveals positions farther apart, returns more often to earlier masks, and commits numbers after more context is available. In Omni, changing one position from the same partial canvas changes the later image. DPLM supplies the opposite control: token order changes, but the final files remain identical. Full protocols, ablations, computational comparisons, uncertainty analyses, and host-level interpretations are given in the Supplementary Material. Code and retained artifacts are available at \href{https://github.com/DakeBU/DPRM-DLLM}{GitHub} and \href{https://huggingface.co/DarkerBu/DPRM-DLLM}{Hugging Face}, respectively.

\noindent\textbf{Prior dissemination and journal extension.}
An earlier version was posted on arXiv and presented orally at the non-archival ICML 2026 FoGen workshop \citep{bu2026dprmworkshop}. This journal version adds the Omni and LLaDA-V studies, including disjoint test sets, controlled order interventions, and the negative AI2D result. It also adds full trace analyses, controller ablations, cost and uncertainty reporting, scientific replications, and the DPLM order-insensitivity test. The Supplementary Material documents the new experiments and their settings host by host.


\section{DPRM as a Plug-in Ordering Module}
\label{sec:framework}
\begin{figure*}[ht]
    \centering
    \includegraphics[width=\linewidth]{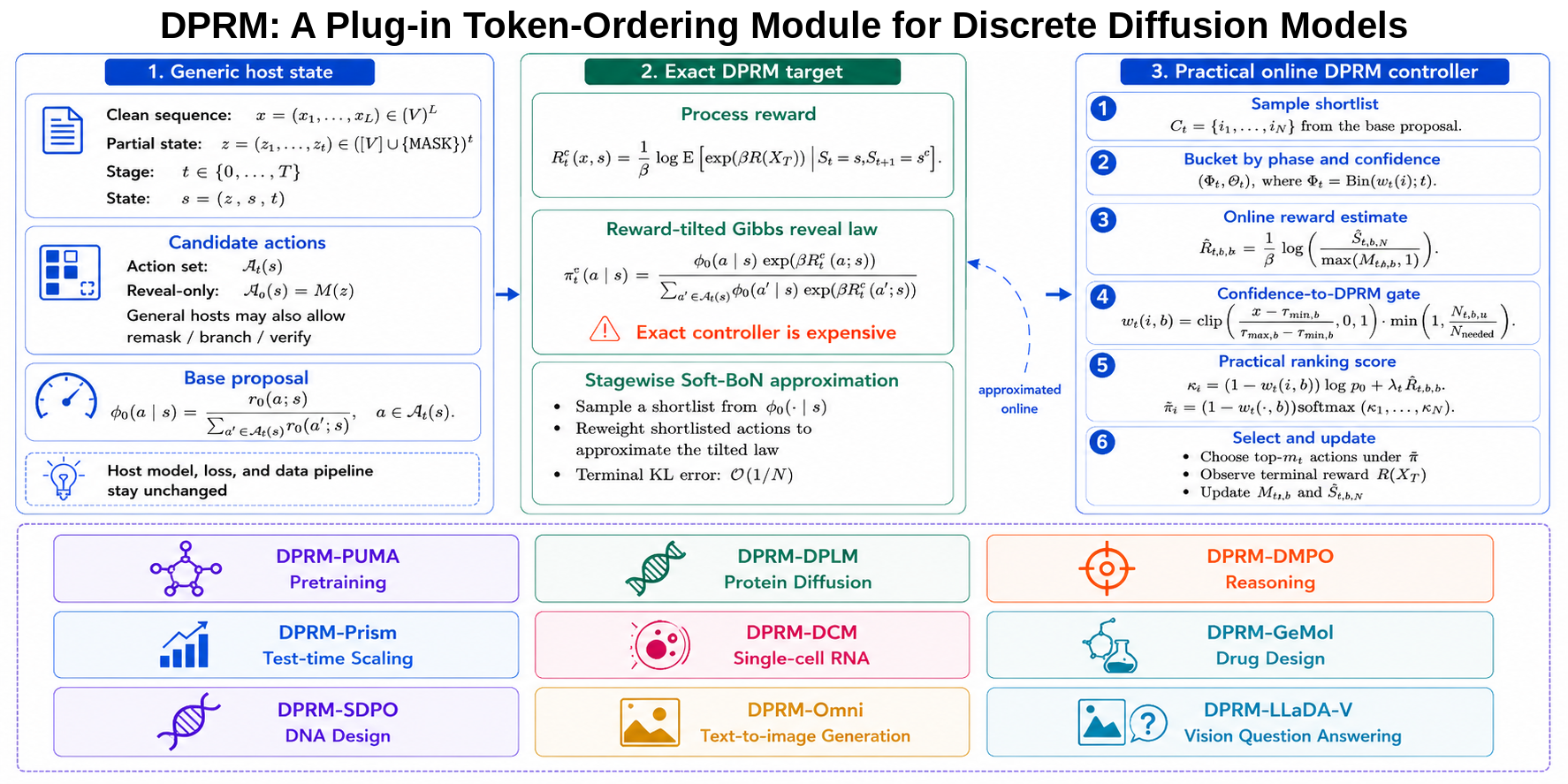}
    \caption{DPRM as a plug-in token-ordering module. The host supplies candidate positions and provisional token values; DPRM changes only their order. Confidence initializes the proposal, and a process-reward estimate reranks either a sampled shortlist or the full candidate set.}
    \label{fig:DPRM}
\end{figure*}
\subsection{A Generic View of Token Ordering in Diffusion Models}
\label{sec:framework_generic}

Let a clean token sequence be \(x=(x_1,\dots,x_L)\in[V]^L\), where \(V\) is the vocabulary and \(L\) is the sequence length; in language settings, \(x\) may be a prompt--response pair. Let \(c\) denote any fixed conditioning input, such as a prompt or image. A partially observed token array is
\[
z=(z_1,\dots,z_L)\in([V]\cup\{\texttt{[MASK]}\})^L,
\]
where each coordinate is either a visible token or the mask symbol. We use \(t\in\{0,\dots,T\}\) for the denoising step and write the available host state as \(s=(c,z,t)\).

Many diffusion algorithms share the same ordering problem. At state \(s\), let \(\mathcal I(s)\) be the candidate items that may be updated next. In a reveal-only host these are the masked positions; hosts with remasking, branching, or verification may include visible positions or partial hypotheses. Selecting item \(i\in\mathcal I(s)\) triggers the host-specific update at that position, such as reveal, remask, branch, or verify.

The generic computation has three steps:
\begin{enumerate}[leftmargin=18pt]
    \item compute a base proposal score \(\psi_i(s)\) for each \(i\in\mathcal I(s)\); 
    \item rank candidate items using an ordering score;
    \item choose an update budget \(m(s)\) and apply the host update to the top-\(m(s)\) items.
\end{enumerate}

Different hosts differ mainly in the base proposal \(\psi_i(s)\) and the update attached to each item \(i\). Random masking uses a uniform proposal. Random decoding corresponds to \(\psi_i(\cdot)=1\). Confidence-based decoding uses \(\psi_i(s)=p_i(s)\), where \(p_i(s)\) is the model's current confidence at item \(i\). In reveal-only settings, we later write \(M(z)\subseteq[L]\) for the masked coordinates and recover the simpler notation \(\mathcal I(s)=M(z)\).

We now introduce DPRM as a generic plug-in controller for this template.

\subsection{Progressive online DPRM}
\label{sec:framework_schedule}

We propose \textbf{DPRM} (\textbf{D}oob \(h\)-transform \textbf{P}rocess \textbf{R}eward \textbf{M}odel), a token-ordering module induced by a reward-tilted trajectory law.
The name comes from the following construction.
Let \(q_0\) denote the transition law induced by the host's base controller. Under \(q_0\), let \((S_0,S_1,\dots,S_T)\) be the host Markov chain, where \(S_t\) is the partially observed state at step \(t\), and let \(X_T\) be the completed sample.
For a terminal reward \(R(X_T)\) and inverse temperature \(\beta>0\), define the positive space--time harmonic function
\[
h_t(s)
:=
\mathbb E_{q_0}\!\left[\exp(\beta R(X_T))\mid S_t=s\right],
\]
where \(h_t(s)\) is the future exponentiated reward expected from state \(s\). The classical Doob \(h\)-transform \citep{doob1957conditional,leonard2014a,chen2021stochastic} tilts the base transition from \(s\) to the next state by the ratio \(h_{t+1}/h_t\).
Thus high-reward future continuations receive larger transition probability.
DPRM applies this idea to token-ordering: it changes the local ordering policy while leaving the host architecture, denoising objective, and data pipeline unchanged.

For a current state \(s\), let \(\mathcal I(s)\) be the candidate positions in the next token-order decision.
Write \(K_t(ds\prime\mid s,i)\) for the host state-update kernel after selecting position \(i\in\mathcal I(s)\). For deterministic updates, denote the next state by \(s^i\); this includes teacher forcing conditional on a fixed example and deterministic token decoding.
The normalized base proposal is
\begin{equation}
q_0(i\mid s):=\frac{\psi(i;s)}{\sum_{j\in\mathcal I(s)}\psi(j;s)},
\qquad i\in\mathcal I(s),
\label{eq:local_base_proposal}
\end{equation}
where \(\psi(i;s)\) is the host's incumbent local score.
The full Doob target reweights the base trajectory law by \(\exp(\beta R(X_T))\); larger \(\beta\) gives terminal reward more influence. We condition on a fixed initial state unless stated otherwise. For a random initial law \(\rho_0\), the globally tilted law starts from \(\rho_\beta(ds)=\rho_0(ds)h_0(s)/Z_\beta\), where \(Z_\beta=\mathbb E_{\rho_0}h_0(S_0)\).

The exact DPRM process reward is the log-moment future reward after selecting position \(i\):
\begin{equation}
R_t^\star(i;s)
:=
\frac{1}{\beta}
\log
\mathbb E\!\left[
\exp(\beta R(X_T))
\mid S_t=s,\; A_t=i
\right].
\label{eq:dprm}
\end{equation}
For deterministic updates, \(\exp(\beta R_t^\star(i;s))=h_{t+1}(s^i)\). In general it equals \(\int h_{t+1}(s\prime)K_t(ds\prime\mid s,i)\). The action marginal of the Doob transform is therefore
\begin{equation}
\pi_t^\star(i\mid s)
:=
\frac{q_0(i\mid s)\exp(\beta R_t^\star(i;s))}
{\sum_{j\in\mathcal I(s)}q_0(j\mid s)\exp(\beta R_t^\star(j;s))}.
\label{eq:local_gibbs_reveal}
\end{equation}
This is the Gibbs ordering policy targeted by DPRM. With deterministic updates and the matching initial law, it realizes the full reward-tilted trajectory distribution. With stochastic updates, the full Doob transform also reweights \(K_t\) by \(h_{t+1}\); keeping \(K_t\) fixed defines a Gibbs-controlled reference process that need not have the globally tilted terminal law. Our order-only implementation targets this reference policy.
It is also expensive: evaluating \(R_t^\star\) requires future reward information, and normalizing \(\pi_t^\star\) over all candidate positions can be costly.

Exact position values are impractical because almost every partial sequence is new. We therefore summarize a candidate with quantities available without another model call. Its \emph{phase} \(\phi_t(s_t)\) records generation progress, and its \emph{confidence bin} \(b_t(s_t,i)\) discretizes the host score \(\psi_i(s_t)\). Hosts may add a small positional or task-specific state when phase and confidence alias distinct token-order decisions. This design follows evidence that reasoning changes at sparse uncertain decisions and that reward learning reallocates probability among available reasoning paths \citep{yue2025rlcapacity,wang2025entropyminority,zhao2025absolutezero}. It is also consistent with learned order policies based on confidence, mask status, time, and positional attention, and with diffusion analyses that identify position and EOT structure as order-relevant \citep{park2026confidence,jazbec2026unmasking,ni2026flexibility}.

The pair \((\phi,b)\) is a \emph{bucket}: earlier decisions made at similar progress and confidence share a future-utility estimate. Formally, the projection is
\[
    \zeta_t(s_t,i)=(\phi_t(s_t), b_t(s_t,i)),
    \qquad
    b_t(s_t,i)=\operatorname{Bin}(\psi_i(s_t)).
\]
The full state \(s_t\) still determines the host logits; only the reward correction is shared. Thus a new sequence receives guidance when its bucket has prior evidence. For each bucket, \(N_{\phi,b}\) counts selected events and \(S_{\phi,b}\) accumulates their exponentiated terminal rewards.
A \emph{shortlist} is a candidate collection sampled from the base proposal before reranking; the Soft-BoN analysis uses an i.i.d. multiset and retains duplicate draws. Large candidate sets use a shortlist; small ones can use the full set.

\textbf{Stage 1: confidence alignment.}
Following PUMA \citep{kim2026puma}, teacher-forced training and decoding use the same confidence-progressive order. This replaces random training states by states visited by the decoder and is efficient when confidence identifies easy, stable decisions \citep{cai2026confidence}.

\textbf{Stage 2: reward-guided shortlist control.}
After warmup, the controller draws a shortlist from \(q_0(\cdot\mid s)\) and scores it with a process-reward estimate. Each past selected event \(j\) contributes its phase \(\phi_j\), confidence bin \(b_j\), and completed-trajectory reward \(R_j\).
Define
\[
\mathcal H_{\phi,b}:=\{j:\phi_j=\phi,\; b_j=b\}.
\]
For phase \(\phi\) and bucket \(b\), we maintain
\begin{equation}
    S_{\phi,b}=\sum_{j\in \mathcal H_{\phi,b}} \exp(\beta R_j),
    \qquad
    N_{\phi,b}=|\mathcal H_{\phi,b}|.
\end{equation}
This gives the estimator
\begin{equation}
    \widehat R_{\phi,b}
    =
    \begin{cases}
    \frac{1}{\beta}\log\!\left(\frac{S_{\phi,b}}{N_{\phi,b}}\right),
        & N_{\phi,b}>0,\\
    0, & N_{\phi,b}=0.
    \end{cases}
    \label{eq:dprm_estimator}
\end{equation}
For a candidate position \(i\) in phase \(\phi\) and bucket \(b_i\), the DPRM-guided score is
\begin{equation}
    g_i
    =
    \log \psi_i(s)
    +
    \beta \widehat R_{\phi,b_i}.
\end{equation}
The gate determines when this estimate is trusted. Let \(u\) be the controller schedule index, and let \(T_{\mathrm{warm}}<T_{\mathrm{switch}}\) and \(N_{\mathrm{ready}}>0\) be the warmup, full-switch, and bucket-readiness thresholds.

\begin{algorithm}[t]
\caption{Generic DPRM ordering module}
\label{alg:framework_dprm_module}
\begin{algorithmic}[1]
\REQUIRE state \(s_t\), candidates \(\mathcal I(s_t)\), scores \(\psi_i(s_t)\), phase \(\phi_t\), budget \(m_t\), shortlist size \(N_t\), schedule index \(u\), statistics \(\{N_{\phi,b},S_{\phi,b}\}\)
\STATE sample \(C_t=\{i_1,\dots,i_{N_t}\}\) from the base proposal over \(\mathcal I(s_t)\)
\FOR{each \(i\in C_t\)}
    \STATE \(b_i \gets \operatorname{Bin}(\psi_i(s_t))\)
    \STATE \(\widehat R_{\phi_t,b_i} \gets 0\) if \(N_{\phi_t,b_i}=0\), else \(\frac{1}{\beta}\log\!\left(\frac{S_{\phi_t,b_i}}{N_{\phi_t,b_i}}\right)\)
    \STATE \(\eta_i \gets \eta_u(\phi_t,b_i)\)
    \STATE \(g_i \gets \log \psi_i(s_t)+\beta\widehat R_{\phi_t,b_i}\)
    \STATE \(s_i \gets (1-\eta_i)\log\psi_i(s_t)+\eta_i g_i\)
\ENDFOR
\STATE select the top-\(m_t\) set \(\widehat A_t(s_t)\) under \(s_i\), and apply the host updates
\STATE complete the trajectory and observe \(R(X_T)\)
\FOR{each \(i\in \widehat A_t(s_t)\)}
    \STATE \(N_{\phi_t,b_i}\leftarrow N_{\phi_t,b_i}+1\), \(S_{\phi_t,b_i}\leftarrow S_{\phi_t,b_i}+\exp(\beta R(X_T))\)
\ENDFOR
\end{algorithmic}
\end{algorithm}

\begin{table*}
\centering
\caption{Model calls per output relative to confidence. Seven hosts require no extra denoiser calls. Prism and Omni estimate rewards online and therefore use more evaluations.}
\label{tab:main_inference_compute}
\setlength{\tabcolsep}{5pt}
\renewcommand{\arraystretch}{0.94}
\begin{tabular}{lcl}
\toprule
\textbf{Hosts} & \textbf{Relative cost} & \textbf{Why} \\
\midrule
PUMA, DMPO, DPLM, DCM, GenMol, SDPO, LLaDA-V & \(1\times\) & Scalar lookup; no extra denoiser evaluation \\
Prism & \textbf{\(1.76\times\) NFE} & Reward estimates computed during search \\
Omni-Diffusion & \textbf{\(5\times\) paths} & Five continuations from the same prompt and canvas \\
\bottomrule
\end{tabular}

\end{table*}

We interpolate between confidence and DPRM through
\begin{equation}
   \resizebox{0.9\linewidth}{!}{$ \eta_u(\phi,b)
    =
    \underbrace{\mathrm{clip}\!\left(\frac{u-T_{\mathrm{warm}}}{T_{\mathrm{switch}}-T_{\mathrm{warm}}},0,1\right)}_{\text{global transition}}
    \cdot
    \underbrace{\min\!\left(\frac{N_{\phi,b}}{N_{\mathrm{ready}}},1\right)}_{\text{bucket readiness}},
    \label{eq:gate}
$}\end{equation}
so the practical ranking score is
\begin{equation}
    s_i
    =
    (1-\eta_u(\phi,b_i))\log \psi_i(s)
    +
    \eta_u(\phi,b_i)\,g_i.
    \label{eq:reveal_score}
\end{equation}

\noindent\textbf{Computational cost.}
At seven hosts, training or calibration produces a reusable value table. DPRM then reranks candidates already scored by the host, leaving the number of denoiser calls unchanged from confidence. Prism has no such precomputation stage: it is a test-time scaling method, and its self-verification reward becomes available only after search trajectories have been evaluated. Its table therefore starts from confidence on each question and is populated during search. Omni has a different limitation. A table trained on earlier images failed to predict paired improvements on new prompts, so the reported controller compares five token-order continuations from the current canvas instead of transferring values across unrelated images. \Cref{tab:main_inference_compute} reports both costs; the Supplementary Material gives the failed Omni transfer test and measured wall-clock and NFE records.

The score admits two selection rules. For a shortlist drawn from \(q_0\), Soft-BoN samples with weights proportional to \(\exp(\beta\widehat R_i)\), because candidate frequency already represents the base factor; with exact values, fresh i.i.d. draws, and retained duplicates, its marginal law is analyzed in \Cref{sec:theory}. The hard-tilt rule in \Cref{alg:framework_dprm_module} selects the largest guided scores and is the MAP limit used when one output is required. PUMA and DMPO use sampled shortlists. When repeated bucket statistics are unavailable but a test-time reward oracle is available, we estimate the same conditional log-moment value from position-conditioned continuations:
\begin{equation}
\begin{aligned}
\widehat R^{\rm MC}(i;s)
&=\frac1\beta\log\frac1K\sum_{k=1}^{K}
\exp\!\bigl(\beta R(X_T^{(k)})\bigr),\\
X_T^{(k)}&\sim q_0(\cdot\mid s,i).
\end{aligned}
\label{eq:mc_position_value}
\end{equation}
Omni-Diffusion uses this online form at step 96, evaluating the native position and confidence-rank alternatives \(0.70,0.85,0.90,0.95\) with one deterministic continuation each. DPRM selects by native order score plus terminal-reward tilt; checkpoint, provisional values, seed, and the 260-step schedule are shared. We report its five-path cost together with uniform-order and reward-only controls.

\begin{table*}
\centering
\caption{Theory-to-evidence map. The observable checks correspond to the stated implications of each result. Supplementary Figs.~21--25 provide the full training, coverage, concentration, and outcome diagnostics.}
\label{tab:theory_map}
\scriptsize
\setlength{\tabcolsep}{5pt}
\renewcommand{\arraystretch}{1.18}
\begin{tabular}{p{0.17\linewidth}p{0.38\linewidth}p{0.38\linewidth}}
\toprule
\textbf{Claim} & \textbf{Compact statement} & \textbf{Observable check} \\
\midrule
Bucket tracking &
\(\ell=\log(\mathcal N/\delta)\), \quad \(r(N)=\sqrt{\ell/N}+\ell/N\), \quad
\(\epsilon_t=O\!\left(\beta\eta_t r(N_{\min})+\mathrm{Bias}_t\right)\),
\quad \(\mathrm{Reg}_t\le2m_t\epsilon_t\). &
\Cref{fig:theory_mechanism_summary}(c,d) gives the learned-value/readiness and richer-state controls. Supplementary Figs.~23--24 report late bucket coverage, selected values, empirical-Bernstein radii, and drift. \\
Soft-BoN shortlist &
\(\displaystyle \mathrm{KL}(\nu_\beta\|\widehat\nu_N)
\le T\sinh^2(\beta/2)/N\), for marginal laws, exact rewards, fresh i.i.d. multisets, and matching initial/update laws. &
\Cref{tab:main_inference_compute} reports the deployed shortlist and evaluation budgets; \Cref{fig:omni_matched_intermediate_main} shows terminal changes after reranking positions from one shared state. \\
Finite-sample ordering &
\(T_{\rm conf}=\widetilde O(\kappa^2/\varepsilon)\) vs. \(T_{\rm rand}=\widetilde\Omega(e^{ad}/\varepsilon)\); late, \(T_{\rm DPRM}=\widetilde O(\log 1/\varepsilon)\) vs. \(T_{\rm conf}=\widetilde\Omega(e^{bh}\log 1/\varepsilon)\). &
\Cref{fig:theory_mechanism_summary}(a,b) summarizes lower early CE and larger hard/OOD residual gains; \Cref{tab:main_results} gives the matched pass@\(K\) comparison. Supplementary Figs.~21--23 and 25 show the training reward, confidence--CE proxy, recovered low-confidence mass, and hard/OOD outcome checks in full. \\
\bottomrule
\end{tabular}
\end{table*}

\textbf{Multi-objective terminal utility.}
Scientific benchmarks often judge each sample by several properties. We orient and normalize property \(j\) as a benefit \(y_j(X_T)\in[0,1]\), then declare \(\lambda_j\ge0\) with \(\sum_j\lambda_j=1\). A weighted sum permits compensation; when every property must remain competitive, Tchebycheff scalarization controls the largest weighted shortfall from \(y_j=1\). We use its log-sum-exp smoothing \citep{lin2024smooth}:
\begin{equation}
\resizebox{0.93\linewidth}{!}{$
R_{\rm WS}=\sum_j\lambda_jy_j,
\qquad
R_{\rm STCH}=1-\mu\log\sum_j
\exp\!\left(\frac{\lambda_j(1-y_j)}{\mu}\right)
+\rho\sum_j\lambda_jy_j .
\label{eq:main_multiobjective_reward}$}
\end{equation}
As \(\mu\downarrow0\), \(R_{\rm STCH}\) approaches the Tchebycheff reward; \(\rho\) favors larger average benefit at ties. Smoothing avoids abrupt bucket and rank changes near metric ties. Preferences and normalization are fixed before evaluation, while dataset-level properties such as uniqueness remain external axes. The Supplementary Material specifies every task reward.

\section{Theory and Empirical Implications}
\label{sec:theory}

The theory separates one background alignment fact from three DPRM claims: online value tracking, shortlist approximation, and conditional finite-sample optimization. We state only the operational conclusions here. Full forms appear in the Supplementary Material sections ``Teacher-forced alignment'', ``Online tracking theorem'', ``Soft-BoN approximation'', and ``Finite-sample forward-KL separations''.

\textbf{Background alignment result.}
Following \citet{kim2026puma}, visible-state-dependent reveal decisions yield teacher-forced marginal agreement and preserve the Bayes denoiser under the support conditions in the Supplementary Material. A separate latent-family result considers an explicit restricted observation protocol: independent random masks expose parameter information only when all \(d\) latent coordinates are visible and the supervised response coordinate is masked. The entire visible record, including any labels or feedback, must be parameter-independent on every other mask. With \(a=q(1-q)^d\), \(M=|\Theta|\), and \(0<\delta<1-1/M\), recovery requires at least \(\lceil\log((1-1/M)/\delta)/[-\log(1-a)]\rceil\) random-mask records. Informative oracle trajectories instead suffice with \((d+1)\lceil\kappa^{-1}\log((M-1)/\delta)\rceil\) visited states. This is exponential versus linear in \(d\) for fixed \(q,M,\delta,\kappa\); it is an observation-protocol result, not a lower bound for learners with unrestricted access to complete training examples. Alignment alone does not decide which aligned order is best.

\textbf{Online bucket tracking.}
For a reachable state \(s\), let \(N_{\min,t}(s)\) be the smallest count among candidate buckets, \(\eta_t(s)\) the largest readiness gate, and \(\mathcal N\) the number of bucket--time events controlled simultaneously. Write
\begin{equation}
\begin{aligned}
r_t(s;\delta)
&:=\sqrt{\frac{\log(\mathcal N/\delta)}{N_{\min,t}(s)}}
+\frac{\log(\mathcal N/\delta)}{N_{\min,t}(s)},\\
\epsilon_t(s;\delta)
&:=O\!\left(\beta\eta_t(s)r_t(s;\delta)
+\operatorname{Bias}_t(s;\delta)\right),
\end{aligned}
\label{eq:main_bucket_error}
\end{equation}
where \(\operatorname{Bias}_t\) collects state coarsening, incomplete warmup/readiness, and nonstationary drift.
\begin{theorem}[Informal: online bucket score tracking]
\label{thm:online_main_informal}
Under bounded rewards and the regularity conditions stated in the Supplementary Material, with high probability
\[
\begin{aligned}
\sup_i|\widehat g_t(i;s)-g_t^\star(i;s)|
&\le\epsilon_t(s;\delta),\\
\operatorname{Regret}_t(s)
&\le2m(s)\epsilon_t(s;\delta).
\end{aligned}
\]
\end{theorem}
The theorem predicts shrinking sampling error, confidence fallback for sparse cells, and better transfer from a less aliased state. \Cref{fig:theory_mechanism_summary}(c,d) shows the value/readiness and state-abstraction checks; the Supplementary Material gives the radius trajectory and proof.

\textbf{Soft-BoN approximation to the Gibbs-policy target.}
Let \(\nu_\beta\) be the terminal law under the exact Gibbs rule in \cref{eq:local_gibbs_reveal}, and \(\widehat\nu_N\) the terminal law of stagewise Soft-BoN after averaging over all candidate draws and selections.
\begin{theorem}[Marginal Soft-BoN terminal approximation]
\label{thm:layerwise_softbon}
Assume finite horizon \(T\), finite candidate sets, \(\beta>0\), and \(R(X_T)\in[0,1]\). Both processes use the same deterministic state updates and the same initial law. At each step, Soft-BoN draws \(N\ge1\) fresh conditionally i.i.d. candidates from \(q_0(\cdot\mid s)\), retains duplicates, and selects a candidate occurrence with its exact weight \(\exp(\beta R_t^\star)\). Then
\begin{equation}
\mathrm{KL}\!\left(\nu_\beta \,\middle\|\, \widehat{\nu}_{N}\right)
\le
\frac{T\,\sinh(\beta/2)^2}{N}.
\label{eq:terminal_kl_rate_uniform}
\end{equation}
To identify \(\nu_\beta\) with the globally reward-tilted terminal law, the common initial law is fixed or equals \(\rho_\beta\) as defined above.
\end{theorem}
The proof bounds the forward KL to the marginal action policy before applying the trajectory chain rule and data processing. It also applies to a common stochastic update kernel when \(\nu_\beta\) denotes the Gibbs-controlled reference law. Estimated rewards, deterministic top-\(m\) selection, shortlist deduplication, and reuse of candidate sets are outside this exact-weight bound. The Supplementary Material gives the pathwise proof, an initial-distribution correction, and the host shortlist settings; \Cref{tab:main_inference_compute} reports model calls.

\textbf{Conditional finite-sample optimization advantage.}
The Bayes-predictor invariance result motivates studying finite-sample optimization with a common population target. The following optimization guarantees use their own geometric and coverage assumptions; they do not follow from the restricted-observation recovery bound. Importance-sampling SGD motivates confidence as an early proxy for gradient importance \citep{zhao2015importance,hanchi2022srg,chen2023hetersgd}; confidence gating can later under-cover useful trajectories \citep{fang2026locally}.
\begin{theorem}[Informal: conditional ordering sample complexity]
\label{thm:optimization_advantage_informal}
Under the uniform noise and two-sided progress conditions in the Supplementary Material, if confidence is a \(\kappa_t\)-accurate importance proxy and informative aligned orders form an exponentially smaller family of difficulty \(d_t\), then, in the stated small-error regime,
\begin{equation}
\begin{aligned}
T_t^{\rm conf}(\varepsilon)
&=\widetilde O\!\left(\frac{\kappa_t^2}{\varepsilon}\right),\\
T_t^{\rm rand}(\varepsilon)
&=\widetilde \Omega\!\left(\frac{e^{a_t d_t}}{\varepsilon}\right),
\quad a_t>0.
\end{aligned}
\label{eq:early_order_complexity}
\end{equation}
If confidence visits a necessary residual family with exponentially small probability in \(h_t\), residual progress satisfies the stated two-sided bounds, and the actual DPRM proposal assigns that family a probability bounded below independently of \(h_t\), then
\begin{equation}
T_{t,\mathrm{late}}^{\rm DPRM}(\varepsilon)
=
\widetilde O\!\left(\log\frac{1}{\varepsilon}\right),
\qquad
T_{t,\mathrm{late}}^{\rm conf}(\varepsilon)
=
\widetilde \Omega\!\left(e^{b_t h_t}\log\frac{1}{\varepsilon}\right).
\label{eq:late_order_complexity}
\end{equation}
\end{theorem}
The late-stage coverage condition requires a sufficiently large score gap and, for a sampled controller, adequate candidate coverage; small score error alone is insufficient. These rates count parameter updates. If the shortlist proposal itself puts exponentially small mass on the residual family, a constant-size shortlist cannot supply constant residual coverage, and the required candidate-generation cost may be exponential. Likewise, a bounded reward with fixed \(\beta\) cannot amplify exponentially small proposal mass to a constant. Under the stated conditions, the results predict lower local loss under confidence early and larger DPRM gains on hard or OOD cases later. \Cref{fig:theory_mechanism_summary}(a,b) checks these qualitative signatures. \Cref{tab:scientific_ordering_results}(a) shows that changing the order is not enough when the host maps different orders to the same final output.

\begin{table*}[!t]
\centering
\caption{Multimodal ordering on fixed host models: Omni-Diffusion orders visual-code positions on \(512\) untouched confirmation prompts, and LLaDA-V orders image-conditioned answer positions on \(509\) strict held-out documents. \textbf{Paired normalized two-host aggregate: confidence 0.92561, DPRM 1.00000.}}
\label{tab:multimodal_ordering_results}
\scriptsize
\setlength{\tabcolsep}{4pt}
\begin{subtable}{0.54\linewidth}
\centering
\caption{Omni-Diffusion, \(512\)-prompt untouched PartiPrompts confirmation.}
\begin{tabular}{lccc}
\toprule
\textbf{Order} & \textbf{CLIP-L/14} & \textbf{CLIP-B/32} & \textbf{Paths} \\
\midrule
\omnimatchedrows
\bottomrule
\end{tabular}
\end{subtable}
\hfill
\begin{subtable}{0.42\linewidth}
\centering
\caption{LLaDA-V strict held-out accuracy.}
\begin{tabular}{lcc}
\toprule
\textbf{Task} & \textbf{Conf.} & \textbf{DPRM} \\
\midrule
RealWorldQA & 0.4735 & \bestcell{0.4892} \\
Numeric/count & 0.3205 & \bestcell{0.4103} \\
\bottomrule
\end{tabular}
\end{subtable}

\vspace{0.55em}

\footnotesize Omni DPRM minus confidence is \(+0.02464\,[+0.02249,+0.02677]\) for CLIP-L/14 and \(+0.01018\,[+0.00820,+0.01221]\) for CLIP-B/32. RealWorldQA overall changes by \(+1.57\)pp with paired \(95\%\) interval \([-0.39,+3.73]\)pp; numeric/count changes by \(+8.97\,[+2.56,+15.38]\)pp.
\end{table*}

\begin{figure*}[!t]
\centering
\includegraphics[width=0.92\linewidth]{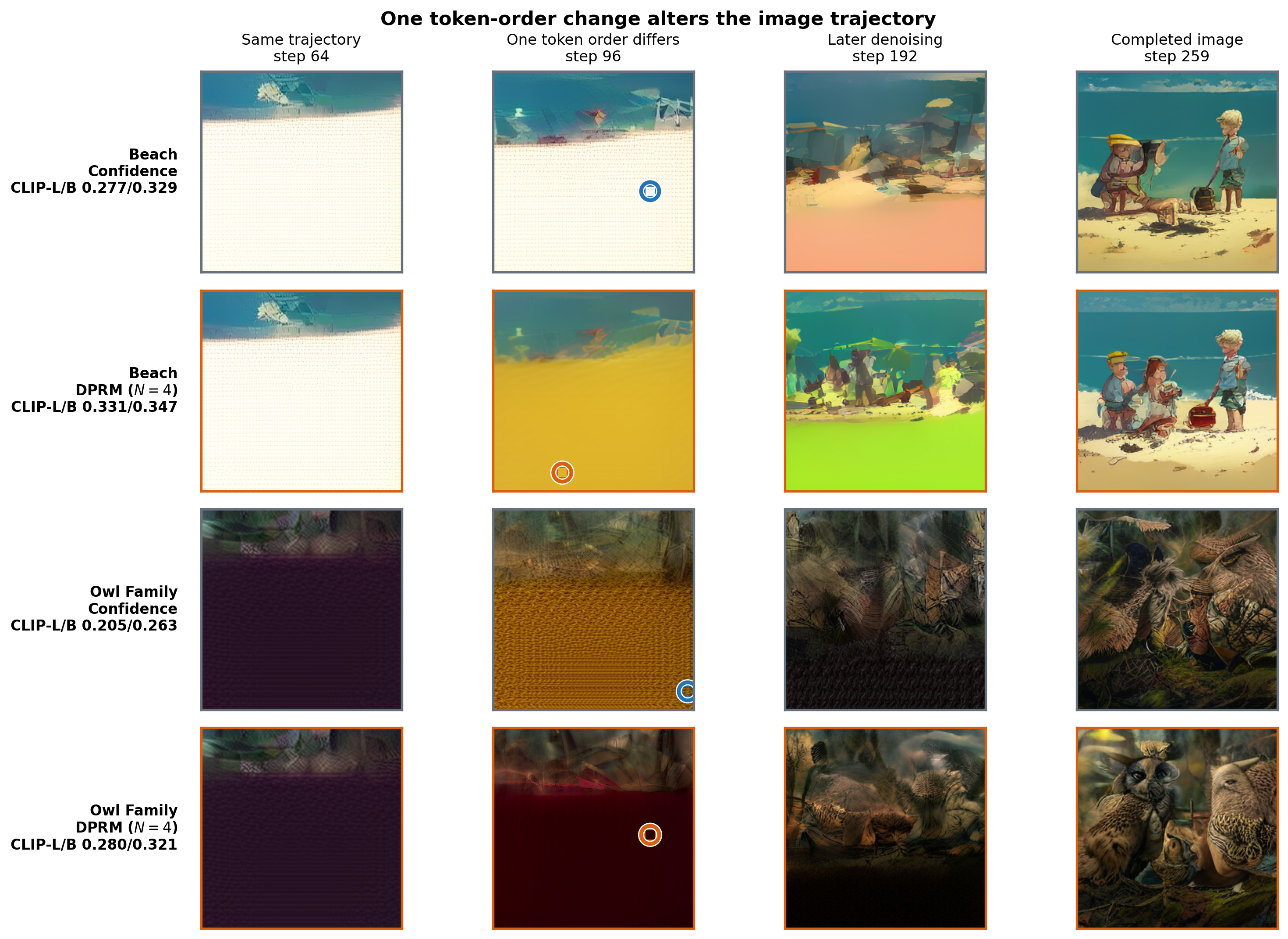}
\caption{Omni-Diffusion token-order diagnostic on a beach scene and an owl family. \textbf{Shared state: confidence and DPRM use the same canvas through step 95.} At step 96, blue marks the lowest-entropy position and orange the lower-confidence position selected by terminal utility. \textbf{Single intervention: all later token values use the same confidence continuation.} The DPRM position opens an unfilled beach region or anchors a peripheral body region, improving object separation. \textbf{Interpretation: one token-order change alters the later image trajectory.} Aggregate conclusions come from the disjoint \(512\)-prompt confirmation in \Cref{tab:multimodal_ordering_results}.}
\label{fig:omni_matched_intermediate_main}
\end{figure*}

\section{Experiments}
\label{sec:experiments}

We test nine published hosts spanning text, images, proteins, cells, molecules, and DNA. In every comparison, the host model and token-value rule are fixed. The tables report task results and uncertainty. \Cref{fig:theory_mechanism_summary,fig:omni_matched_intermediate_main,fig:puma_reveal_order_case_study} give the main ablations and order analyses, while \Cref{tab:main_inference_compute} reports model calls. The Supplementary Material contains the full experimental record. ``Experiment Details and Mechanism Controls'' gives complete settings and host-level ablations. ``Runtime and Sampling-Cost Analysis'' gives measured costs. ``Statistical Uncertainty Protocol'' gives every confidence interval and multi-seed analysis. ``Empirical Diagnostics for the Finite-Sample Theory'' gives the full diagnostic curves used to check the assumptions.

\subsection{DPRM on Multimodal Data}
\label{sec:exp_multimodal}

\noindent\textbf{Omni-Diffusion.}
The text-to-image host \citep{li2026omnidiffusion} normally commits the lowest-entropy visual position. We freeze its checkpoint and token sampler and change one token-order decision, at step \(96\). We first tested the reusable-table design used by the sequence hosts. On prompt-grouped validation, the best phase/confidence table had worse RMSE than a zero predictor (\(0.0356\) versus \(0.0344\)) and negative Pearson correlation (\(-0.067\)); on \(96\) new prompts, it changed only eight outputs and gave a mean CLIP-L/14 change of \(+0.00018\) with an interval crossing zero. Historical spatial values therefore did not transfer reliably between unrelated canvases.

\begin{figure*}
\centering
\includegraphics[width=0.98\linewidth]{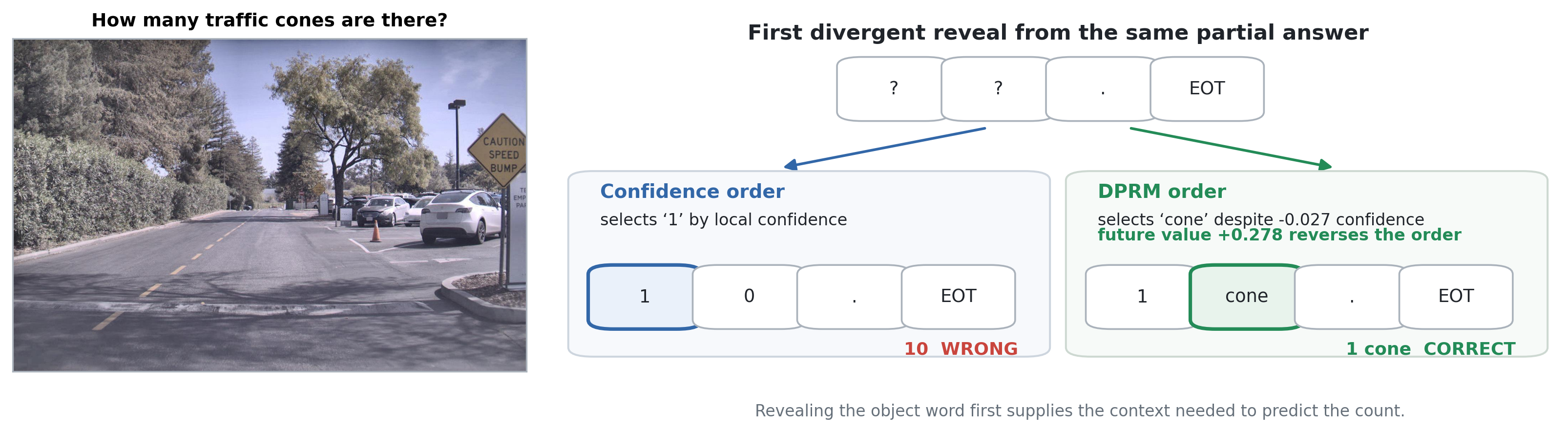}
\caption{Same-canvas RealWorldQA intervention. From the same partial answer, confidence reveals \texttt{1} and later produces \texttt{10}; DPRM uses estimated future utility to reveal \texttt{cone}, then predicts the correct count. The Supplement reports all seven DPRM-only numeric wins.}
\label{fig:rwqa_order_case_main}
\end{figure*}

\begin{table*}[!t]
\centering
\caption{Natural-language ordering results across pretraining, reasoning post-training, and test-time scaling. Each comparison preserves the host scaffold and changes the ordering controller.}
\label{tab:natural_language_results}
\scriptsize
\setlength{\tabcolsep}{4pt}
\renewcommand{\arraystretch}{0.92}

\begin{subtable}{0.54\linewidth}
\centering
\caption{PUMA on GSM8K at EMA step \(2.00\)M with unmasking width 3.}
\label{tab:puma_results}
\begin{tabular}{lcc}
\toprule
\textbf{Method} & \textbf{Accuracy} & \textbf{Paired delta} \\
\midrule
PUMA (confidence) & 30.10 & -- \\
DPRM-PUMA & \bestcell{34.50} & \bestcell{+4.40 [1.90, 6.90]} \\
\bottomrule
\end{tabular}
\end{subtable}
\hfill
\begin{subtable}{0.44\linewidth}
\centering
\caption{DPRM-Prism on GSM8K with LLaDA-2.0-mini.}
\label{tab:prism_results}
\setlength{\tabcolsep}{1.5pt}
\begin{tabular}{lcccc}
\toprule
\textbf{Method} & \textbf{Voted} & \textbf{Rank-1} & \textbf{Any-4} & \textbf{Mean NFE} \\
\midrule
Prism & 82.41 & 82.11 & 84.61 & 609 \\
DPRM-Prism & \bestcell{83.85} & \bestcell{83.70} & \bestcell{86.58} & 1{,}071 \\
\bottomrule
\end{tabular}
\end{subtable}

\vspace{0.55em}

\begin{subtable}{\linewidth}
\centering
\caption{Arithmetic mean of pass@\(K\) over \(K\in\{1,2,4,8,16,32\}\) from complete step-5k success matrices.}
\label{tab:main_results}
\begin{tabular}{lcccc}
\toprule
\textbf{Model} & \textbf{MATH} & \textbf{MATH Hard} & \textbf{Countdown} & \textbf{Countdown Hard} \\
\midrule
\dmpomatchedrows
\bottomrule
\end{tabular}
\end{subtable}

\vspace{0.28em}

\end{table*}

The reported controller instead evaluates the native position and four fixed confidence-rank alternatives on the current canvas. We select the guidance scale on \(128\) prompts, then test it once on \(512\) disjoint prompts. DPRM improves CLIP-L/14 from \(0.18661\) to \(0.21125\) and the independent CLIP-B/32 check from \(0.23836\) to \(0.24854\); both paired intervals are above zero. Uniform-order and reward-only controls use the same five paths. \Cref{fig:omni_matched_intermediate_main} shows the single changed token-order decision from an identical canvas. The Supplementary Material contains every selected position, hash, score, and output path.

\noindent\textbf{LLaDA-V.}
LLaDA-V \citep{you2025llada} provides the image encoder, discrete diffusion model, prompt, provisional token values, and four-step decoder; DPRM changes only the order of answer positions. Documents \(0\)--\(127\), \(128\)--\(255\), and \(256\)--\(764\) are used for fitting, selection, and a strict \(509\)-document test. Overall accuracy changes by \(+1.57\)pp with an interval crossing zero. Numeric/count accuracy improves by \(+8.97\)pp, with seven wins and no losses. Shuffled values, smaller state tables, a numeric-only controller, and the negative AI2D result show where the method does and does not transfer. In \Cref{fig:rwqa_order_case_main}, confidence commits a plausible number first; DPRM reveals the object word before predicting the count. The Supplementary Material reports the full format breakdown.

\subsection{DPRM on Natural Language}
\label{sec:exp_natural_language}

\begin{table*}[t]
\centering
\caption{PUMA reveal-order audit over all $1{,}319$ GSM8K examples at $2.00$M/unmasking-3. Accuracy improves $30.10\%\to34.50\%$, with $169$ DPRM-only and $111$ confidence-only wins. Number timing is normalized by decode length.}
\label{tab:main_reasoning_trace_case_study}
\scriptsize
\setlength{\tabcolsep}{4pt}
\begin{tabular}{@{}p{0.18\linewidth}cccp{0.31\linewidth}@{}}
\toprule
\textbf{Order signal} & \textbf{Confidence} & \textbf{DPRM} & \textbf{Change} & \textbf{Interpretation} \\
\midrule
Within-step adjacency & $50.1\%$ & $37.4\%$ & $-12.7$pp $[-13.2,-12.2]$ & Less contiguous grouping. \\
Mean within-step span & $5.21$ & $6.17$ & $+0.96$ $[+0.82,+1.10]$ & Wider same-step context. \\
Nonlocal steps (span $\geq 8$) & $15.6\%$ & $19.5\%$ & $+3.9$pp $[+3.4,+4.4]$ & More separated positions revealed together. \\
Steps returning earlier & $21.6\%$ & $27.7\%$ & $+6.1$pp $[+5.7,+6.6]$ & More non-left-to-right backfill. \\
First number revealed & $9.36\%$ & $11.53\%$ & $+2.18$pp $[+1.26,+3.10]$ & More context before arithmetic. \\
\bottomrule
\end{tabular}
\end{table*}

\begin{figure*}[t]
\centering
\includegraphics[width=0.92\linewidth]{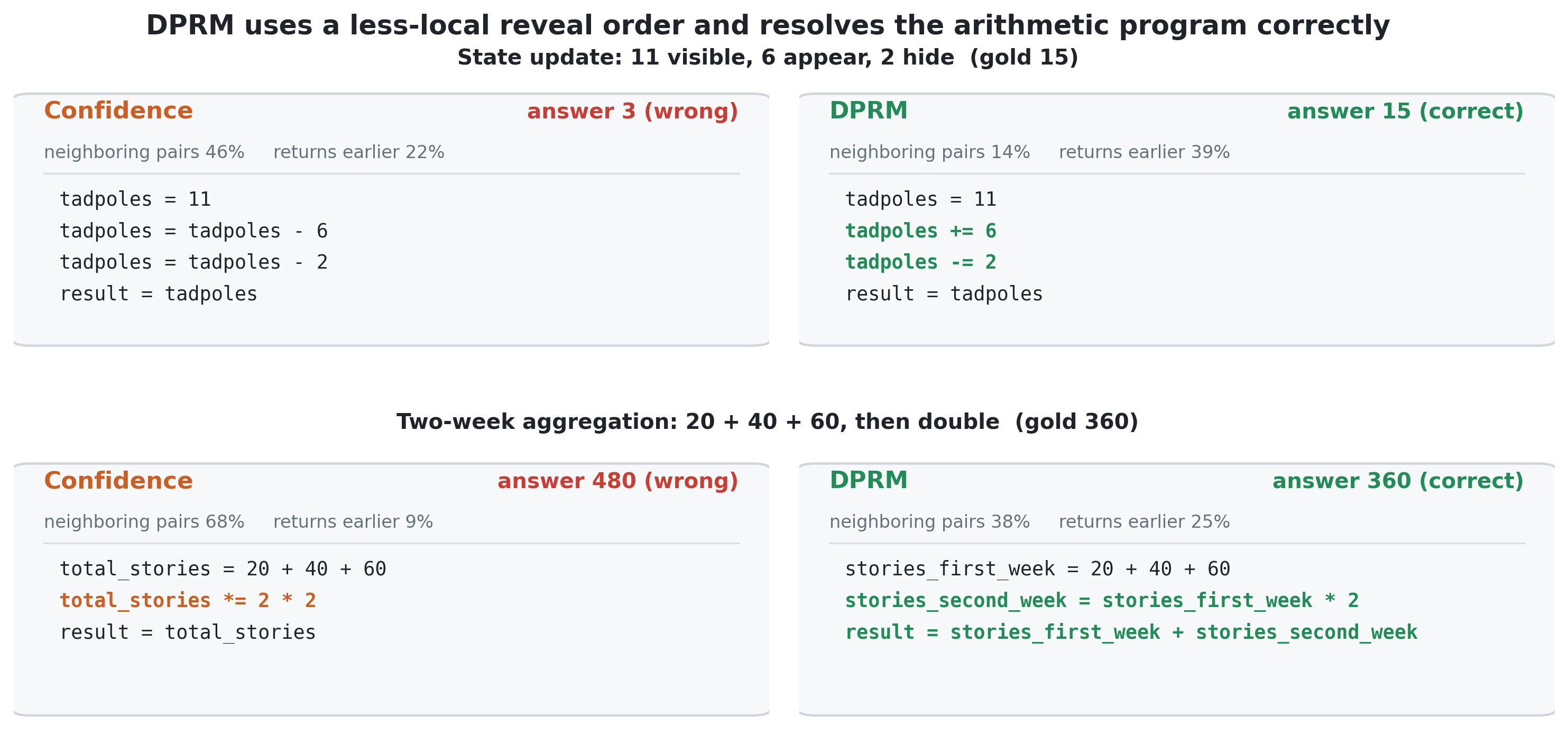}
\caption{Two DPRM-only GSM8K wins. Confidence forms more local reveal sets and produces the wrong state update or aggregation. DPRM returns to earlier positions and produces the correct arithmetic program. Aggregate evidence is in \Cref{tab:main_reasoning_trace_case_study}.}
\label{fig:puma_reveal_order_case_study}
\end{figure*}

\noindent\textbf{PUMA.}
PUMA \citep{kim2026puma} already aligns teacher-forced training with confidence-progressive inference. DPRM preserves its architecture, loss, optimizer, progressive horizon, and reveal width, replacing only confidence top-\(k\) order. At the shared \(2.00\)M EMA endpoint, all \(1{,}319\) GSM8K questions are decoded zero-shot at temperature \(0\) and reveal width \(3\). \Cref{tab:puma_results} gives the \(30.10\%\to34.50\%\) result and its \(5{,}000\)-resample paired interval.

Following DiffusionGemma's intermediate-canvas analysis \citep{engels2026transparentdiffusiongemma}, we compare every paired trajectory after removing EOT-only tails. \Cref{tab:main_reasoning_trace_case_study,fig:puma_reveal_order_case_study} show wider reveal spans, more returns to earlier masks, and later numeric commitment. These traces describe the trained DPRM pipeline; they do not isolate the final bucket table from training. The Supplementary Material gives the metric definitions, paired tests, and all examples on which the two methods disagree.

These are order proxies, not token replacement: a return moves to an earlier still-masked slot.

\noindent\textbf{DMPO.}
Progressive DMPO \citep{zhu2025dmpo} replaces random masks with confidence-aligned masks. DPRM changes only this order during the matched \(5\)k-step training and decoding runs. MATH uses \(500\) problems, Countdown uses \(5{,}120\), and each success matrix contains \(32\) samples per problem. \Cref{tab:main_results} reports the matched comparison. We also vary difficulty, pass@\(K\), \(\beta\), the base order, bucket values, and readiness gates. Gains are largest on hard or OOD questions and at larger \(K\), where confidence is most likely to miss a useful order.

\noindent\textbf{Prism.}
Prism \citep{bai2026prism} uses hierarchical search, self-verification, pruning, and survivor voting; these parts are unchanged. Unlike the training and post-training hosts, Prism begins with each test question and has no earlier trajectory stream from which to estimate verifier-conditioned token-order values. DPRM therefore starts from confidence and updates its table from self-verification scores produced during that question's search. On all \(1{,}319\) GSM8K questions, it improves voted, rank-1, and any-of-four accuracy, while mean NFE rises from \(609\) to \(1{,}071\). To test whether the gain comes only from more computation, we increase the confidence-only budget on a \(250\)-question subset; it still remains below DPRM. \Cref{tab:prism_results,tab:main_inference_compute} report accuracy and cost, and the Supplementary Material gives rank-wise intervals and the full budget sweep.

\subsection{Theory-Facing Diagnostics and Failure Boundaries}
\label{sec:exp_theory_diagnostics}

\begin{figure*}[t]
\centering
\includegraphics[width=\textwidth]{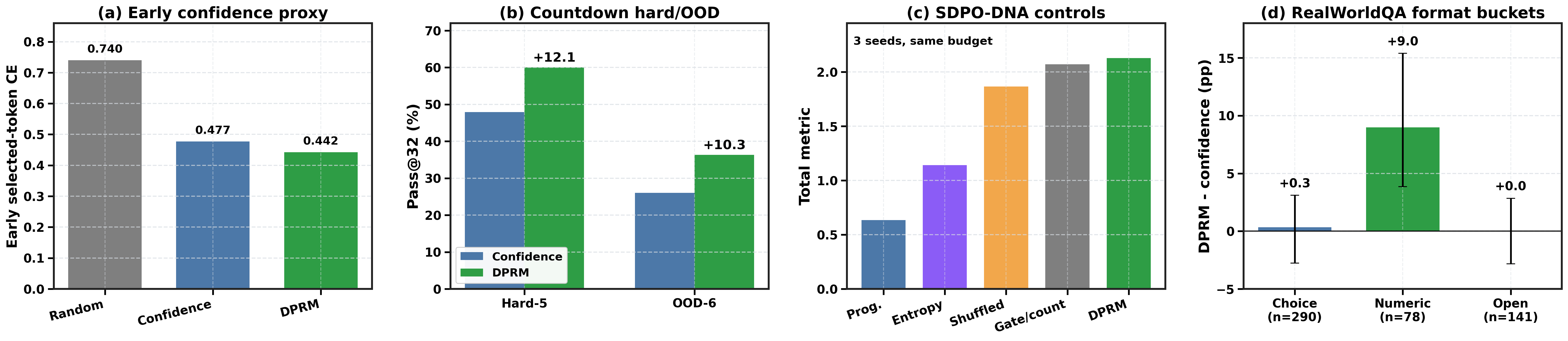}
\caption{Direct checks for the conditional theory. \textbf{(a) Early proxy:} confidence-progressive selection lowers early token CE relative to random order. \textbf{(b) Residual coverage:} DPRM improves hard/OOD Countdown trajectories undersampled by confidence. \textbf{(c) Value association:} SDPO-DNA requires learned utility values, not only uncertainty, counts, or gates. \textbf{(d) Abstraction bias:} format/EOT/position cells isolate the RealWorldQA numeric order class. Complete radius and coverage trajectories are in the Supplementary Material.}
\label{fig:theory_mechanism_summary}
\end{figure*}

The theory is conditional, so final accuracy alone is not a sufficient validation. We test its observable implications and retain cases where an assumption fails.

\begin{figure*}[t]
\centering
\includegraphics[width=0.98\textwidth]{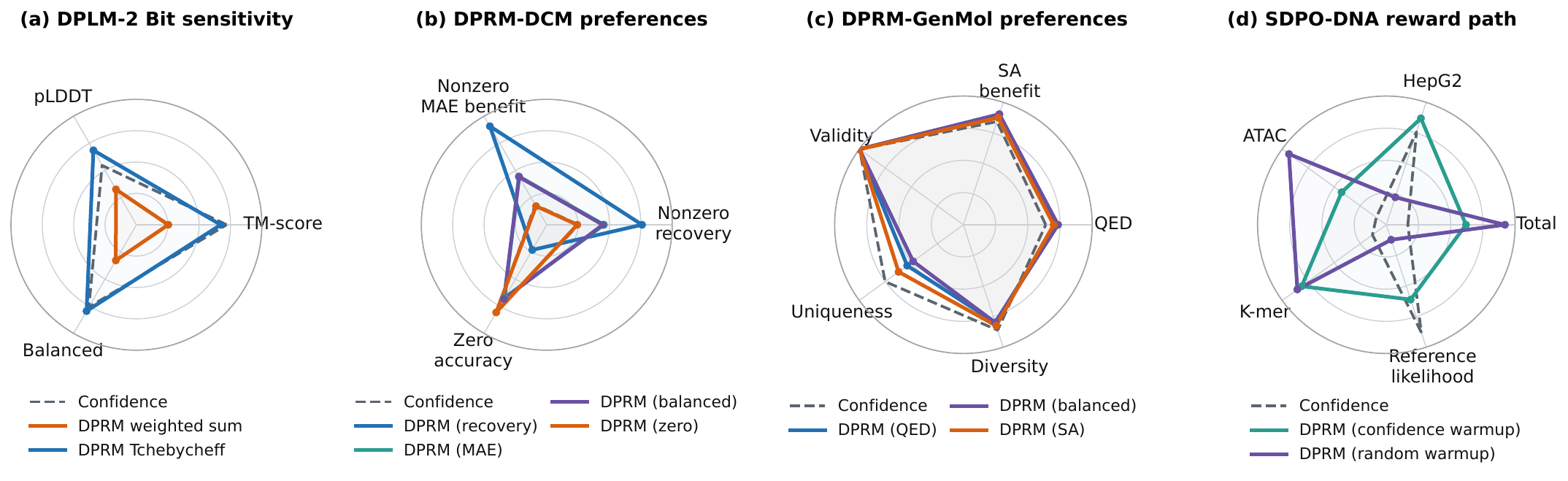}
\caption{Scientific ordering under declared terminal rewards. Higher is better on every axis; the dashed curve is confidence order and solid curves are DPRM endpoints. DPLM-2 Bit is the terminal-invariance control. DCM and GenMol show preference responses; SDPO shows how a HepG2-guided order changes external accessibility and sequence-quality metrics. Native values and intervals are in the Supplement.}
\label{fig:scientific_preference_control}
\end{figure*}
\subsection{DPRM on Scientific Data}
\label{sec:exp_scientific_data}

\begin{table*}
\centering
\caption{Scientific ordering under declared terminal rewards. Higher is better on every listed axis. Color and boldface mark the best value in each column; confidence remains best on several external diversity or likelihood axes, while DPRM controls the rewarded preference. Confidence intervals and the complete preference sweeps are in the Supplementary Material.}
\label{tab:scientific_ordering_results}
\scriptsize
\setlength{\tabcolsep}{2.2pt}
\renewcommand{\arraystretch}{0.92}

\begin{subtable}[t]{0.47\linewidth}
\centering
\caption{DPLM-2 Bit, CoGen-200 development audit.}
\begin{tabular}{@{}lccc@{}}
\toprule
\textbf{Order} & \textbf{TM} & \textbf{pLDDT} & \textbf{Balanced} \\
\midrule
Confidence & \bestcell{0.6303} & 0.8583 & 0.7355 \\
DPRM weighted & 0.5978 & 0.8428 & 0.7097 \\
DPRM Tcheby. & 0.6270 & \bestcell{0.8680} & \bestcell{0.7377} \\
\bottomrule
\end{tabular}
\end{subtable}
\hfill
\begin{subtable}[t]{0.51\linewidth}
\centering
\caption{DCM, Dentate Gyrus reconstruction.}
\begin{tabular}{@{}lccc@{}}
\toprule
\textbf{Order} & \textbf{NZ recovery} & \textbf{MAE benefit} & \textbf{Zero acc.} \\
\midrule
Confidence & 0.0183 & 0.6410 & \bestcell{0.9975} \\
DPRM-recovery & \bestcell{0.0198} & \bestcell{0.6432} & 0.9973 \\
DPRM-MAE & 0.0189 & 0.6418 & \bestcell{0.9975} \\
DPRM-balanced & 0.0188 & 0.6418 & \bestcell{0.9975} \\
DPRM-zero & 0.0182 & 0.6410 & \bestcell{0.9975} \\
\bottomrule
\end{tabular}
\end{subtable}

\vspace{0.35em}
\begin{subtable}[t]{0.49\linewidth}
\centering
\caption{GenMol V2, de novo generation.}
\setlength{\tabcolsep}{1.5pt}
\begin{tabular}{@{}lccccc@{}}
\toprule
\textbf{Order} & \textbf{QED} & \textbf{SA} & \textbf{Valid.} & \textbf{Unique} & \textbf{Diverse} \\
\midrule
Confidence & 0.6392 & 0.8425 & 0.9970 & \bestcell{0.7513} & \bestcell{0.8607} \\
DPRM-QED & 0.7254 & 0.9006 & \bestcell{0.9990} & 0.5395 & 0.7982 \\
DPRM-balanced & \bestcell{0.7350} & \bestcell{0.9032} & 0.9980 & 0.4830 & 0.7988 \\
DPRM-SA & 0.7008 & 0.8748 & 0.9980 & 0.6222 & 0.8242 \\
\bottomrule
\end{tabular}
\end{subtable}
\hfill
\begin{subtable}[t]{0.49\linewidth}
\centering
\caption{SDPO-DNA, GOSAI generation.}
\setlength{\tabcolsep}{1.5pt}
\begin{tabular}{@{}lccccc@{}}
\toprule
\textbf{Order} & \textbf{Total} & \textbf{HepG2} & \textbf{ATAC} & \textbf{K-mer} & \textbf{Log lik.} \\
\midrule
Confidence & 0.384 & 4.138 & 0.168 & 0.553 & \bestcell{-224.31} \\
DPRM-conf. & 1.423 & \bestcell{4.304} & 0.401 & 0.825 & -233.61 \\
DPRM-random & \bestcell{2.119} & 3.337 & \bestcell{0.754} & \bestcell{0.842} & -250.72 \\
\bottomrule
\end{tabular}
\end{subtable}
\end{table*}

\noindent\textbf{1. Early confidence proxy.}
On Countdown traces, confidence-bin/CE correlations range from \(-0.94\) to \(-0.97\). Early selected-token CE is \(0.740\) under random order, \(0.477\) under progressive confidence, and \(0.442\) under DPRM at \(\beta=1\). This supports the local importance-proxy premise behind the confidence-first stage; it does not assert that confidence is globally optimal.
The corresponding CE comparison is visualized in \Cref{fig:theory_mechanism_summary}(a).

\noindent\textbf{2. Sampling error, readiness, and residual coverage.}
The logged empirical-Bernstein radius decreases as bucket counts accumulate, while empty cells revert to confidence. At the coverage-supported SDPO setting \(P=1\), all \(10/10\) buckets and all selected decisions are ready; finer \(P=4/8\) tables lose readiness and utility. Late Countdown selection from confidence bins \(0\)--\(5\) rises from \(3.1\%\) under confidence to \(4.2\%\) under DPRM (\(\beta=1\)); hard/OOD pass@32 then rises \(47.9\to60.0\) and \(26.0\to36.3\), respectively. These are the predicted sampling-radius and residual-family directions.
The outcome-level residual check is \Cref{fig:theory_mechanism_summary}(b); the Supplementary Material reports the radius and coverage trajectories.

\noindent\textbf{3. Learned values and abstraction bias.}
On the three-seed SDPO audit, DPRM total utility is \(2.129\), compared with \(1.866\) after shuffling the learned values, \(2.070\) for gate/count only, and \(1.141\) for entropy order. At a fixed DMPO checkpoint, entropy and random order are also more nonlocal than DPRM but less accurate. Simply reversing confidence is therefore not enough: uncertainty identifies candidate positions, while the terminal-reward estimate determines which uncertain positions are useful. RealWorldQA shows the complementary role of state design: a phase/confidence table transfers negatively (\(0.418\) versus \(0.460\)), whereas adding format/EOT/position cells yields \(+8.97\)pp on numeric/count questions. \Cref{fig:theory_mechanism_summary}(c,d) visualizes the value and state-abstraction checks.

\noindent\textbf{4. Terminal sensitivity boundary.}
DPLM provides a necessary counterexample. Weighted and Tchebycheff controllers change \(97.67\%\) and \(97.68\%\) of structure-order decisions with full selected readiness, yet all \(163\) CAMEO terminal files are identical. Active guidance is therefore insufficient when the host sampler erases the intervention; \Cref{tab:scientific_ordering_results}(a) reports this boundary. The Supplementary Material gives the terminal-file audit.

\noindent\textbf{Common protocol.}
Scientific hosts expose several valid objectives. We declare a preference vector \(\lambda\) before evaluation and use the weighted or smooth-Tchebycheff utility in \cref{eq:main_multiobjective_reward}; set-level properties that cannot score one trajectory remain external axes. \Cref{tab:scientific_ordering_results,fig:scientific_preference_control} therefore report both rewarded response and non-rewarded trade-offs.

\noindent\textbf{Uncertainty and repeated runs.}
We use paired tests when methods share evaluation examples and ordinary bootstrap intervals otherwise. The Supplementary Material section ``Statistical Uncertainty Protocol'' gives the evaluation unit and resampling rule for every host, together with the multi-seed results for stochastic scientific runs. Seed-level samples are analyzed within each seed before results are averaged. DPLM is deterministic at the final-file check, so it is evaluated by exact file comparison rather than repeated decoding.

\noindent\textbf{DPLM-2 Bit: activation without terminal sensitivity.}
The protein host \citep{wang2025dplm2} uses matched \(5\)k-step controllers and all \(163\) CAMEO targets. DPRM changes almost every structure-order decision, yet the \(500\)-step sampler produces exactly the same final files. On the CoGen-200 development test, the selected Tchebycheff setting has a small positive point estimate but its interval crosses zero, so we do not run the planned confirmation. This result shows that a host must be sensitive to order for DPRM to help.

\noindent\textbf{DCM: paired preference control.}
We keep the single-cell denoiser fixed \citep{bhattacharya2026discrete} and decode the same \(293\) held-out cells under recovery, MAE, balanced, and zero-expression preferences. Each preference moves its target metric in the intended direction, with a positive interval and full bucket readiness. Two additional decode seeds reproduce all three directions, for three seeds in total. The change therefore comes from the order and stated preference, not from a different checkpoint.

\noindent\textbf{GenMol V2: rewarded gain with an external trade-off.}
The SAFE tokenizer, denoiser, temperature, and \(1{,}000\)-molecule budget are fixed \citep{lee2025genmol}. Balanced DPRM raises QED by \(0.0956\) and SA benefit by \(0.0607\) over confidence, with positive bootstrap intervals. Uniqueness and diversity fall, so the result is a quality--diversity trade-off rather than an improvement on every metric.

\noindent\textbf{SDPO-DNA: value and bucket ablations.}
All rows use the same pretrained DNA model, two training epochs, \(640\) generated sequences, and evaluation oracles \citep{wang2025finetuning}. Only HepG2 expression is used as the controller reward; ATAC, k-mer correlation, and reference likelihood are held out for evaluation. A single phase (\(P=1\)) gives complete bucket coverage and the best total utility. Finer tables (\(P=4/8\)) leave more buckets unready and perform worse. Across three independently fine-tuned seeds, mean total utility is \(2.129\) for DPRM(random), \(1.389\) for native SDPO, \(1.141\) for entropy order, \(0.779\) for DPRM(confidence), and \(0.634\) for confidence-progressive order. Shuffling the learned values also lowers utility, as shown in \Cref{fig:theory_mechanism_summary}(c). The gain therefore depends on both useful reward values and enough data per bucket.

\FloatBarrier
\raggedbottom
\section{Conclusion, Limitation and Future Work}
\label{sec:conclusion}

Confidence-based decoding is a strong practical heuristic, but it optimizes local certainty rather than terminal quality. DPRM replaces this heuristic with a principled target: tilt the host trajectory distribution by terminal reward, derive its conditional process reward for each candidate position, and approximate that correction with a compact table. Exact global tilting holds for deterministic updates with the matching initial law; stochastic hosts retain their token kernel and use the induced Gibbs ordering policy as the reference. Confidence remains the base order and the safe fallback, while phase, position, and task state provide only the additional structure needed for reward transfer. The host denoiser, token values, supervision, and objective remain unchanged.

Across nine hosts, this lightweight correction improves language reasoning, numeric VQA, visual-token generation, and scientific objectives. Entropy and random-order controls show that the improvement is not generic exploration. PUMA and Omni traces instead show DPRM selecting a position whose revealed context improves later generation. DPLM marks the necessary limit: token order cannot help when the host sampler maps different orders to the same final output.

Where reward statistics transfer, DPRM adds only table lookup and scalar reranking. Prism lacks a training history, and Omni's cross-prompt table did not transfer, so these two hosts estimate values online at additional cost. Reducing that cost, extending repeated-run coverage for expensive hosts, and adding blinded human evaluation for visual generation remain important directions.

\begin{table}[!h]
\centering
\small
\caption{Practical deployment rule distilled from the nine-host audit. DPRM is used only when its value estimate is supported and the host is order-sensitive.}
\label{tab:practical_deployment_rule}
\setlength{\tabcolsep}{3pt}
\begin{tabular}{p{0.29\columnwidth}p{0.63\columnwidth}}
\toprule
\textbf{Observed condition} & \textbf{Evidence-based response} \\
\midrule
Supported value cells & \textbf{Apply DPRM:} selected readiness is high, the radius shrinks, and value controls pass. \\
Sparse or aliased cells & \textbf{Fall back to confidence:} counts are low, the radius is large, or a richer-state ablation transfers negatively. \\
Order-insensitive host & \textbf{Do not promote:} decisions change but terminals do not, as in DPLM. \\
Prompt-local values only & \textbf{Report rollout cost:} values cannot transfer safely across states, as in Omni. \\
\bottomrule
\end{tabular}
\end{table}

\ifdefined\DPRMArxivVersion
\section*{Acknowledgements}
DB and HW are supported in part by the Research Grants Council of the Hong Kong Special Administrative Region (Project No. CityU 11206622). WH is supported by JSPS KAKENHI (24K20848) and JST BOOST (JPMJBY24G6). TS was partially supported by JSPS KAKENHI (24K02905) and JST CREST (PMJCR2015). This research is supported by the National Research Foundation, Singapore and the Ministry of Digital Development and Information under the AI Visiting Professorship Programme (award number AIVP-2024-004). Any opinions, findings, conclusions, or recommendations expressed in this material are those of the authors and do not reflect the views of the National Research Foundation, Singapore or the Ministry of Digital Development and Information.
\fi

\bibliographystyle{IEEEtranN}
\bibliography{ref}

\ifdefined\DPRMArxivVersion
\def\DPRMEmbeddedSupplement{1}
\ifdefined\DPRMEmbeddedSupplement
\else
\PassOptionsToPackage{dvipsnames,svgnames,table}{xcolor}
\documentclass[10pt,journal,compsoc,onecolumn]{IEEEtran}

\usepackage[utf8]{inputenc}
\usepackage[T1]{fontenc}
\usepackage{microtype}
\usepackage{graphicx}
\usepackage{subcaption}
\usepackage{booktabs}
\usepackage{amsmath,amssymb,amsfonts,mathtools,amsthm}
\usepackage{bm}
\usepackage{enumerate}
\usepackage{enumitem}
\usepackage{pifont}
\usepackage{placeins}
\usepackage{multirow}
\usepackage[normalem]{ulem}
\usepackage[most]{tcolorbox}
\usepackage{colortbl}
\usepackage{wrapfig}
\usepackage{float}
\usepackage{xurl}
\usepackage{xcolor}
\usepackage{algorithm}
\usepackage{algorithmic}
\usepackage[numbers,sort&compress]{natbib}
\usepackage[hidelinks]{hyperref}
\usepackage[capitalize,noabbrev]{cleveref}

\graphicspath{{figs/}}
\DeclareGraphicsExtensions{.pdf,.png,.jpg,.jpeg}

\providecommand{\theHalgorithm}{\arabic{algorithm}}
\theoremstyle{plain}
\newtheorem{theorem}{Theorem}[section]
\newtheorem{proposition}[theorem]{Proposition}
\newtheorem{lemma}[theorem]{Lemma}
\newtheorem{corollary}[theorem]{Corollary}
\theoremstyle{definition}
\newtheorem{definition}[theorem]{Definition}
\newtheorem{assumption}[theorem]{Assumption}
\theoremstyle{remark}
\newtheorem{remark}[theorem]{Remark}

\crefname{equation}{}{}
\crefname{figure}{Fig.}{Figs.}
\crefname{section}{Sec.}{Secs.}
\crefname{appendix}{App.}{Apps.}
\crefname{table}{Tab.}{Tabs.}
\crefname{theorem}{Thm.}{Thms.}
\crefname{lemma}{Lem.}{Lems.}
\crefname{assumption}{Assump.}{Assumps.}
\crefname{proposition}{Prop.}{Props.}
\crefname{corollary}{Cor.}{Cors.}
\crefname{definition}{Def.}{Defs.}
\crefname{remark}{Rmk.}{Rmks.}
\crefname{algorithm}{Alg.}{Algs.}

\definecolor{myred}{RGB}{190,30,45}
\definecolor{myblue}{RGB}{28,76,170}
\definecolor{mygreen}{RGB}{20,120,60}
\definecolor{myorange}{RGB}{210,120,20}
\definecolor{mygray}{RGB}{70,70,70}
\definecolor{boxborder}{RGB}{40,55,90}
\definecolor{boxbg}{RGB}{248,248,248}
\definecolor{mypurple}{RGB}{120,55,160}
\definecolor{lightblue}{RGB}{232,241,255}
\definecolor{lightgreen}{RGB}{232,247,238}
\definecolor{lightpurple}{RGB}{242,235,250}
\definecolor{LightCyan}{rgb}{0.88,1,1}

\newcommand{\bbE}{\mathbb{E}}
\newcommand{\KL}{\mathrm{KL}}
\newcommand{\Bias}{\operatorname{Bias}}
\newcommand{\rad}{\operatorname{rad}}
\newcommand{\drift}{\operatorname{drift}}
\newcommand{\calK}{\mathcal K}
\newcommand{\calO}{\mathcal O}
\newcommand{\calE}{\mathcal E}
\newcommand{\calR}{\mathcal R}
\newcommand{\tO}{\widetilde O}
\newcommand{\tOm}{\widetilde \Omega}
\newcommand{\AH}[1]{}
\newcommand{\cmark}{\ding{51}}
\newcommand{\xmark}{\ding{55}}
\newcommand{\bestcell}[1]{\cellcolor{LightCyan}\textbf{#1}}
\newcommand{\omnimatchedrows}{%
Random & 0.17939 & 0.23381 & 1 \\
Omni default & 0.18661 & 0.23836 & 1 \\
Uniform order & 0.18566 & 0.23796 & 5 \\
DPRM & \bestcell{0.21125} & \bestcell{0.24854} & 5 \\
}

\newcommand{\multimodalaggregaterow}{%
Two-host mean & 0.92561 & \bestcell{1.00000} \\
}

\hyphenation{op-tical net-works semi-conduc-tor}

\title{Supplementary Material for ``DPRM: A Plug-in Token-Ordering Module for Discrete Diffusion Models''}
\author{Dake Bu, Wei Huang, Andi Han, Si Wu, Hau-San Wong, Qingfu Zhang, Taiji Suzuki, and Atsushi Nitanda}

\begin{document}
\maketitle

This document contains the proofs, implementation details, additional diagnostics, and extended experimental results for the paper.
\fi

\ifdefined\DPRMEmbeddedSupplement
\clearpage
\onecolumn
\fi

\appendices

\section{Supplement Roadmap}
\label{app:roadmap}

The supplement contains four items not repeated from the main paper: controller settings, host-level audits, diagnostic trajectories, and full proofs. \Cref{fig:supp_nine_host_evidence_map} is the index; each cell points to evidence retained in the corresponding host section.

\begin{figure}[!htbp]
\centering
\includegraphics[width=0.96\linewidth]{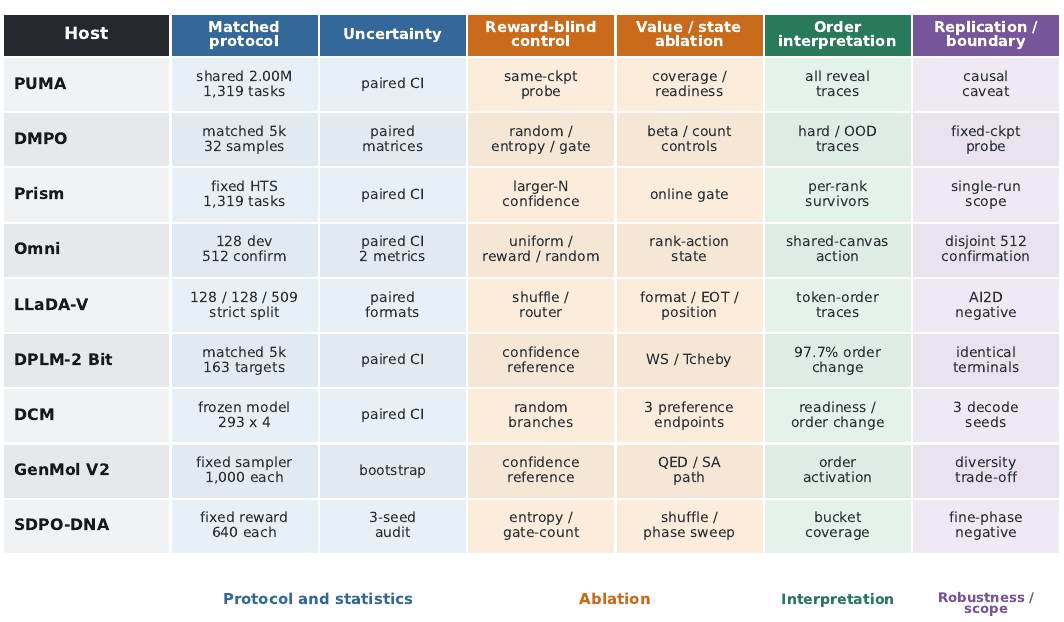}
\caption{Supplement-only evidence map. Column colors group protocol and uncertainty (blue), controls and ablations (orange), order interpretation (green), and robustness or scope (purple).}
\label{fig:supp_nine_host_evidence_map}
\end{figure}

\section{Extended Related Work}\label{sec:add_related}
\paragraph{Ordering and reward tilt.}
Adaptive inference order and PUMA's teacher-forced masking establish confidence as a strong aligned baseline \citep{kim2025train,kim2026puma}. Confidence is nevertheless a heuristic for local certainty, not an objective for terminal quality; it can suppress exploration or commit EOT too early \citep{cai2026confidence,fang2026locally,park2026confidence}. High-entropy tokens can carry important reasoning signal \citep{wang2025entropyminority}, but entropy alone does not identify which uncertain position improves the completion. Decode-time planners address related failures \citep{liu2025think,lee2026lookahead,hong2026wide,israel2026planned}; DPRM instead derives a local process reward from the terminal-reward-tilted trajectory law \citep{doob1957conditional,leonard2014a,chen2021stochastic,bu2026distributionalbiasesposttrainingmarkovian}.

\paragraph{Compact state, training, and search.}
Phase and confidence summarize progress and uncertainty. This choice is consistent with learned unmasking policies that map token confidence, mask status, and time through a small position-aware transformer \citep{jazbec2026unmasking}, as well as analyses of reasoning forks, EOT behavior, and uncertain-token coverage \citep{yue2025rlcapacity,wang2025entropyminority,park2026confidence,ni2026flexibility}. DPRM uses these variables to share only the reward correction; their information loss appears explicitly in the abstraction-bias term. In DMPO and Prism we keep the post-training objective and search scaffold fixed \citep{zhu2025dmpo,bai2026prism}; caching and KV reuse remain orthogonal model-evaluation optimizations \citep{liu2025dllm,ma2025dkv,hu2025accelerating}.

\paragraph{Non-text tokens.}
Masked visual generation exposes an order over codebook positions, while LLaDA-V orders answer tokens under image conditioning \citep{chang2022maskgit,li2026omnidiffusion,you2025llada}. DPLM-2, DCM, GenMol, and SDPO expose the same decision over protein, count-bin, SAFE, and DNA tokens \citep{wang2025dplm2,bhattacharya2026discrete,lee2025genmol,wang2025finetuning}. These studies hold the tokenizer, host model, and evaluator fixed.

\section{Experiment Details and Mechanism Controls}

\subsection{DPRM Utility Oracles and Temperature Settings}
\label{app:dprm_reward_oracles}

DPRM reuses a scalar already available to the host or evaluator; it does not replace the host objective. Every controller uses
\[
g_i(s)=\log \psi_i(s)+\eta_i(s)\,\beta \widehat R_{\phi,b_i},
\qquad
\widehat R_{\phi,b}
=
\frac{1}{\beta}
\log
\frac{\sum_{j\in\mathcal H_{\phi,b}}\exp(\beta R_j)}
{|\mathcal H_{\phi,b}|}.
\]
\Cref{fig:supp_host_configuration_map} shows what is ordered, stored, rewarded, and selected in each host.

\begin{figure}[!htbp]
\centering
\includegraphics[width=0.98\linewidth]{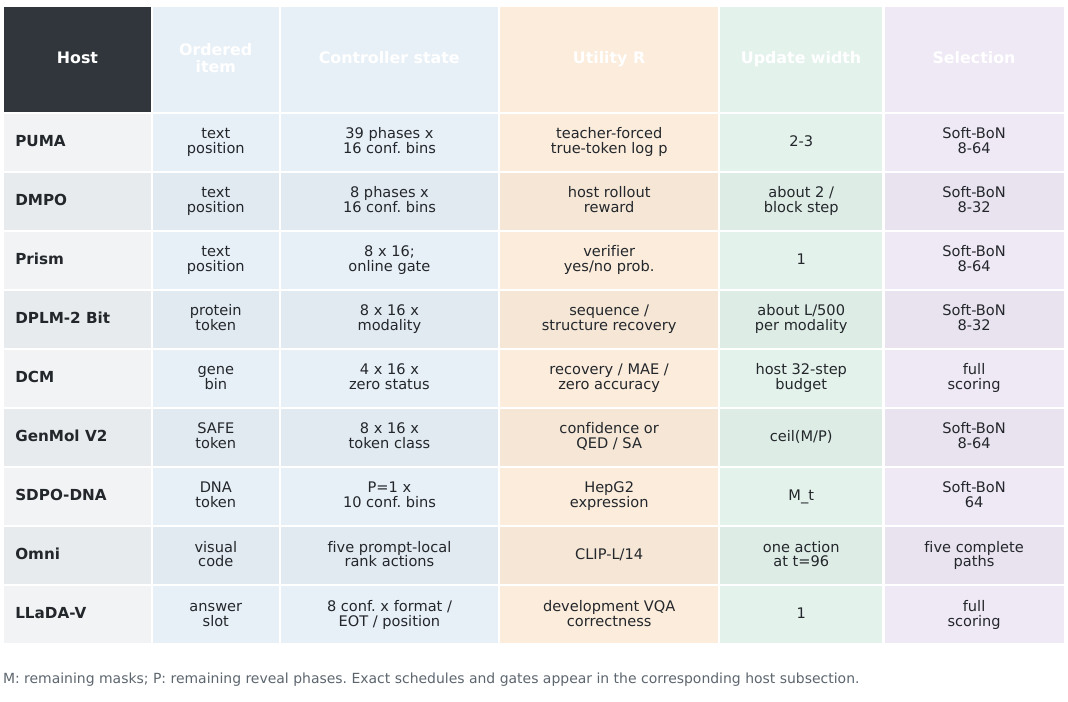}
\caption{Controller specification for all nine hosts. Exact training and evaluation settings appear in the corresponding host subsection.}
\label{fig:supp_host_configuration_map}
\end{figure}

The reward temperature is \(\beta=1\) except for DPLM-2 Bit (\(\beta=8\)). DMPO's DPRM temperature is distinct from its target temperature \(\alpha\). The gate tuples \((T_{\rm warm},T_{\rm switch},N_{\rm ready})\) are
\[
\begin{array}{ll}
\text{PUMA train/decode: }(2\mathrm{k},60\mathrm{k},128)/(0,16,16),
& \text{DMPO, DPLM, GenMol: }(500,2000,128),\\
\text{Prism: }(6,22,64),
& \text{SDPO: }(100,400,64),\quad \text{LLaDA-V: }(0,4,4).
\end{array}
\]
DCM uses readiness \(64\) and development-selected guidance \(2\). Omni has no reusable table: development selects guidance \(8\) with reward scale \(0.03\), then confirmation evaluates five prompt-local token-order choices.

\paragraph{Scientific scalar rewards.}
Each component is oriented as a benefit \(y_j(X)\in[0,1]\). We use a weighted sum or the smooth Tchebycheff reward (\(\mu=\rho=0.05\)):
\begin{align}
R_{\rm WS}(X;\lambda)
&=\sum_{j=1}^{J}\lambda_j y_j(X),\\
R_{\rm STCH}(X;\lambda)
&=1-\mu\log\sum_{j=1}^{J}
\exp\!\left(\frac{\lambda_j(1-y_j(X))}{\mu}\right)
+\rho\sum_{j=1}^{J}\lambda_j y_j(X).
\label{eq:scientific_scalarizations}
\end{align}
Preferences are fixed before evaluation. Set-level metrics are excluded because they cannot score one trajectory.

\begin{table}[!htbp]
\centering
\small
\caption{Scientific sample-level rewards. Letters name preference endpoints used later.}
\label{tab:scientific_scalarization_config}
\setlength{\tabcolsep}{4pt}
\begin{tabular}{p{0.15\linewidth}p{0.31\linewidth}p{0.46\linewidth}}
\toprule
Host & Benefits & Scalarization / preference weights \\
\midrule
DPLM-2 Bit & sequence; structure recovery & WS and STCH: \((.5,.5)\) \\
DCM & nonzero recovery; \(1-\mathrm{MAE}_{\rm nz}/7\); zero accuracy & STCH: R \((.90,.075,.025)\), M \((.05,.90,.05)\), B \((.45,.35,.20)\), Z \((.025,.075,.90)\) \\
GenMol V2 & QED; \((10-\mathrm{SA})/9\) & STCH: Q \((.95,.05)\), B \((.55,.45)\), S \((.05,.95)\) \\
SDPO-DNA & HepG2 expression & single objective; ATAC, k-mer, and likelihood stay external \\
\bottomrule
\end{tabular}
\end{table}

\subsection{Inference Procedure and Computational Cost}
\label{app:dprm_inference_cost}

For masked positions \(M_t\), DPRM replaces an exact state-position value by its bucket projection:
\[
R_t^\star(i;s_t)
=
\frac{1}{\beta}
\log
\mathbb E\!\left[
\exp(\beta R(X))
\mid S_t=s_t,\ I_t=i
\right],
\qquad
(\phi_t,b_i)
=
\bigl(\operatorname{Phase}(t,s_t),\operatorname{Bin}(\psi_i(s_t))\bigr).
\]
The full state still determines \(\psi_i\) and the provisional token; only the reward correction is shared. One denoising step is:
\begin{enumerate}[leftmargin=18pt]
    \item run the host denoiser and compute \(\psi_i(s_t)\);
    \item build the bucket key, read \((N,\widehat R)\), and apply the readiness gate;
    \item rank the full set or shortlist by \(\log\psi_i+\eta_i\beta\widehat R_i\);
    \item reveal \(m_t\) positions with host-provided token values, then update visited cells after reward arrives.
\end{enumerate}
The scalar overhead is
\[
O(|M_t|)+O(N_t\log N_t),
\qquad \text{memory }O(PBA),
\]
for \(P\) phases, \(B\) confidence bins, and \(A\) auxiliary states. This adds no denoiser call in seven hosts. Prism is initialized separately for each test question because its verifier values exist only inside test-time search. Omni evaluates five complete paths because its historical cross-prompt table fails to transfer. Matched observed costs are reported once in \Cref{tab:multimodal_runtime_cost}; the main paper contains the unified deployment table.

\paragraph{Unseen states and fallback.}
An unseen partial sequence can still use a populated coarse cell. An empty cell has \(\widehat R=0\) and gate \(0\), exactly recovering the host confidence order. \Cref{fig:supp_bucket_support_map} separates cell coverage from selected-decision activation; these should not be conflated.

\begin{figure}[!htbp]
\centering
\includegraphics[width=0.98\linewidth]{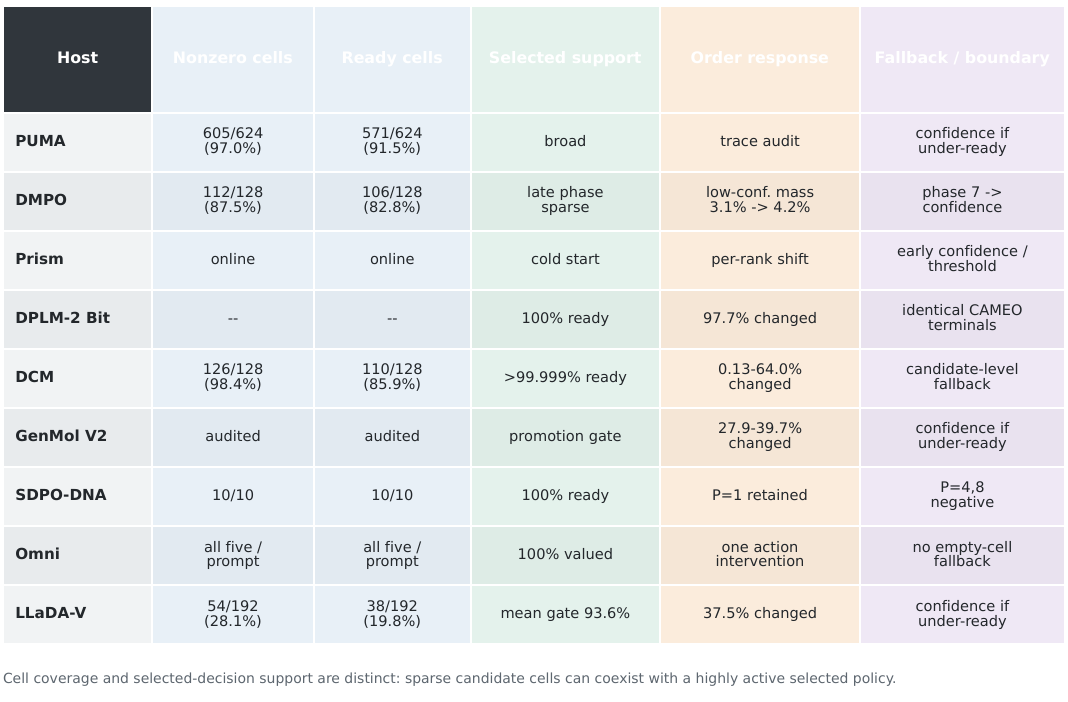}
\caption{Bucket support and fallback. Coverage refers to stored cells; selected support reports the decisions actually used by the controller.}
\label{fig:supp_bucket_support_map}
\end{figure}

The main paper reports all endpoint comparisons once. The remaining subsections therefore focus on additional protocol, uncertainty, mechanism, and negative-result evidence.

\subsection{DPRM-PUMA: Setup, Evaluation, and Interpretation}
\label{app:exp_detail_DPRM_PUMA}

Both runs use TinyGSM (hidden size \(512\), \(14\) layers, \(8\) heads, block \(256\), batch \(32\), learning rate \(3\times10^{-4}\), EMA \(0.9999\)) and the same progressive horizon. DPRM adds \(16\) confidence bins, \(\beta=1\), gate \((2\mathrm{k},60\mathrm{k},128)\), and shortlist \(N_t=\min\{64,\max(8,4m_t)\}\). At the shared \(2.00\)M EMA endpoint, all \(1{,}319\) GSM8K questions are decoded at temperature \(0\) and reveal width \(3\). The saved per-question outputs and reveal events support the paired interval \([1.90,6.90]\)pp and the trace audit below.

\subsection{Reasoning-Behavior Trace Audit}
\label{app:reasoning_behavior_trace_audit}

We audit saved reveal events rather than inferred chain-of-thought. PUMA supplies \(1{,}319\) paired GSM8K traces plus a \(512\)-example same-checkpoint controller probe; DMPO supplies a \(64\)-example fixed-checkpoint probe over confidence, DPRM, entropy, random, gate, and count controls.

EOT-only commitments are removed. For each content-bearing step, we compute adjacency, selected-position span, nonlocal rate (span at least eight), centroid backfill, and first-digit timing. These are order statistics: monotone reveal traces cannot establish literal token replacement or hidden-logit self-correction. The main paper reports the five paired aggregate shifts; the supplement below records category concentration and the causal boundary.

Paired intervals exclude zero for span \([+0.820,+1.102]\), nonlocal rate \([+0.0338,+0.0442]\), adjacency \([-0.1319,-0.1215]\), backfill \([+0.0568,+0.0658]\), and first-number timing \([+0.0130,+0.0313]\).
The same directions hold within the (169) DPRM-only wins, so the aggregate contrast is not solely produced by examples on which both methods fail.

The largest PUMA gains occur on comparison (\(+5.00\)pp), fraction/percent (\(+4.96\)pp), many-number (\(+4.40\)pp), and aggregation (\(+3.77\)pp) problems. This supports broader context before arithmetic commitment, not cleaner left-to-right completion.

\paragraph{Causal boundary.}
On the DPRM-PUMA checkpoint, top-\(k\) and DPRM Soft-BoN are behavior-identical on the \(512\)-example probe. PUMA therefore links the trained pipeline to the trace shift but does not isolate the decode controller. DMPO supplies that isolation: at one fixed checkpoint, DPRM lowers adjacency by \(0.1100\), raises backfill by \(0.0821\), raises nonlocal rate by \(0.1453\), and widens span by \(1.2244\). Entropy and random order are even more nonlocal but less accurate, so neither uncertainty nor nonlocality alone explains the gain.

Buckets share only the reward correction: token identity and partial context still enter through the host logits. The exact theory conditions on the full state and selected position; the bucket approximation does not.

\subsection{DPRM-DMPO: Setup, Pass@K, and Difficulty Analysis}
\label{app:exp_detail_DPRM_DMPO}
\paragraph{Details of DPRM-DMPO}

Training remains teacher-forced: DPRM changes the visible-state order, while WDCE labels remain the ground-truth tokens. The host specialization is:

\begin{algorithm}[!htbp]
\caption{DPRM specialization for DMPO training and decoding}
\label{alg:dprm-dmpo-compact}
\begin{algorithmic}[1]
\REQUIRE DMPO model and rollouts, phases \(K\), bins \(B\), reward \(R\), gate schedule, reveal budget \(m_t\)
\STATE Initialize bucket counts \(N_{\phi,b}\) and log-moment sums \(S_{\phi,b}\)
\FOR{each rollout refresh}
    \STATE Generate DMPO completions; compute host rewards and WDCE weights
    \FOR{each reuse update and sample}
        \STATE Run the denoiser; compute WDCE on masked positions
        \STATE Map each masked position \(i\) to phase \(\phi\), confidence \(p_i\), and bin \(b_i\)
        \STATE Set \(s_i=\log p_i+\eta(\phi,b_i)\widehat R_{\phi,b_i}\)
        \STATE Sample the host-confidence shortlist and select top-\(m_t\) positions under \(s_i\)
        \STATE Teacher-force selected positions and add \(\exp(\beta R)\) to their cells
    \ENDFOR
    \STATE Update the DMPO model using WDCE
\ENDFOR
\STATE At decode time, load the table and replace teacher-forced tokens by host provisional tokens
\end{algorithmic}
\end{algorithm}

\paragraph{Protocol.}
Both methods start from LLaDA-8B-Instruct and use \(5{,}000\) updates, eight reveal phases, \(128\) diffusion steps, temperature \(0.2\), eight rollout generations, and the same reveal budget. Progressive DMPO uses confidence order; DPRM adds \(16\) bins, gate \((500,2000,128)\), and shortlist \(N_t=\min\{32,\max(8,4m_t)\}\). Retained success matrices contain \(32\) samples for each of \(500\) MATH and \(5{,}120\) Countdown problems.

\begin{figure}[H]
\centering
\includegraphics[width=\linewidth]{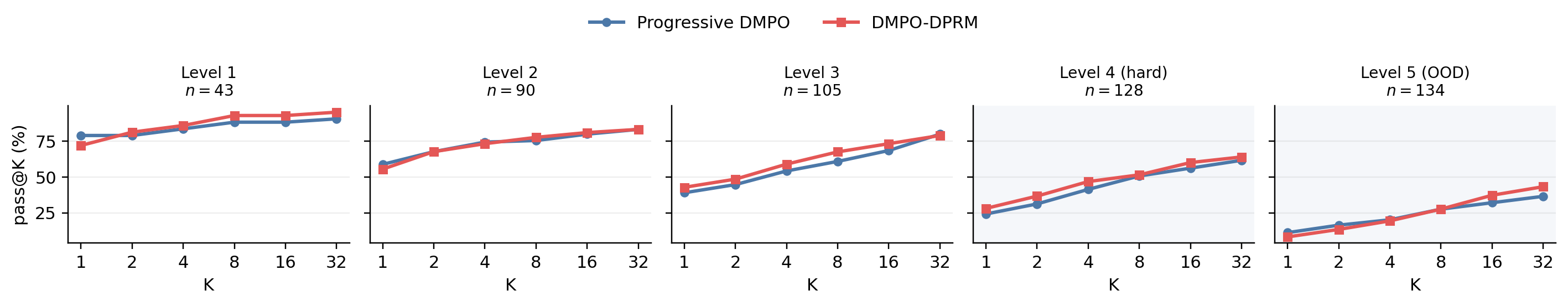}
\caption{MATH pass@$K$ curves by standard difficulty level (1--5). DPRM-DMPO provides its largest gains on the higher-difficulty tiers.}
\label{fig:passk_math_levels}
\end{figure}

Decoding uses \texttt{fast\_dllm}, block length \(32\), generation length \(256\), and the native transfer schedule (initially two tokens per local step). The checkpoint table is loaded with the full gate. Learning rates are \(3\times10^{-6}\) for MATH and \(1\times10^{-6}\) for Countdown; all other task-specific optimizer fields are shared within each pair. The comparison is end to end: it does not separate train-time from decode-time DPRM.

\begin{table}[H]
\centering
\small
\caption{Paired deltas in mean pass@$K$ over \(K\in\{1,2,4,8,16,32\}\), in percentage points. Intervals use \(5{,}000\) bootstrap resamples over shared problems.}
\label{tab:dmpo_bootstrap_deltas}
\begin{tabular}{lc}
\toprule
Task & DMPO-DPRM $-$ Progressive DMPO \\
\midrule
MATH & $+2.10~[-0.33,\ 4.53]$ \\
Countdown & $+1.67~[0.82,\ 2.55]$ \\
\bottomrule
\end{tabular}
\end{table}

The Countdown interval excludes zero; the MATH interval does not. \Cref{fig:passk_math_levels,fig:passk_countdown_levels} locates the larger gaps in MATH levels 4--5 and the five-/six-operand Countdown subsets. Hard Countdown pass@32 rises from (47.9) to (60.0); OOD gains are likewise concentrated at larger (K).

\begin{figure}[H]
\centering
\includegraphics[width=\linewidth]{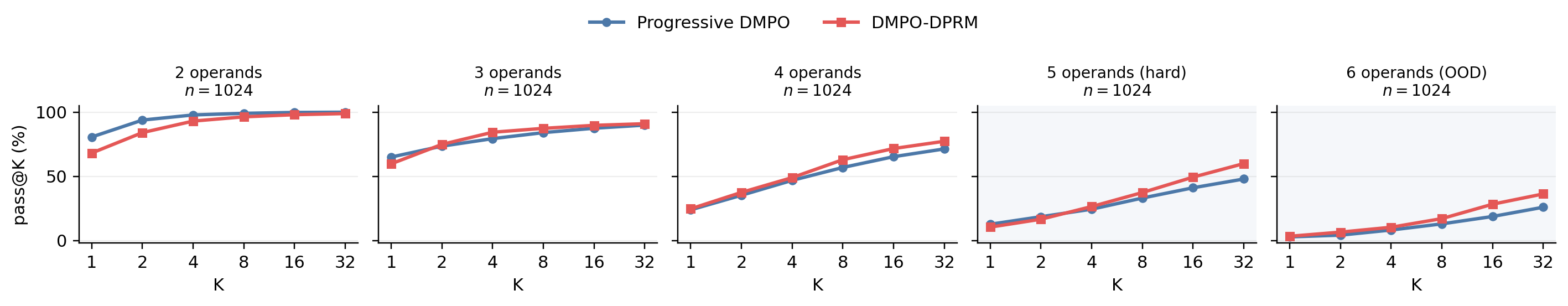}
\caption{Countdown pass@$K$ curves by number of target operands (2--6). The largest DPRM gaps occur on the five-operand hard and six-operand OOD subsets.}
\label{fig:passk_countdown_levels}
\end{figure}

These results support a narrower claim than uniform dominance: reward-guided order helps most where confidence under-covers successful trajectories.

\subsection{DPRM-Prism: Test-Time Search and Cost-Quality Trade-off}
\label{app:exp_detail_DPRM_Prism}

\paragraph{Protocol.}
Both rows use LLaDA-2.0-mini and the same Prism HTS/SVF scaffold \citep{bie2025llada2,bai2026prism}: branching \(N=16\), four final survivors, pruning interval \(3\), block and decode length \(32\), generation length \(256\), and temperature \(0.7\). DPRM changes confidence top-\(k\) to Soft-BoN with \(8\times16\) cells, gate \((6,22,64)\), and shortlist \(8\)--\(64\). The full \(1{,}319\)-question comparison uses one run per method, so intervals measure dataset rather than run-to-run variation.

\begin{figure}[h]
\centering
\includegraphics[width=0.72\linewidth]{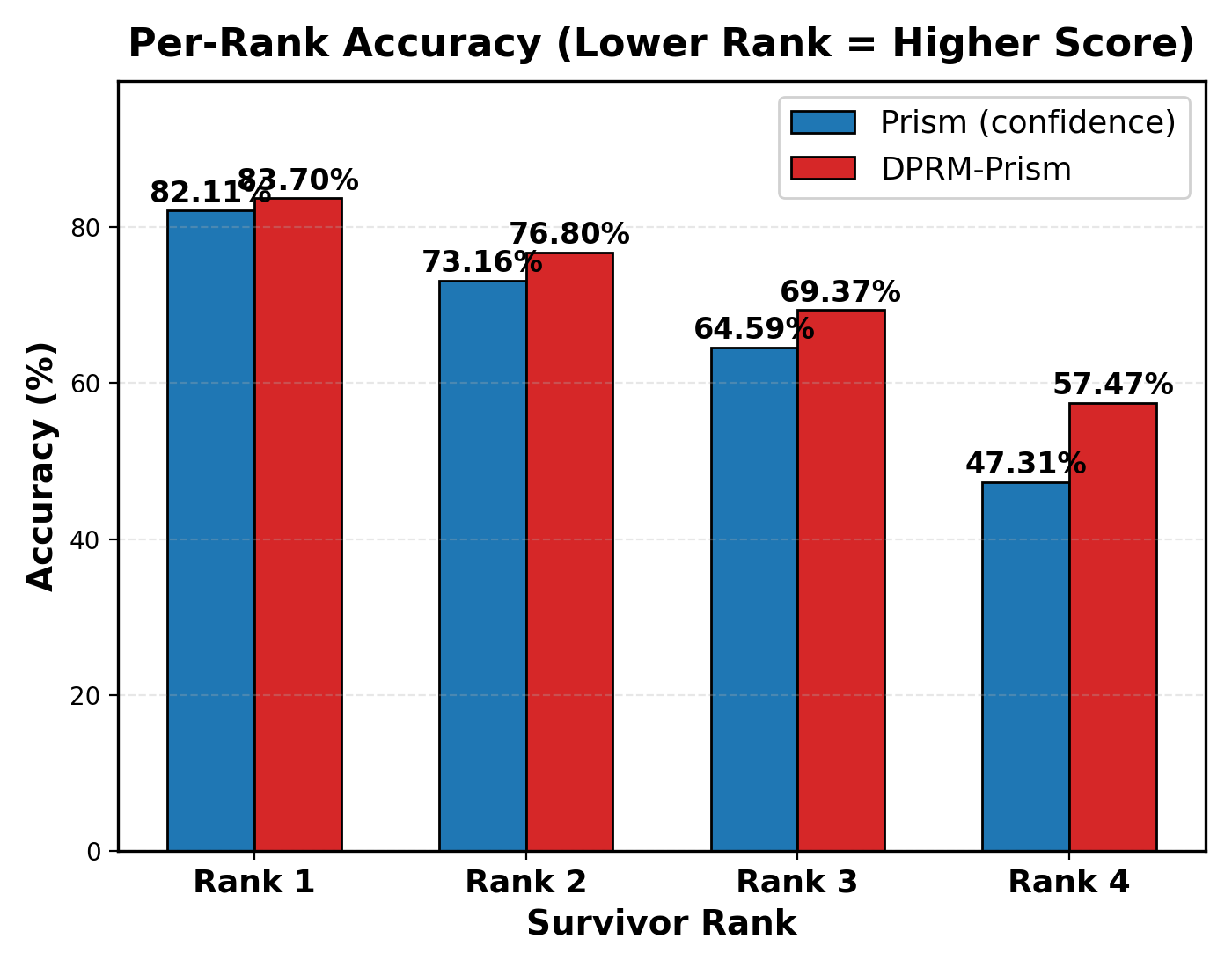}
\caption{GSM8K accuracy by survivor rank. The largest gap is at rank 4 (\(+10.2\)pp).}
\label{fig:prism_per_rank}
\end{figure}

\begin{figure}[h]
\centering
\includegraphics[width=\linewidth]{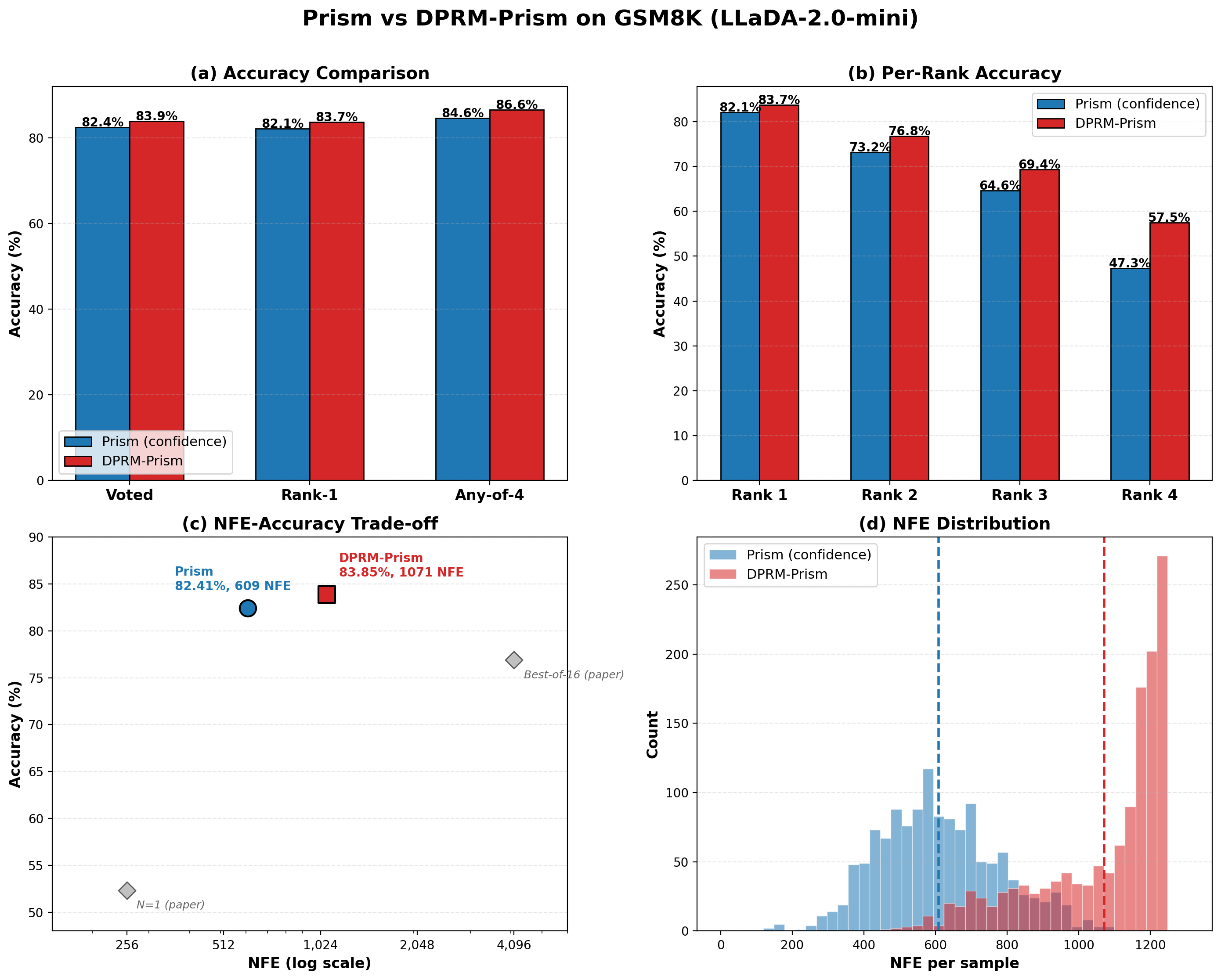}
\caption{Prism quality and cost. Left: voted accuracy versus mean NFE. Right: per-question NFE distributions.}
\label{fig:prism_nfe}
\end{figure}

Mean NFE rises \(609\to1{,}071\) while each SVF call remains \(29\) NFE. The \(250\)-question sweep below raises the confidence-only budget; it improves, but remains below DPRM on rank-1 and any-of-4 accuracy.

\begin{table}[H]
\centering
\small
\caption{Prism cost control on \(250\) GSM8K questions. Larger-\(N\) confidence spends more compute without matching DPRM accuracy.}
\label{tab:prism_compact_cost_quality}
\begin{tabular}{lccc}
\toprule
Policy & Rank-1 & Any-of-4 & Mean NFE \\
\midrule
Confidence Prism & 0.828 & 0.856 & 610.6 \\
Confidence Prism, larger \(N\) & 0.844 & 0.872 & 745.2 \\
DPRM, warmup \(0.2T\) & \bestcell{0.860} & \bestcell{0.896} & 1061.1 \\
\bottomrule
\end{tabular}
\end{table}

Prism is therefore a quality--cost result, not a compute-free gain. The reason is structural rather than a missing implementation optimization: Prism is applied only at test time, and its SVF reward is observed only after a candidate trajectory has been evaluated. The DPRM controller is initialized afresh for each question, starts from confidence, and fills its cells from that question's verified search trajectories. There is no pretraining or post-training stream from which to preload the table. The larger-\(N\) confidence control in \Cref{tab:prism_compact_cost_quality} shows that the gain is not reproduced by spending a smaller amount of extra NFE on confidence alone.

\subsection{Protein Ordering with DPLM-2 Bit}
\label{app:exp_detail_protein}
\subsubsection{DPRM-DPLM: Setup and Protein Metrics}
\label{app:exp_detail_DPRM_DPLM}

The comparison uses one confidence reference and two \(5\)k-step DPRM controllers with separate sequence and structure rewards. All share the model, loss, data, optimizer, reveal budget, \(8\times16\) cells, and gate \((500,2000,128)\). Forward folding uses all \(163\) CAMEO 2022 targets. The predeclared CoGen-200 criterion is
\[
U_{\mathrm{bal}}=\sqrt{\mathrm{TM}\cdot(\mathrm{pLDDT}/100)}.
\]
A variant opens the \(50\)-sample, five-length confirmation only if the \(95\%\) lower bound of its gain is positive and neither component falls by more than \(0.02\).

\begin{table}[H]
\centering
\small
\caption{DPLM terminal-sensitivity audit. CoGen-200 is the predeclared development gate; balanced deltas use paired bootstrap intervals.}
\label{tab:dplm_terminal_audit}
\begin{tabular}{lcccccc}
\toprule
Method & CAMEO RMSD & CAMEO TM & CoGen TM & Norm. pLDDT & Balanced & Balanced delta \\
\midrule
Progressive-DPLM & 35.4667 & 0.3066 & 0.6303 & 0.8583 & 0.7355 & -- \\
DPRM weighted & 35.4667 & 0.3066 & 0.5978 & 0.8428 & 0.7097 & $-0.0258~[-0.1025,0.0166]$ \\
DPRM Tcheby. & 35.4667 & 0.3066 & 0.6270 & 0.8680 & 0.7377 & $+0.0022~[-0.0043,0.0091]$ \\
\bottomrule
\end{tabular}
\end{table}

Both controllers are fully active and change \(97.7\%\) of structure-order decisions, yet all \(163\) CAMEO terminal files are identical. Tchebycheff has a small positive CoGen point estimate, but its interval crosses zero; weighted sum is worse. Neither passes the gate, so confirmation is not opened. Earlier runs with different controller designs are excluded from this result and are not treated as replications of the selected configuration.

\subsection{Multimodal Scope and Shared Controls}
\label{app:multimodal_controls}

The multimodal extension uses two hosts because they expose two different order targets.
Omni-Diffusion\footnote{\url{https://github.com/VITA-MLLM/Omni-Diffusion}} orders visual codebook-token positions during text-to-image generation.
LLaDA-V \citep{you2025llada} orders generated answer-token positions under image conditioning.

Both hosts leave the token-value kernel fixed. Omni estimates prompt-local confidence-rank position values online; LLaDA-V loads a frozen confidence/answer-structure table. The main paper reports their endpoint metrics. This section adds split construction, token-order controls, order traces, and the negative AI2D diagnostic.

\subsection{DPRM-Omni: Visual-Token Ordering}
\label{app:exp_detail_DPRM_Omni}

Omni orders visual-code positions after applying the host's top-\(p\), repetition, and position penalties. Its native score is negative entropy,
\[
c_i(s_t)=-H[p_\theta(\cdot\mid s_t,i)]
\]
and the decoder commits its maximum. The 260-token template contains 256 visual codes and four fixed format tokens. Every policy keeps the provisional argmax token and changes only its position.

We first attempted to reuse token-order values learned from earlier prompts. Prompt-grouped five-fold validation selected an \(8\)-phase, \(4\)-confidence-bin table, but its RMSE was \(0.03562\), worse than the zero predictor's \(0.03441\), and its held-out Pearson and Spearman correlations were \(-0.067\) and \(-0.022\). The next-best tables showed the same near-zero or negative correlation. On a separate \(96\)-prompt check, the best surviving configuration changed only eight outputs and produced a mean CLIP-L/14 difference of \(+0.00018\), with paired interval \([-0.00065,+0.00106]\). Thus phase and confidence were not sufficient to transfer the value of a visual position across unrelated canvases.

The reported controller avoids this failed transfer. At step \(96\), five positions are fixed: the native confidence choice and confidence-rank quantiles \(0.70,0.85,0.90,0.95\). Each starts from the same current canvas and receives the same deterministic continuation. The deployed score is
\[
g(i;s_{96})=c_i(s_{96})+
\gamma\frac{C_{\mathrm{L/14}}(X_T^i)-C_{\mathrm{L/14}}(X_T^{i_0})}{0.03}.
\]
CLIP-L/14 supplies the reward; CLIP-B/32 is held out from selection. Uniform-order and reward-only controls reuse the five paths. Rank, step, and reward scale are fixed before \(128\) development prompts; \(\gamma=8\) is then frozen for \(512\) disjoint confirmation prompts. The release stores shared-canvas hashes, selected coordinates, scores, and paths for all prompts.

The beach and Owl Family trajectories already appear in the main paper and are not repeated here. \Cref{fig:omni_curated_four_policy} instead gives independent four-policy examples for visual inspection; it does not enter controller selection or aggregate evaluation.

\begin{figure}[H]
\centering
\includegraphics[width=0.98\linewidth]{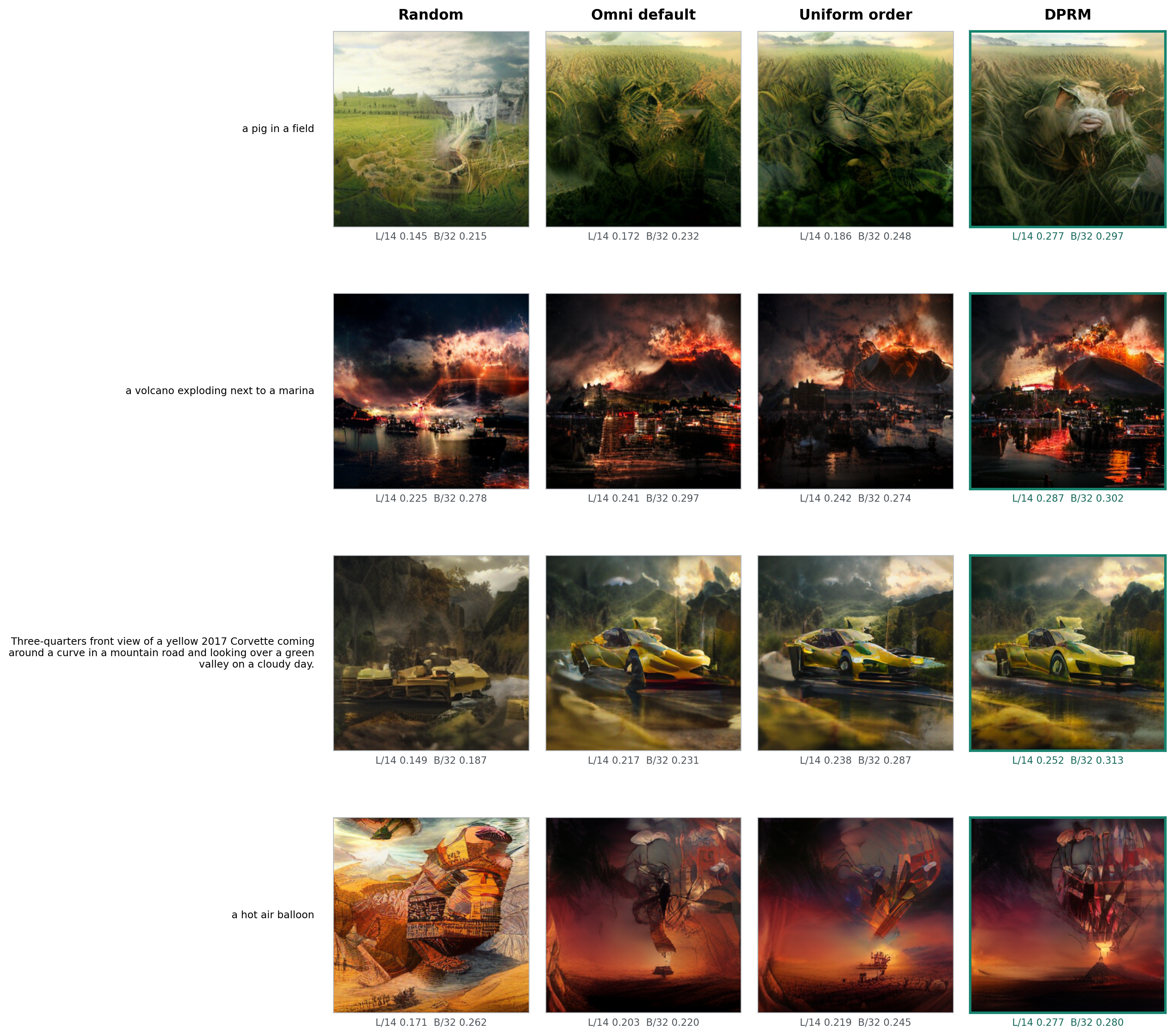}
\caption{Four-policy Omni-Diffusion comparisons from the frozen confirmation split. Random order, confidence, the uniform-order control, and DPRM use the same prompts and initial noise. These examples are for visual inspection only; controller selection uses the full confirmation set.}
\label{fig:omni_curated_four_policy}
\end{figure}

\subsection{DPRM-LLaDA-V: Image-Conditioned Text Ordering}
\label{app:exp_detail_DPRM_LLaDAV}

LLaDA-V orders four generated answer slots under image conditioning; provisional token values and target-normalized evaluation remain fixed. Because RealWorldQA mixes choice, count, and open answers, the table uses one phase, eight confidence bins, and prompt format \(\times\) candidate-EOT \(\times\) relative position (\(24\) auxiliary cells). Documents \(0\)--\(127\) fit the table, \(128\)--\(255\) select among six settings, and \(256\)--\(764\) form the strict test. The selected controller changes \(37.5\%\) of reveal orders, with mean gate \(0.936\). A symmetric zero-count normalization affects four responses; without it the overall direction remains positive (\(+1.18\)pp).

\begin{figure}[!htbp]
\centering
\includegraphics[width=0.98\linewidth]{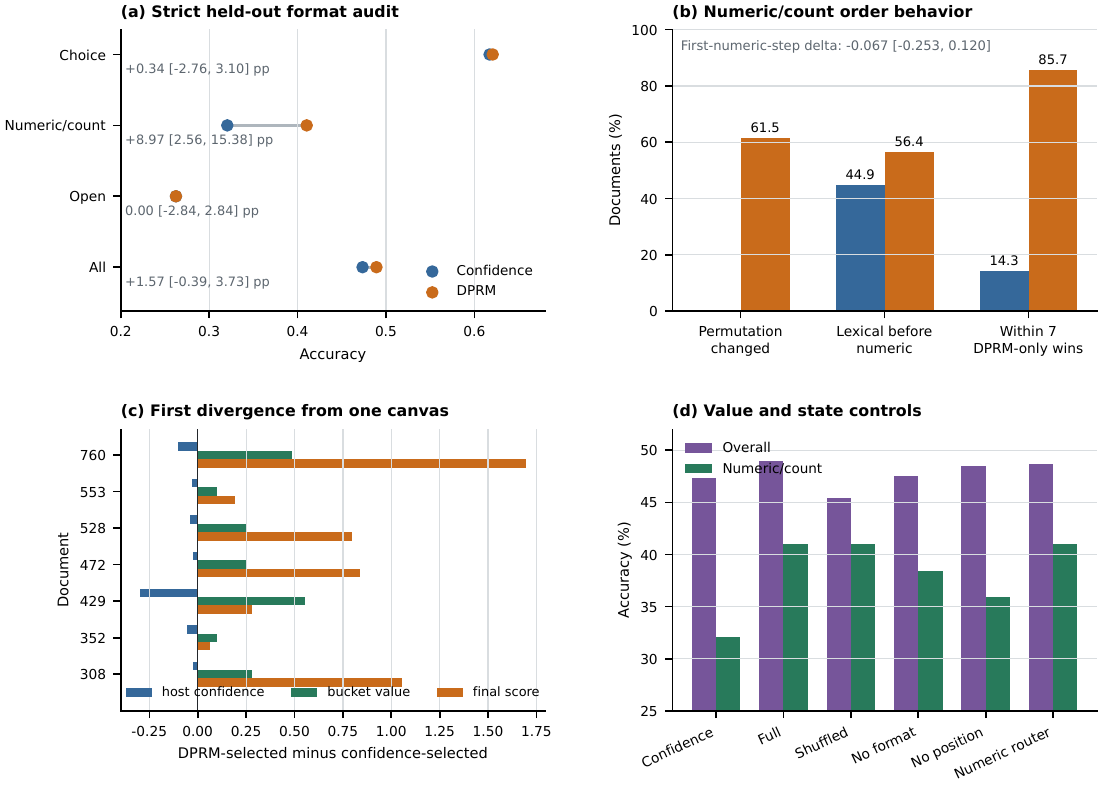}
\caption{RealWorldQA mechanism audit: format effects, numeric-token order, same-canvas score reversal, and state/value controls.}
\label{fig:supp_lladav_mechanism_audit}
\end{figure}

\paragraph{Mechanism readout.}
\Cref{fig:supp_lladav_mechanism_audit}(a) localizes the positive interval to numeric/count questions: all seven discordant outcomes favor DPRM. Panel (b) shows that \(48/78\) reveal permutations change and that six of seven DPRM-only wins expose lexical context before the first number, versus one under confidence. The mean first-number timing remains near zero, so the controller does not simply delay every number.

Panel (c) isolates the first divergence from an identical canvas. DPRM chooses the lower-confidence position in all seven wins (mean confidence difference \(-0.083\)); its mean bucket-value advantage \(+0.291\) reverses the final score by \(+0.702\). Panel (d) shows that shuffling value-to-cell assignments damages overall accuracy, while removing format or position weakens the numeric point estimate. These controls support a covered residual order class, not a uniform RealWorldQA improvement.

\paragraph{Strict held-out visual examples.}
\Cref{fig:lladav_realworldqa_gallery_1,fig:lladav_realworldqa_gallery_2} shows all seven DPRM-only wins in the prompt-defined numeric/count class.
No DPRM losses occur in this class.
Each row pairs the official RealWorldQA image with the matched confidence-order answer, DPRM answer, and target; the examples are the same records used by the paired audit above.

\begin{figure}[p]
\centering
\includegraphics[width=0.84\linewidth]{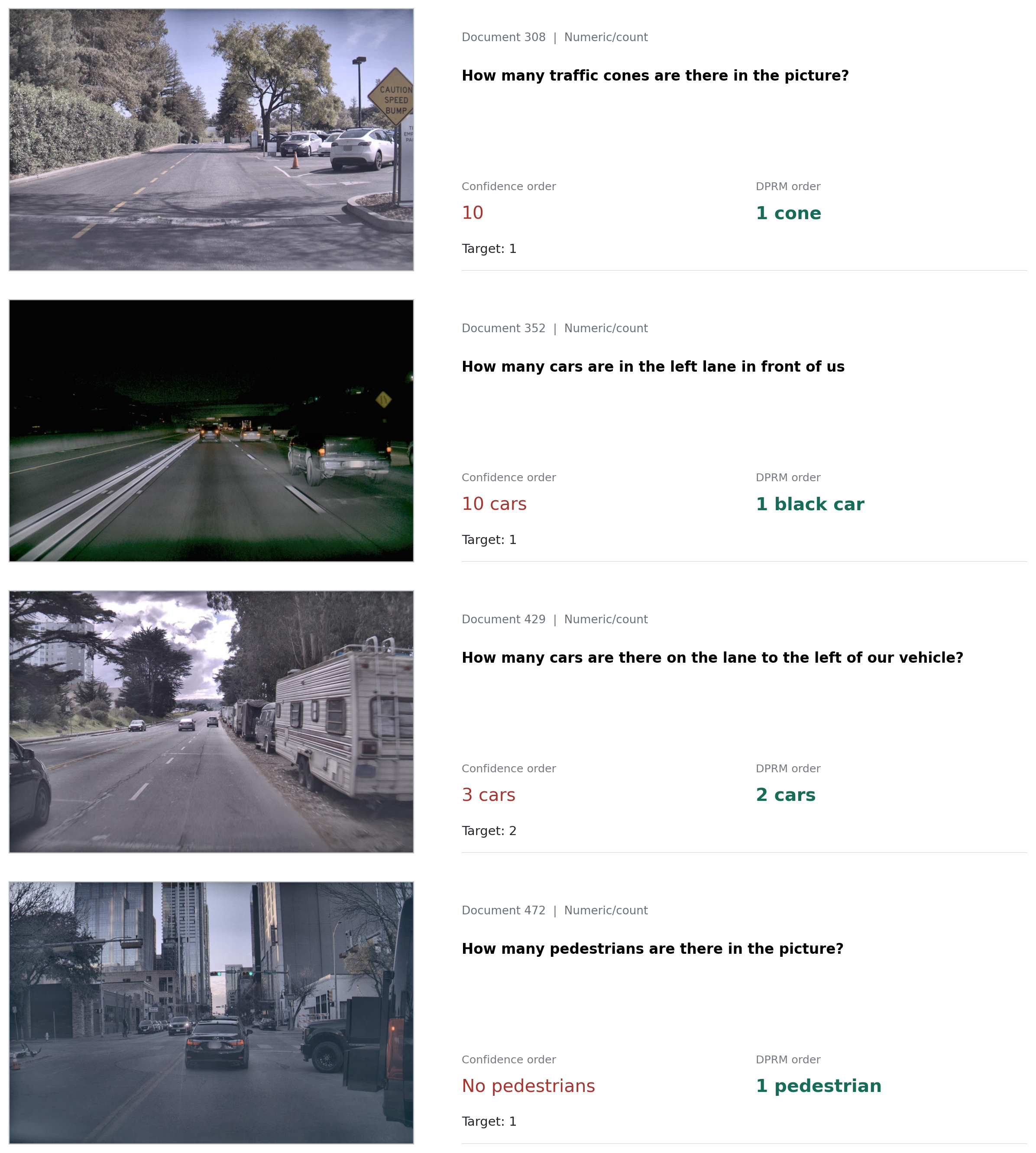}
\caption{RealWorldQA numeric/count wins (1/2). \textbf{DPRM exposes lexical context before committing the count.}}
\label{fig:lladav_realworldqa_gallery_1}
\end{figure}

\clearpage
\begin{figure}[p]
\centering
\includegraphics[width=0.94\linewidth]{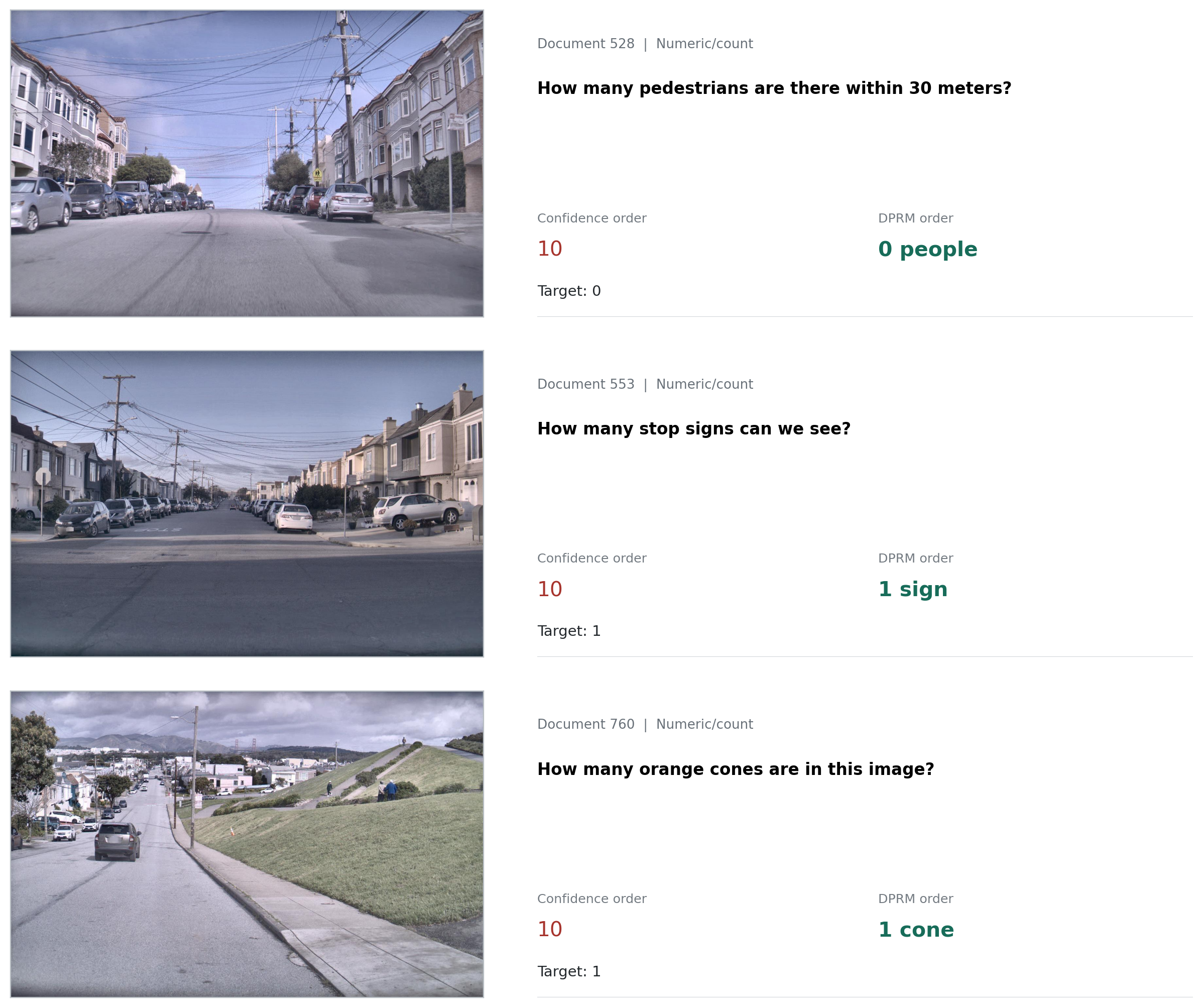}
\caption{RealWorldQA numeric/count wins (2/2). Together the two pages show all seven DPRM-only wins and no loss in this prompt-defined class.}
\label{fig:lladav_realworldqa_gallery_2}
\end{figure}
\clearpage

\paragraph{Preregistered AI2D diagnostic.}
We separately use AI2D documents (0)--(127) to fit tables, (128)--(255) to select among phase counts ({1,4}) and guidance scales ({1,2,4,8}), and (256)--(499) for one independent confirmation. Development selects one phase, eight confidence bins, candidate-EOT state, four relative-position cells, guidance (8), readiness (4), and gate ((0,4)). The controller is active on confirmation, changing (65.2

\subsection{DPRM-DCM: Single-Cell Gene-Token Ordering}
\label{app:exp_detail_DPRM_DCM}

We use the public Dentate Gyrus data from DCM \citep{bhattacharya2026discrete}: \(2{,}930\) cells, the top \(5{,}000\) variable genes, log transform, quantile-binned nonzero expression, and a fixed \(90/10\) split with \(293\) validation cells.

\paragraph{Controller and protocol.}
One SEDD-style Progressive-DCM checkpoint is frozen. On \(256\) training cells, shared states at steps \(0,8,16,24\) branch to confidence, predicted-nonzero, predicted-zero, and random position choices; every branch then uses the same confidence decoder. A table keyed by four phases, \(16\) confidence bins, and provisional zero/nonzero status stores the terminal vector in \Cref{tab:scientific_scalarization_config}. Guidance is selected from \(\{0.5,1,2,4\}\) on another \(96\) training cells; scale \(2\) passes the written utility and order-change gate. Validation then decodes all cells from all-mask for \(32\) steps, with four samples per cell and \(5{,}000\) paired-bootstrap resamples. The table is read-only on validation.
The data split uses seed \(42\); controller development uses \(20260812\). The formal decode and two independent replication decodes use \(20260812\), \(20260815\), and \(20260816\), respectively; each retains its own paired bootstrap.

\paragraph{Results.}
All three declared directions have positive \(95\%\) intervals: recovery focus improves nonzero recovery by \(0.00155\), MAE focus reduces nonzero MAE by \(0.00575\), and zero focus improves zero accuracy by \(0.000173\), each against the opposite endpoint. Selected-decision readiness is \(1.000\). All three decode seeds reproduce every direction; \Cref{tab:dcm_preference_replications} and the native-value table report the complete values.

\IfFileExists{artifacts/dcm_preference_replications.tex}{%
\begin{table}[H]
\centering\small
\caption{DCM preference response across decode seeds. Rows use separate paired cell bootstraps; positive follows the requested preference.}
\label{tab:dcm_preference_replications}
\begin{tabular}{lccc}
\toprule
Decode seed & Recovery response & MAE reduction & Zero response \\
\midrule
20260812 & $+0.00155~[+0.00141,+0.00168]$ & $+0.00575~[+0.00536,+0.00613]$ & $+0.000172~[+0.000140,+0.000205]$ \\
20260815 & $+0.00154~[+0.00141,+0.00167]$ & $+0.00560~[+0.00521,+0.00598]$ & $+0.000161~[+0.000129,+0.000192]$ \\
20260816 & $+0.00156~[+0.00142,+0.00170]$ & $+0.00557~[+0.00518,+0.00595]$ & $+0.000174~[+0.000143,+0.000208]$ \\
\bottomrule
\end{tabular}
\end{table}

}{}

\subsection{DPRM-GenMol: Molecular SAFE-Token Ordering}
\label{app:exp_detail_DPRM_GenMol}

All runs use GenMol V2 with bracket-SAFE \citep{lee2025genmol}. The checkpoint, tokenizer, denoising loss, sampler temperature, and RDKit metrics are fixed.

\paragraph{Controller and protocol.}
The rewards and three principal preferences are in \Cref{tab:scientific_scalarization_config}; weighted-sum and intermediate \((.80,.20)/(.20,.80)\) points are path controls. Each controller uses \(5{,}000\) steps, random warmup, eight phases, \(16\) confidence bins, \(\beta=1\), gate \((500,2000,128)\), and shortlist \(N_t=\min\{64,\max(8,4m_t)\}\). An auxiliary key groups provisional SAFE tokens into eight fixed chemistry/syntax classes. Under-ready cells fall back to confidence.

\paragraph{Evaluation protocol.}
Each checkpoint generates \(1{,}000\) molecules at temperature \(1.0\), randomness \(0.3\), and minimum added length \(60\). We use \(5{,}000\) independent-bootstrap resamples because generated molecules are not paired. QED and SA benefit test the declared reward; validity, canonical uniqueness, and Morgan diversity expose external trade-offs and are never folded into the terminal score.

\IfFileExists{artifacts/scientific_preference_native_values.tex}{%
\begin{figure}[H]
\centering
\includegraphics[width=0.98\linewidth]{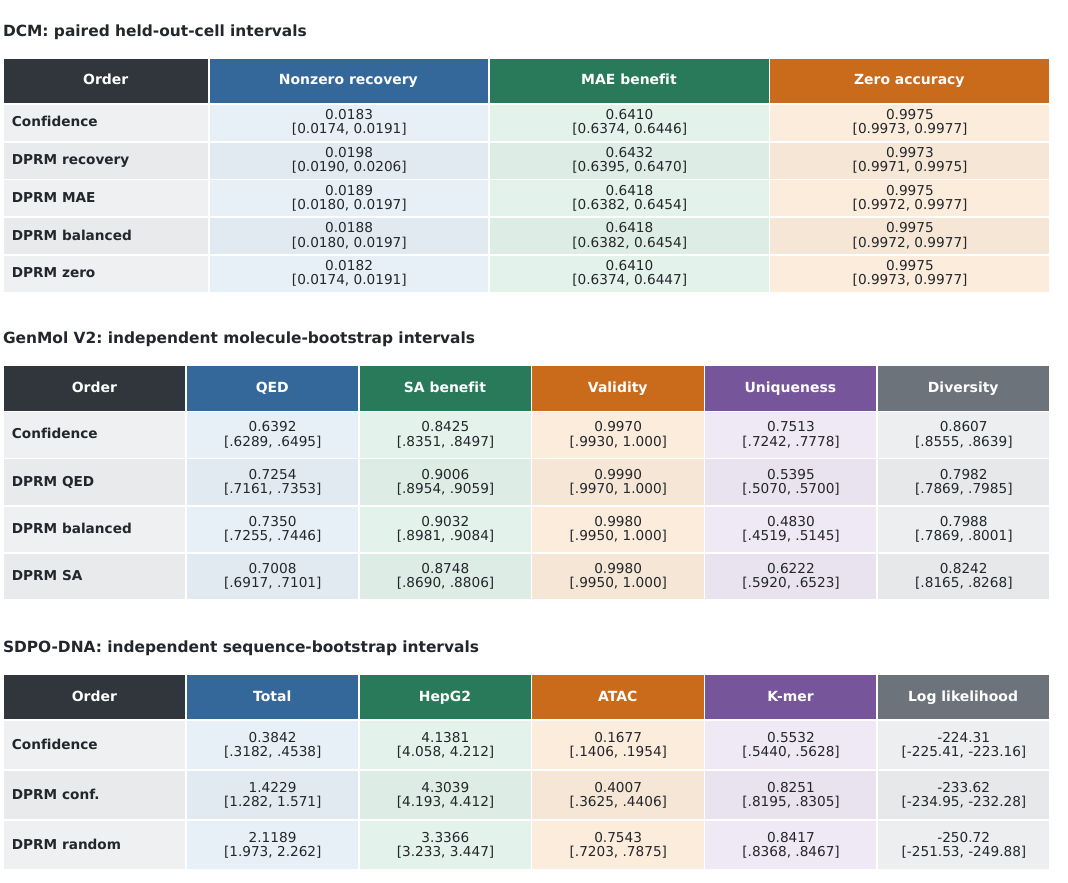}
\caption{Scientific uncertainty audit. DCM uses paired held-out-cell intervals; GenMol and SDPO use method-specific sample intervals.}
\label{fig:scientific_preference_native_values}
\end{figure}

}{%
\PackageWarning{DPRM}{Scientific preference native-value table is not yet generated}%
}

\paragraph{Results.}
Balanced DPRM improves QED by \(0.0956\) and SA benefit by \(0.0607\) over confidence, with positive intervals. It lowers uniqueness and diversity, so we report a quality--diversity Pareto path rather than uniform dominance or bidirectional QED--SA control. The three principal controllers alter \(27.9\%/39.7\%/31.8\%\) of reveal rows; the native-value table gives all metrics, intervals, and endpoints.

\subsection{DPRM-SDPO: DNA Reward-Optimized Ordering}
\label{app:exp_detail_DPRM_SDPO_DNA}

This subsection documents the DNA reward-optimization comparison reported in the main-paper SDPO-DNA table. SDPO \citep{wang2025finetuning} fine-tunes discrete diffusion models with reward optimization and includes DNA regulatory-sequence design experiments. Our goal here is narrower than a full SDPO reproduction: we test whether the DPRM ordering module can be inserted into the SDPO code path while keeping the pretrained diffusion checkpoint, reward optimizer, oracle suite, and sequence representation fixed.

\paragraph{Model and variants.}
All runs start from the same public pretrained DNA diffusion checkpoint and use the same SDPO fine-tuning configuration: \(K=2000\), two epochs, learning rate \(10^{-5}\), and SDPO temperature \(0.5\). We compare the original SDPO ordering, confidence-progressive ordering, DPRM-SDPO, and DPRM(random)-SDPO. To measure the effect of the bucket abstraction itself, we sweep the number of phase buckets \(P\in\{1,4,8\}\) while keeping \(10\) confidence bins, \(\beta=1.0\), \((T_{\mathrm{warm}},T_{\mathrm{switch}},N_{\mathrm{ready}})=(100,400,64)\), and Soft-BoN shortlist size \(N_t=64\). The main table reports the \(P=1\) abstraction because it has full selected-decision readiness and the best total metric among the phase variants. DPRM-SDPO starts from confidence ordering before switching to DPRM, whereas DPRM(random)-SDPO starts from random ordering before the same DPRM phase.

\paragraph{Evaluation protocol.}
Each seed generates \(10\) batches of \(64\) DNA sequences per method. We report the GOSAI HepG2 expression score, ATAC success rate, high-expression k-mer Pearson correlation, reference-model log-likelihood, and the product-style total metric used by the SDPO evaluation. The full \(P=1\) family is independently fine-tuned and evaluated for three seeds. Within each seed, \(95\%\) intervals use \(1{,}000\) ordinary bootstrap resamples over generated sequences; seed-level records are not pooled.

\begin{figure}[H]
\centering
\includegraphics[width=\linewidth]{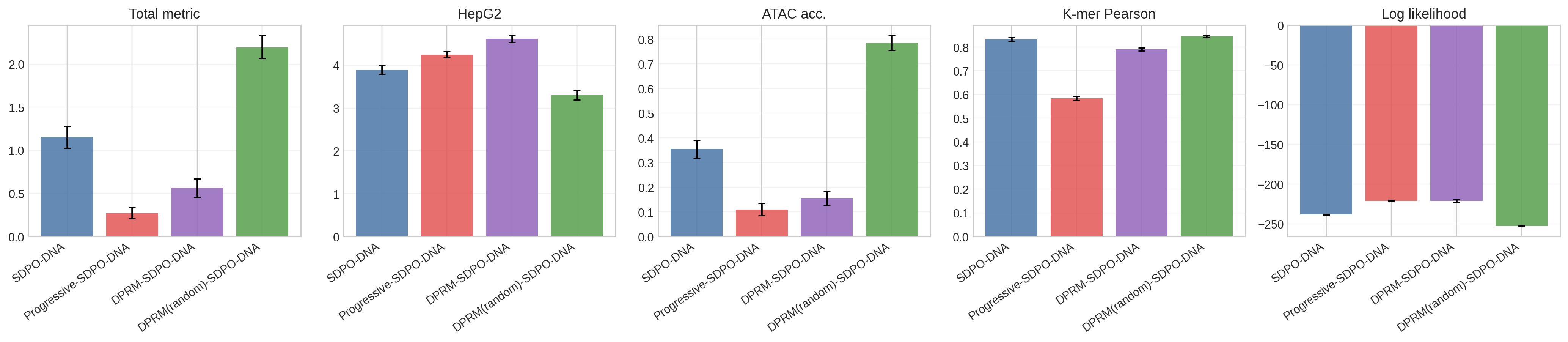}
\caption{DPRM-SDPO DNA reward-optimization metrics. All rows use the same pretrained DNA diffusion checkpoint, SDPO objective, oracle suite, and sampling budget; only the reveal-order controller changes. DPRM(random)-SDPO gives the highest total metric, ATAC success, and k-mer Pearson correlation, while DPRM-SDPO gives the highest HepG2 score.}
\label{fig:sdpo_dna_ordering_metrics}
\end{figure}

\begin{table}[H]
\centering
\small
\caption{SDPO-DNA phase-granularity audit. Finer phase buckets reduce ready coverage and degrade the total metric; the reported SDPO rows therefore use the phase-independent \(P=1\) abstraction.}
\label{tab:sdpo_phase_granularity}
\setlength{\tabcolsep}{4pt}
\begin{tabular}{lccc}
\toprule
Policy & Phases \(P\) & Ready buckets & Total metric \\
\midrule
DPRM(conf.) & 1 & \(10/10\) & 1.423 \\
DPRM(conf.) & 4 & \(34/40\) & 0.592 \\
DPRM(conf.) & 8 & \(60/80\) & 0.317 \\
DPRM(random) & 1 & \(10/10\) & \bestcell{2.119} \\
DPRM(random) & 4 & \(35/40\) & 1.197 \\
DPRM(random) & 8 & \(64/80\) & 1.032 \\
\bottomrule
\end{tabular}
\end{table}

\paragraph{Interpretation.}
The SDPO result depends on bucket granularity. With \(P=1\), DPRM(random)-SDPO produces the highest product-style total, ATAC success, and k-mer correlation, improving all three over the SDPO baseline, while DPRM(conf.)-SDPO achieves the highest HepG2 expression score. The independent three-seed total-utility means are \(2.1287\) for DPRM(random), \(1.3885\) for native SDPO, \(1.1411\) for entropy order, \(0.7788\) for DPRM(confidence), and \(0.6344\) for confidence-progressive order. The same three DPRM(random) checkpoints give \(1.8660\) after shuffling bucket values and \(2.0699\) for gate/count-only scoring. With \(P=4\) or \(P=8\), coverage and selected-decision readiness drop and the total metric becomes worse. Thus the useful abstraction level is task dependent, and learned values help without accounting for the entire effect.

\subsection{Runtime and Sampling-Cost Analysis}
\label{app:runtime_cost}

The main paper reports the normalized deployment count for all nine hosts. Because seconds, NFE, and complete paths are not commensurate, \Cref{tab:multimodal_runtime_cost} retains only the four matched empirical records available in the release.

\begin{table}[H]
\centering
\small
\setlength{\tabcolsep}{3pt}
\begin{tabular}{p{0.14\linewidth}p{0.18\linewidth}p{0.15\linewidth}p{0.15\linewidth}p{0.24\linewidth}}
\toprule
Host & Record & Confidence & DPRM & Reading \\
\midrule
PUMA & wall clock & \(387.5\) s & \(417.5\) s & \(1.08\times\); scalar reranking \\
Prism & mean NFE & \(609\) & \(1{,}071\) & \(1.76\times\); online search \\
Omni & complete paths & \(1\) & \(5\) & \(5\times\); prompt-local values \\
LLaDA-V & wall clock & \(2148\) s & \(2127\) s & \(0.99\times\); same model calls \\
\bottomrule
\end{tabular}
\caption{Matched empirical cost records. Ratios are interpreted only within a row.}
\label{tab:multimodal_runtime_cost}
\end{table}

Reusable tables add only scalar work. Prism and Omni instead estimate values online; their gains are quality--compute trade-offs.

\subsection{Statistical Uncertainty Protocol}
\label{app:stats_protocol}

\Cref{fig:supp_statistical_protocol_map} records the evaluation unit, pairing rule, split, and retained uncertainty artifact for every host. Unless marked otherwise, intervals use \(5{,}000\) percentile-bootstrap resamples at \(95\%\) confidence.

\begin{figure}[!htbp]
\centering
\includegraphics[width=0.98\linewidth]{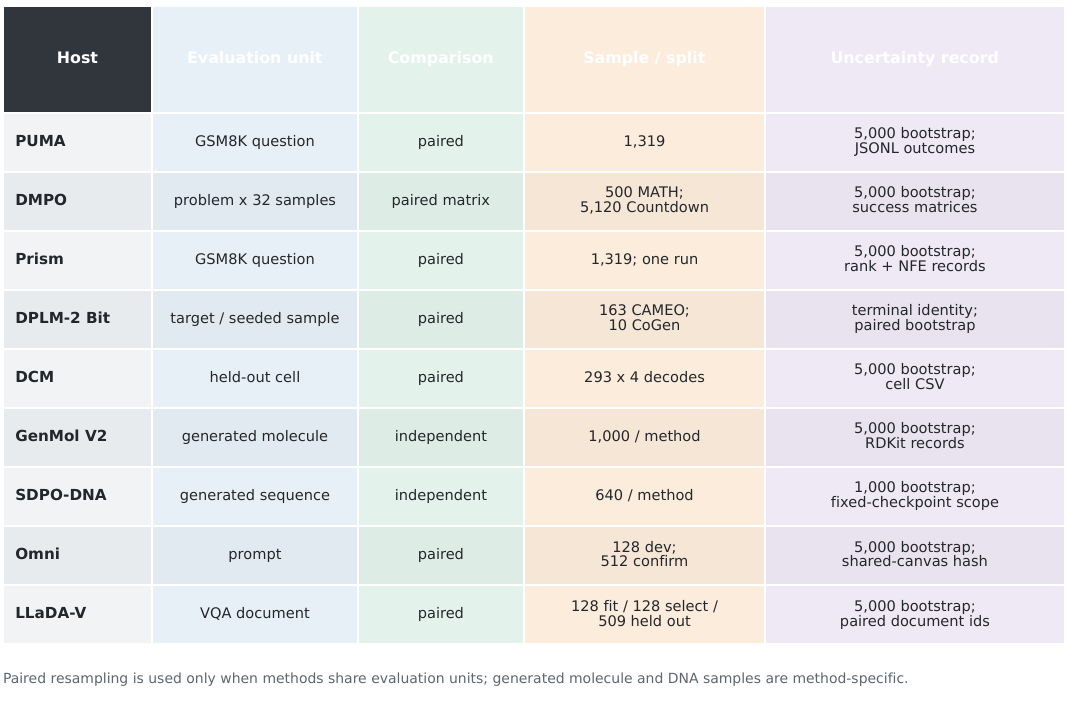}
\caption{Statistical protocol by host. Paired resampling is used only for shared evaluation units.}
\label{fig:supp_statistical_protocol_map}
\end{figure}

For DMPO, each example has a Boolean success row \(S_e\in\{0,1\}^{32}\), and the reported per-example statistic is
\[
\bar p_e=\frac{1}{6}\sum_{K\in\{1,2,4,8,16,32\}}
\mathbf 1\!\left\{\max_{j\le K}S_{e,j}=1\right\}.
\]
Method deltas bootstrap \(\bar p_e^{(B)}-\bar p_e^{(A)}\). Prism uses one run per method, so its interval measures dataset variation, not run-to-run variation. DPLM confirmation remains unopened because its development gate failed. GenMol and SDPO samples are method-specific and therefore use independent resampling.

\subsection{Empirical Diagnostics for the Finite-Sample Theory}
\label{app:theory_observations}

\Cref{app:optimization_advantage} predicts three observable signatures: confidence tracks easier local decisions; reward tilt restores a useful residual order family; and bucket uncertainty shrinks with evidence. Countdown logs phase, confidence, selection, CE, reward, DPRM value, gate, and a gradient proxy for random, confidence, and \(\beta\in\{0.5,1,2\}\). These are tests of assumptions, not estimates of theorem constants.

\paragraph{Optimization trace.}
\Cref{fig:countdown_reward_curve} gives the training-level check. Aligned confidence separates early from random masking, while warmup keeps DPRM close to confidence until bucket values are ready.

\begin{figure}[H]
\centering
\includegraphics[width=0.76\linewidth]{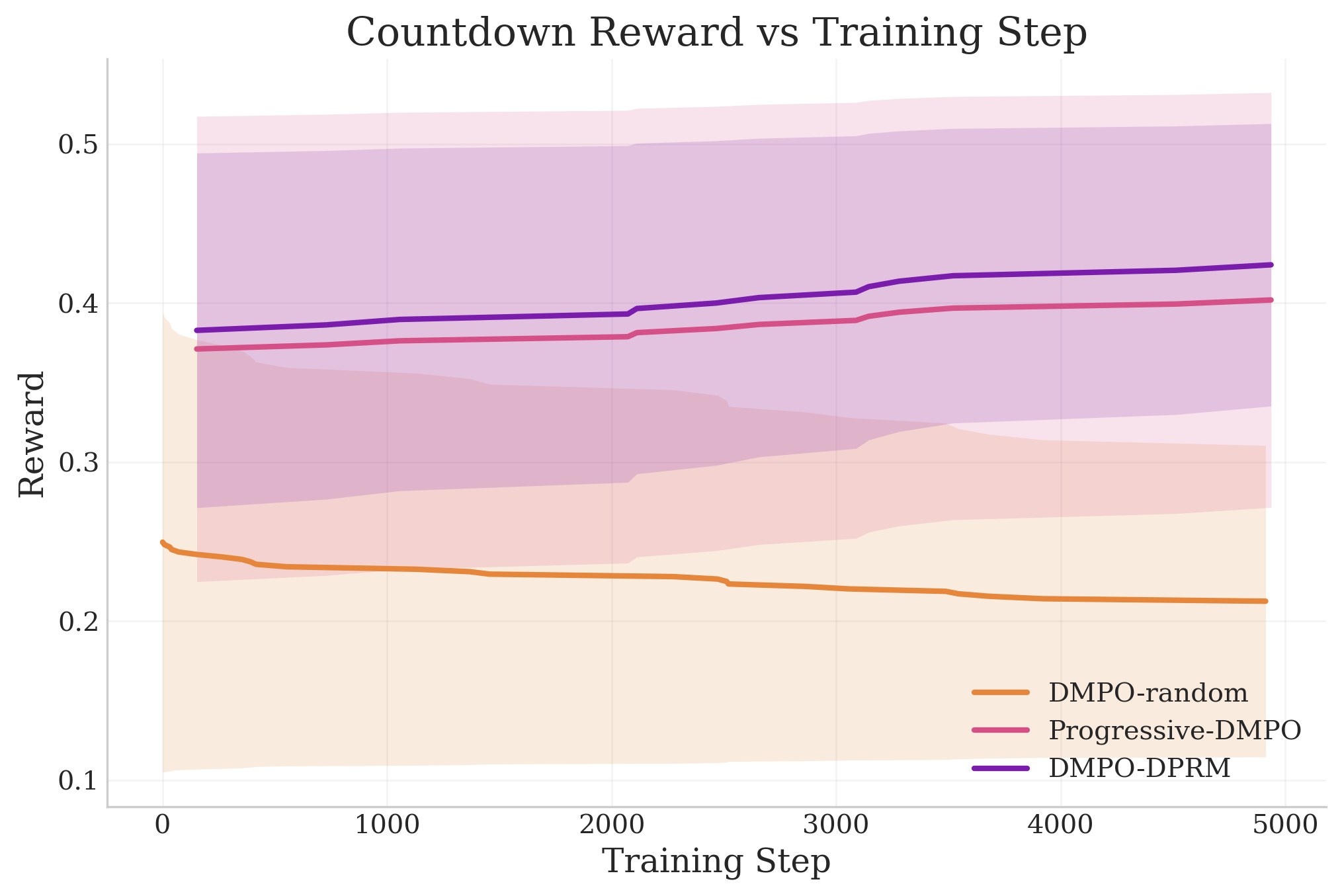}
\caption{Countdown training reward. Bands are logged standard deviations; curves are optimization diagnostics, not evaluation endpoints.}
\label{fig:countdown_reward_curve}
\end{figure}

\paragraph{Token-level local proxy.}
\Cref{fig:countdown_theory_token_diagnostics}(a) shows that confidence and CE have correlation \(-0.94\) to \(-0.97\). Panel (b) then shows lower selected CE for confidence (\(0.477\)) and DPRM-\(\beta{=}1\) (\(0.442\)) than random (\(0.740\)); panel (c) measures the late residual region.

\begin{figure}[H]
\centering
\includegraphics[width=\linewidth]{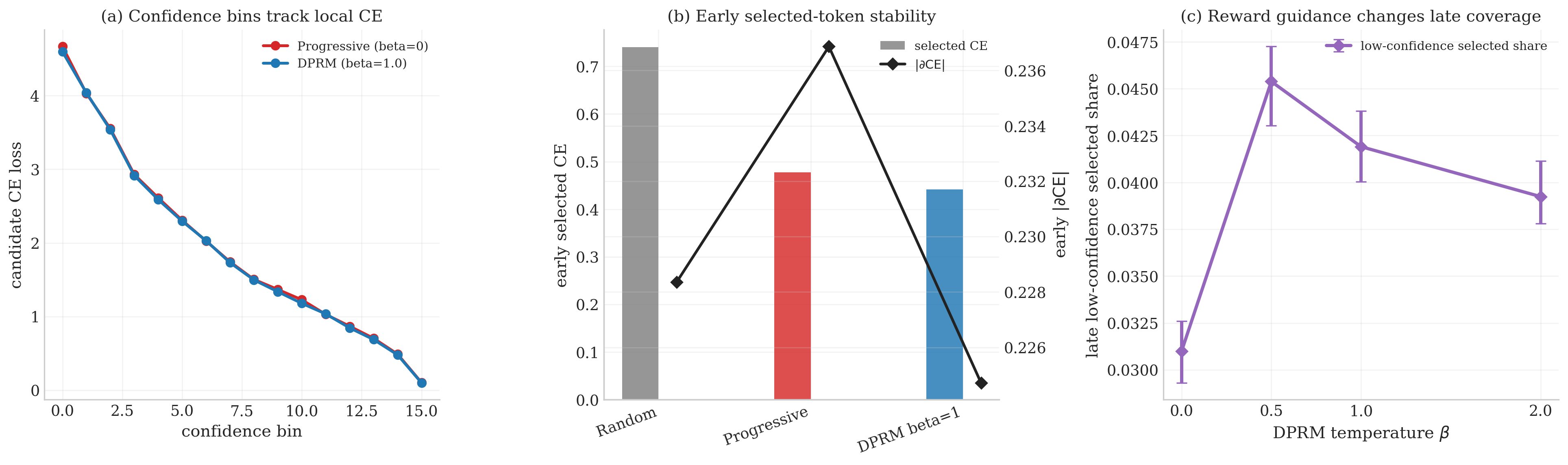}
\caption{Countdown token diagnostics: confidence--CE relation, early selected CE, and late low-confidence coverage.}
\label{fig:countdown_theory_token_diagnostics}
\end{figure}

\paragraph{Low-confidence residual coverage.}
Random order visits many uncertain tokens but has worse selected CE. The relevant comparison in \Cref{fig:countdown_theory_bin_coverage} is confidence versus DPRM: late mass in bins \(0\)--\(5\) changes from \(3.1\%\) to \(4.5\%\), \(4.2\%\), and \(3.9\%\) for \(\beta=0.5,1,2\), while selected DPRM value becomes positive.

\begin{figure}[H]
\centering
\includegraphics[width=\linewidth]{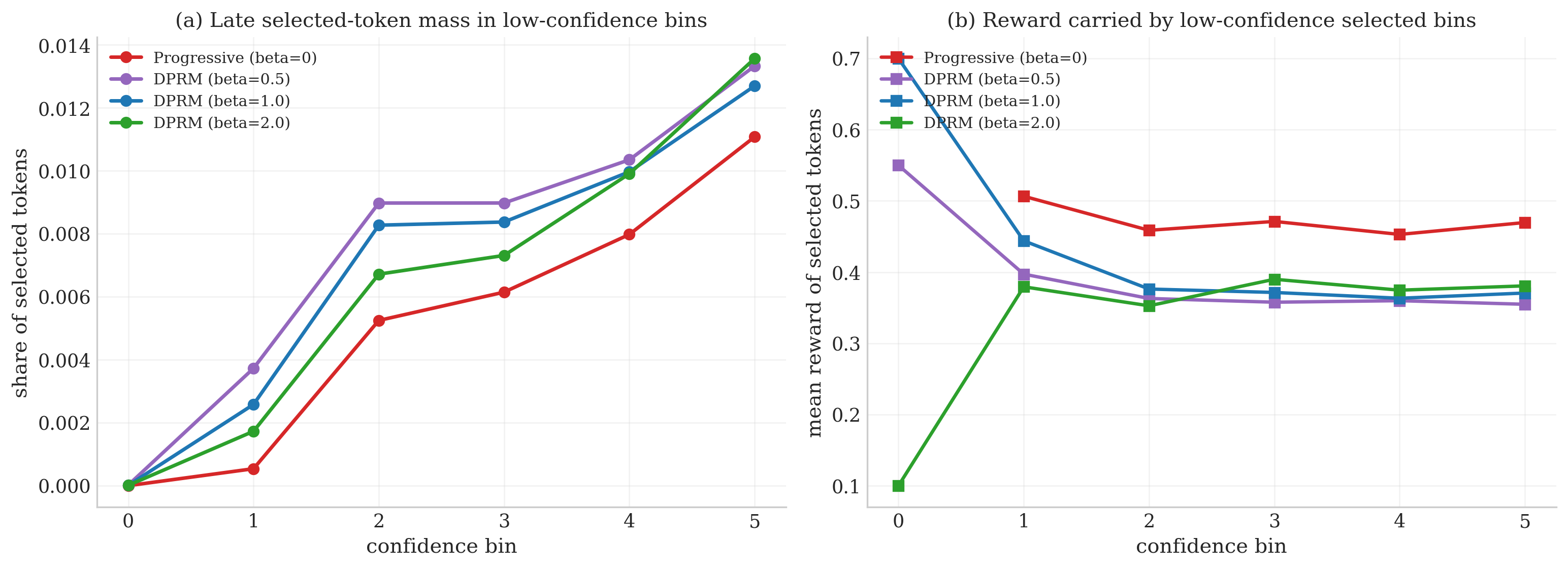}
\caption{Late Countdown coverage. DPRM restores a small, rewarded part of the low-confidence region left by confidence ordering.}
\label{fig:countdown_theory_bin_coverage}
\end{figure}

\paragraph{Empirical-Bernstein rate diagnostic.}
\Cref{thm:online_main} separates sampling error from abstraction and drift. \Cref{fig:countdown_theory_variance_bernstein} shows the predicted radius shrinkage as counts grow and reports drift separately; it does not measure the unobserved exact-state abstraction error.

\begin{figure}[H]
\centering
\includegraphics[width=\linewidth]{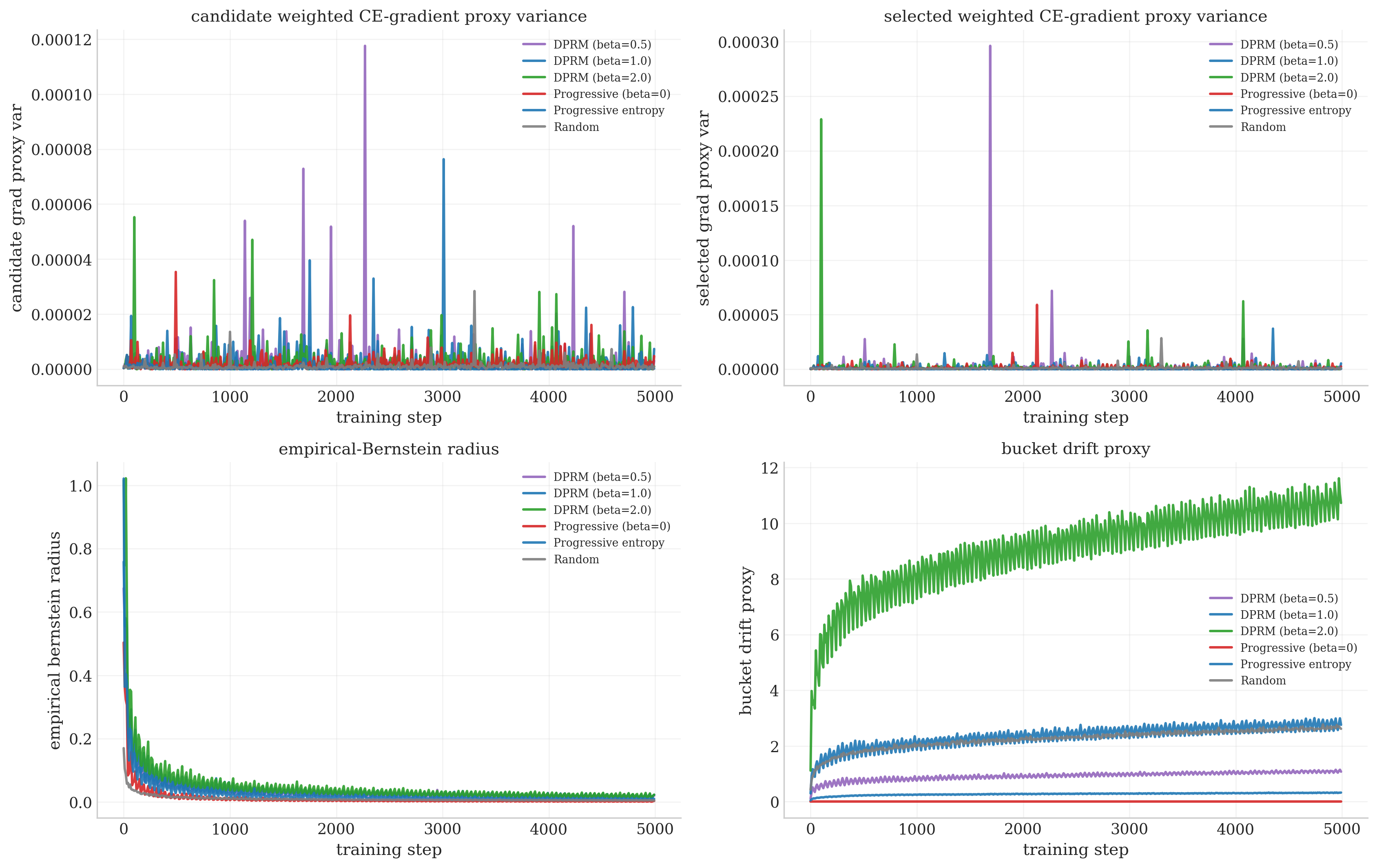}
\caption{Online-estimator diagnostics: gradient proxies, empirical-Bernstein radius, and bucket drift.}
\label{fig:countdown_theory_variance_bernstein}
\end{figure}

\paragraph{Countdown pass@\(K\) consequence.}
On the fixed \(5{,}120\)-example Countdown split, \Cref{fig:countdown_theory_diagnostics}(a) tests the early proxy: confidence gains most over random on easier subsets. Panels (b)--(c) test residual rescue: hard pass@32 rises \(47.9\%\to60.0\%\), and OOD pass@32 rises \(26.0\%\to36.3\%\).

\begin{figure}[H]
\centering
\includegraphics[width=\linewidth]{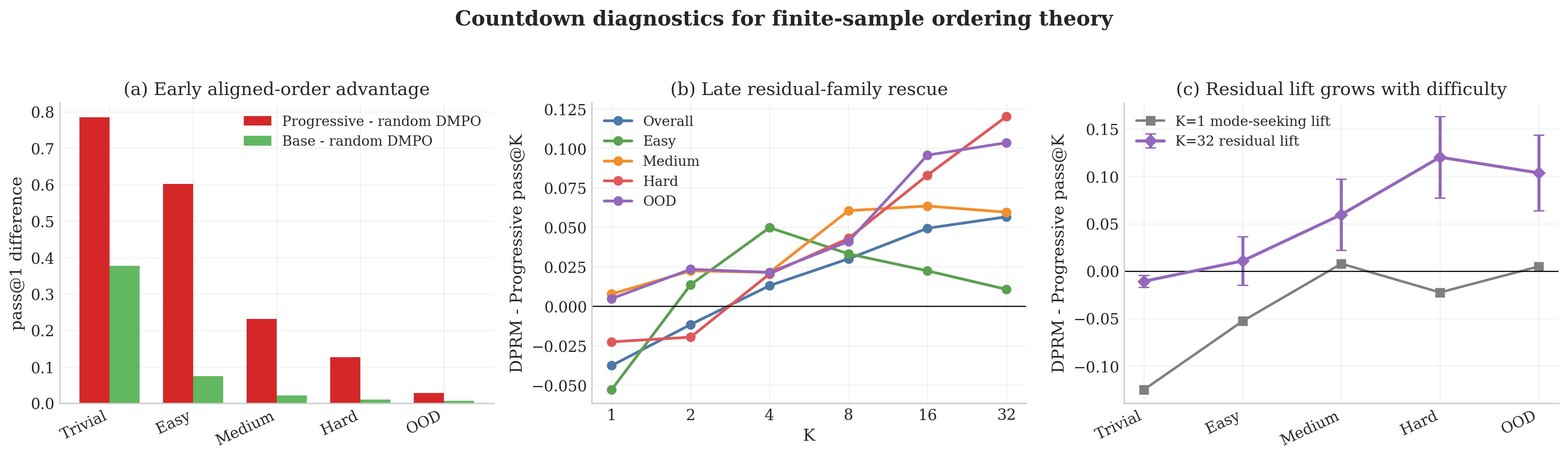}
\caption{Outcome checks for the finite-sample model: early confidence advantage and late hard/OOD DPRM gains.}
\label{fig:countdown_theory_diagnostics}
\end{figure}

\paragraph{Entropy, bucket, and EOT controls.}
The reward-blind control uses
\[
u_i(s_t)=1-\max_v p_\theta(v\mid s_t),
\]
with the same reveal budget. The main-paper mechanism figure reports these controls once. The key readings are: SDPO-DNA DPRM(random) exceeds entropy, shuffled-value, and gate/count controls; RealWorldQA phase/confidence transfers negatively, whereas format/EOT/position state changes \(37.5\%\) of orders and isolates the numeric/count gain. Thus abstraction quality is host-dependent, and reward tilt is not equivalent to generic uncertainty.

\section{Formal Statements and Proofs}

\Cref{fig:supp_proof_dependency_map} shows the proof dependencies. Each row can be read independently; the main paper states only the row endpoint.

\begin{figure}[!htbp]
\centering
\small
\renewcommand{\arraystretch}{1.3}
\begin{tabular}{p{0.16\linewidth}p{0.39\linewidth}p{0.36\linewidth}}
\toprule
\textbf{Result} & \textbf{Required structure} & \textbf{Proof route and conclusion} \\
\midrule
Alignment & Visible-state decisions; support conditions & Transition identity and posterior-weight cancellation imply marginal agreement and Bayes invariance. \\
Parameter recovery & Independent examples/masks; complete visible records are parameter-blind off the informative event & No-hit probability gives a necessary sample count; Chernoff testing bounds informative oracle trajectories. \\
Online controller & Bounded bucket observations; controlled drift and abstraction error & Mean concentration and the log map give score tracking, then top-\(m\) regret. \\
Marginal Soft-BoN & Exact bounded weights; fresh i.i.d. multisets with duplicates; common initial/update laws & Exchangeability and Jensen give local forward KL; the trajectory chain rule and data processing give terminal KL. \\
Ordering complexity & Importance/coverage conditions and two-sided progress assumptions & Early noise comparison and late residual-hit probabilities give conditional optimization separations. \\
\bottomrule
\end{tabular}
\caption{Proof guide. Parameter recovery and optimization complexity require their own observation and progress assumptions. The Soft-BoN result concerns the marginal policy after candidate randomness is averaged out.}
\label{fig:supp_proof_dependency_map}
\end{figure}

\subsection{Teacher-forced alignment, minimizer preservation, and statistical separation}
\label{app:alignment_minimizer}

Following \citet{kim2026puma}, we first state marginal agreement and Bayes-predictor invariance. We then analyze latent-parameter recovery under an explicit restricted observation protocol, separating its information assumptions from the alignment results.

\subsubsection{Teacher-forced alignment}

We use the single-index reveal setting of \citet[Appendix~A.1]{kim2026puma}.

\begin{proposition}[Teacher-forced marginal agreement]
\label{prop:teacher_forced_alignment}
Fix an integer $K\ge 1$ and a time grid $t_j:=1-j/K$ for $j=0,1,\dots,K$.
Fix a prompt $q$ and let $O\sim p_*(\cdot\mid q)$.
Consider the following two Markov chains on partially masked response states, both initialized at the fully masked response state.

\begin{enumerate}[leftmargin=18pt]
    \item \textbf{Idealized posterior-based inference.}
    Given $Z_{t_j}=z$, sample an index
    \[
    I_j\sim g_\phi(\cdot\mid z,q,t_j),
    \]
    then sample a token
    \[
    U_j\sim p_*(O_{I_j}=\cdot\mid Z_{t_j}=z,Q=q),
    \]
    and set
    \[
    Z_{t_{j+1}}:=z^{(I_j\leftarrow U_j)}.
    \]
    Let $q_{t_j}(\cdot\mid q)$ denote the law of $Z_{t_j}$.

    \item \textbf{Teacher-forced chain.}
    First sample $O\sim p_*(\cdot\mid q)$.
    Given $\widetilde Z_{t_j}=z$, sample
    \[
    I_j\sim g_\phi(\cdot\mid z,q,t_j),
    \]
    and set
    \[
    \widetilde Z_{t_{j+1}}:=z^{(I_j\leftarrow O_{I_j})}.
    \]
    Let $\widetilde q_{t_j}(\cdot\mid q)$ denote the marginal law of $\widetilde Z_{t_j}$ after marginalizing over $O$.
\end{enumerate}

Then
\[
q_{t_j}(\cdot\mid q)=\widetilde q_{t_j}(\cdot\mid q)
\qquad
\text{for every }j=0,1,\dots,K.
\]
\end{proposition}

\begin{proof}
We show that the two chains have the same one-step transition kernel.

Fix a step $j$ and a masked state $z$.
Under the teacher-forced chain, the event $\{\widetilde Z_{t_j}=z\}$ is possible only if the sampled response $O$ agrees with $z$ on the revealed coordinates.
Thus, for some scalar $\alpha_j(q,z)\ge 0$ and every response $o$,
\[
\mathbb P(\widetilde Z_{t_j}=z\mid Q=q,O=o)
=
\alpha_j(q,z)\,
\mathbf 1\{o_{\mathrm{um}(z)}=z_{\mathrm{um}(z)}\}.
\]
By Bayes' rule,
\[
\mathbb P(O=o\mid Q=q,\widetilde Z_{t_j}=z)
\propto
p_*(o\mid q)\,
\mathbf 1\{o_{\mathrm{um}(z)}=z_{\mathrm{um}(z)}\}.
\]
Hence, for every masked coordinate $i\in M(z)$ and every token $u$,
\[
\mathbb P(O_i=u\mid Q=q,\widetilde Z_{t_j}=z)
=
p_*(O_i=u\mid Q=q,Z_{t_j}=z).
\]

Now fix $i\in M(z)$ and a token value $u$.
Using the conditional identity above and the fact that the reveal policy depends only on the current visible state and time,
\begin{align*}
\mathbb P\!\left(\widetilde Z_{t_{j+1}}=z^{(i\leftarrow u)} \mid Q=q,\widetilde Z_{t_j}=z\right)
&=
g_\phi(i\mid z,q,t_j)\,
\mathbb P(O_i=u\mid Q=q,\widetilde Z_{t_j}=z) \\
&=
g_\phi(i\mid z,q,t_j)\,
p_*(O_i=u\mid Q=q,Z_{t_j}=z).
\end{align*}
On the other hand, the idealized inference chain satisfies
\[
\mathbb P\!\left(Z_{t_{j+1}}=z^{(i\leftarrow u)} \mid Q=q,Z_{t_j}=z\right)
=
g_\phi(i\mid z,q,t_j)\,
p_*(O_i=u\mid Q=q,Z_{t_j}=z).
\]
Thus the two chains have the same one-step transition kernel and the same initial state.
An induction on $j$ gives
\[
q_{t_j}(\cdot\mid q)=\widetilde q_{t_j}(\cdot\mid q)
\qquad
\text{for all }j.
\]
\end{proof}

\begin{remark}
Proposition~\ref{prop:teacher_forced_alignment} supplies the marginal-agreement part of the main paper's background alignment result.
It is the direct analogue of the discrete-time single-index result in \citet[Appendix~A.1]{kim2026puma}.
\end{remark}

\subsubsection{Minimizer preservation}

We next establish the Bayes-predictor invariance used in the main paper's background alignment result.
This is the Appendix~A.2 argument of \citet{kim2026puma}: the forward-process weight cancels in Bayes' rule, so the Bayes predictor does not depend on the particular masking weights.

For a partially masked response state $z$, let
\[
\mathrm{um}(z):=\{i:z_i\neq \mask\},
\qquad
M(z):=\{i:z_i=\mask\}.
\]

\begin{lemma}[Weighted cross-entropy is minimized by the true conditional]
\label{lem:weighted_ce}
Let $(X,Y)$ be jointly distributed random variables, where $Y$ takes values in a finite set $\mathcal V$.
Let $w(X)\ge 0$ be measurable.
Consider
\[
\mathcal J(r):=\mathbb E\big[w(X)\,(-\log r(Y\mid X))\big],
\]
where $r(\cdot\mid x)$ ranges over all conditional distributions on $\mathcal V$.
Then every minimizer $r^\star$ satisfies
\[
r^\star(\cdot\mid x)=\mathbb P(Y=\cdot\mid X=x)
\]
for almost every $x$ with $w(x)>0$.
\end{lemma}

\begin{proof}
Conditioning on $X=x$ gives
\[
\mathbb E[-\log r(Y\mid X)\mid X=x]
=
H(\mathbb P(Y=\cdot\mid X=x))
+
\mathrm{KL}\!\left(\mathbb P(Y=\cdot\mid X=x)\,\middle\|\,r(\cdot\mid x)\right).
\]
Multiply by $w(x)\ge 0$ and integrate over $X$.
The entropy term is independent of $r$, whereas the KL term is nonnegative and vanishes exactly at the true conditional.
\end{proof}

\begin{proposition}[Forward-process invariance of the Bayes predictor]
\label{prop:bayes_invariance}
Let the forward process take the form
\begin{equation}
\mathbb P(Z_t=z\mid Q=q,O=o,t)
\propto
\alpha_t(q,z)\,
\mathbf 1\{o_{\mathrm{um}(z)}=z_{\mathrm{um}(z)}\},
\label{eq:forward_factorized}
\end{equation}
where $\alpha_t(q,z)\ge 0$ does not depend on $o$.
Let
\[
p_f(q,o,t,z)
:=
\nu_*(q,o)\,\mathbf 1_{[0,1]}(t)\,\mathbb P(Z_t=z\mid Q=q,O=o,t).
\]
Consider the population loss
\begin{equation}
\mathcal L_{\mathrm{fwd}}(\theta)
:=
\mathbb E_{(Q,O,t,Z)\sim p_f}
\left[
\frac{1}{t}\sum_{i\in M(Z)} -\log p_\theta(O_i\mid Z,Q)
\right].
\label{eq:forward_loss}
\end{equation}
Then:
\begin{enumerate}[leftmargin=18pt]
    \item the posterior satisfies
    \begin{equation}
    \mathbb P(O=o\mid Q=q,Z=z,t)
    \propto
    p_*(o\mid q)\,
    \mathbf 1\{o_{\mathrm{um}(z)}=z_{\mathrm{um}(z)}\},
    \label{eq:posterior_cancelled}
    \end{equation}
    so it is independent of $\alpha_t(q,z)$;
    \item every minimizer $\theta^\star$ of \cref{eq:forward_loss} satisfies
    \[
    p_{\theta^\star}(u\mid z,q)
    =
    \mathbb P(O_i=u\mid Q=q,Z=z)
    \qquad
    \text{for almost every }(q,z,i)\text{ with }i\in M(z).
    \]
\end{enumerate}
\end{proposition}

\begin{proof}
Fix $q$, $t$, and $z$.
By Bayes' rule and \cref{eq:forward_factorized},
\[
\mathbb P(O=o\mid Q=q,Z=z,t)
\propto
\mathbb P(Z_t=z\mid Q=q,O=o,t)\,p_*(o\mid q)
\propto
\alpha_t(q,z)\,
\mathbf 1\{o_{\mathrm{um}(z)}=z_{\mathrm{um}(z)}\}\,
p_*(o\mid q).
\]
Since $\alpha_t(q,z)$ does not depend on $o$, it cancels under normalization.
This proves \cref{eq:posterior_cancelled}.

Now fix a coordinate $i$ and define
\[
X:=(Q,Z),
\qquad
Y:=O_i.
\]
The contribution of coordinate $i$ to \cref{eq:forward_loss} is
\[
\mathcal L_i(\theta)
=
\mathbb E\!\left[
\frac{1}{t}\,\mathbf 1\{i\in M(Z)\}\,
\bigl(-\log p_\theta(Y\mid X)\bigr)
\right].
\]
By the first part, the conditional law of $Y$ given $X=(q,z)$ is the ground-truth posterior and does not depend on $\alpha_t$.
Applying Lemma~\ref{lem:weighted_ce} with weight
\[
w_i(Q,Z,t):=\frac{1}{t}\,\mathbf 1\{i\in M(Z)\}
\]
shows that every minimizer matches that conditional posterior on every active masked coordinate.
Summing over $i$ proves the claim.
\end{proof}

\begin{corollary}[Minimizer preservation for admissible teacher-forced masking]
\label{cor:minimizer_preservation}
Suppose that an admissible teacher-forced progressive masking law $q_\pi$ is induced by a forward process of the form \cref{eq:forward_factorized}.
Then every minimizer of the progressive population risk satisfies
\[
p_{\theta^\star}(u\mid z,q)
=
\nu_*(O_i=u\mid Q=q,Z=z)
\]
for $\bar\rho_\pi$-almost every $(q,z,i)$ with $i\in M(z)$.
Hence admissible progressive masking changes only the occupancy measure over masked states, not the Bayes-optimal denoiser.
\end{corollary}

\begin{proof}
This is exactly the second part of Proposition~\ref{prop:bayes_invariance}, rewritten under the progressive population-risk notation.
\end{proof}

\subsubsection{Statistical separation}

We formalize a latent-variable separation in the spirit of \citet[Proposition~3]{kim2026puma}, with an explicit observation protocol. The statistical resource is the number of supervised state records accessible to the learner. In particular, a masked label can be available as supervision even when it is absent from the model input, and this label is part of the learner's observed record.

\begin{lemma}[Chernoff bound]
\label{lem:chernoff}
Let $P$ and $Q$ be distributions on a finite space, and let $Z_1,\dots,Z_T\sim P$ be i.i.d.
Then
\[
\mathbb P\!\left(\prod_{t=1}^T \frac{Q(Z_t)}{P(Z_t)} \ge 1\right)
\le e^{-T C(P,Q)},
\]
where
\[
C(P,Q):=-\log\inf_{0<s<1}\sum_z P(z)^{1-s}Q(z)^s
\]
is the Chernoff information. Likelihood ratios are evaluated on the support of $P$.
\end{lemma}

\begin{proof}
For any $s\in(0,1)$, Markov's inequality and independence give
\[
\begin{split}
\mathbb P\!\left(\prod_{t=1}^T \frac{Q(Z_t)}{P(Z_t)}\ge1\right)
&\le \mathbb E\!\left[\prod_{t=1}^T
\left(\frac{Q(Z_t)}{P(Z_t)}\right)^s\right]\\
&=\left(\sum_z P(z)^{1-s}Q(z)^s\right)^T.
\end{split}
\]
Taking the infimum over $s\in(0,1)$ proves the claim.
\end{proof}

\noindent\textbf{Observation protocol.}
Let $\Theta$ be a known finite hypothesis class with $M:=|\Theta|\ge2$, and let $d\ge1$.
Under parameter $\theta$, a fresh response is
\[
O=\pi(U_{1:d},Y)\sim P_\theta,
\]
where $\pi$ is a fixed known, parameter-independent coordinate permutation.
For each random-mask record $i$, independently draw a fresh response and mask indicators $B_{i,1},\dots,B_{i,d},B_{i,Y}\sim\operatorname{Bernoulli}(q)$, with $q\in(0,1)$ and $B=1$ meaning masked. The masks are independent of the responses and of $\theta$.
Define the informative event
\[
F_i:=\{B_{i,1}=\cdots=B_{i,d}=0,\ B_{i,Y}=1\},
\qquad a:=\mathbb P_\theta(F_i)=q(1-q)^d.
\]
The learner receives $(B_i,V_i)$, where $V_i$ includes \emph{every} accessible coordinate, target label, auxiliary statistic, and metadata item from that example. On $F_i$, it contains the supervised pair $(U_{i,1:d},Y_i)$ and hence an observation with law $P_\theta$ after applying $\pi$.
Conditional on $\theta$, the record-generation channels use independent randomness and independent responses across $i$; the learner cannot remask a previously observed response or access an underlying complete dataset.
Its private random seed is independent of $\theta$ and of all records.
An oracle trajectory instead reveals the $d$ latent coordinates and then queries $Y$; it costs $d+1$ supervised state records and contains one complete observation with law $P_\theta$. Different oracle trajectories use independent responses, and the oracle supplies no parameter information other than these records.

\begin{assumption}[Posterior collapse and uninformative miss records]
\label{ass:posterior}
There exists a fixed distribution $\pi_0$ such that, for every $\theta$ and strict subset $S\subsetneq[d]$, the conditional law of $Y$ given $U_S$ is $\pi_0$ almost surely.
In addition, for every mask pattern $b$ outside $F$, the law of the \emph{entire} released record is parameter-independent: there is a distribution $H_b$ such that
\[
\mathcal L_\theta(V_i\mid B_i=b)=H_b
\qquad\text{for every }\theta\in\Theta.
\]
This requirement covers visible coordinates and all supplied labels, including any loss or feedback computed from them. No additional parameter-dependent observations are available to the learner.
\end{assumption}

\begin{assumption}[Identifiability of complete observations]
\label{ass:identifiability}
The response space is finite, and there is a $\kappa>0$ such that
\[
C(P_\theta,P_{\theta'})\ge\kappa
\qquad\text{for all }\theta\ne\theta'.
\]
\end{assumption}

The posterior condition describes when the target becomes predictable; the condition on $V_i$ separately controls information about the unknown parameter. The latter applies to the learner's full transcript, not just the masked model input. In particular, an interface that exposes every complete response before masking is outside this observation protocol.

\begin{proposition}[Conditional exponential-vs-linear record complexity]
\label{prop:sample_complexity_separation}
Let $\delta\in(0,1)$. Under the observation protocol and Assumptions~\ref{ass:posterior}--\ref{ass:identifiability}, every estimator based on $n$ random-mask records has uniform-prior mean recovery error at least
\begin{equation}
\frac1M\sum_{\theta\in\Theta}
\mathbb P_\theta(\widehat\theta\ne\theta)
\ge (1-1/M)(1-a)^n.
\label{eq:app_no_hit_error}
\end{equation}
Consequently, for $0<\delta<1-1/M$, a necessary condition for either uniform-prior mean error or worst-case error at most $\delta$ is
\begin{equation}
n\ge
\left\lceil
\frac{\log((1-1/M)/\delta)}{-\log(1-a)}
\right\rceil,
\qquad a=q(1-q)^d.
\label{eq:app_no_hit_sample_bound}
\end{equation}
For $\delta\ge1-1/M$, this argument gives only the trivial lower bound $n\ge0$.
In contrast, uniform-prior MAP (equivalently, maximum likelihood) on independent oracle trajectories achieves worst-case error at most $\delta$ with
\begin{equation}
n_{\mathrm{traj}}=(d+1)
\left\lceil\frac{\log((M-1)/\delta)}{\kappa}\right\rceil
\label{eq:app_oracle_record_bound}
\end{equation}
supervised state records.
For fixed $q,M,\delta,\kappa$, with $0<\delta<1-1/M$, the random-mask lower bound is exponential in $d$, whereas this oracle upper bound is linear in $d$.
\end{proposition}

\begin{proof}
Put a uniform prior on $\theta$ and let $H_n:=\bigcap_{i=1}^n F_i^c$.
Independent masking gives
\[
\mathbb P_\theta(H_n)=(1-a)^n
\]
for every $\theta$.
For any sequence of miss mask patterns, conditional independence of records and Assumption~\ref{ass:posterior} make the joint law of all released records a product of parameter-independent distributions $H_{b_i}$. The conditional distribution of the mask patterns given $H_n$ is also parameter-independent. Thus the entire observed transcript, together with the learner's private seed, is independent of $\theta$ conditional on $H_n$.
The posterior on $\theta$ remains uniform on this event, so the conditional recovery error of any estimator is at least $1-1/M$.
Multiplying by $\mathbb P(H_n)=(1-a)^n$ proves \cref{eq:app_no_hit_error}.
The worst-case error is at least the uniform-prior mean error. Taking logarithms and using the integrality of $n$ proves \cref{eq:app_no_hit_sample_bound}.
For the growth rate, the elementary integral bound
\[
a\le -\log(1-a)\le\frac{a}{1-a}
\]
shows that the unrounded lower bound is at least
\[
\frac{1-a}{a}\log\!\left(\frac{1-1/M}{\delta}\right).
\]
Here $a=q(1-q)^d\le1/4$ for $d\ge1$, so the prefactor is of order $q^{-1}(1-q)^{-d}$. For example, for $\delta\le1/4$ this also yields
$n=\Omega(q^{-1}(1-q)^{-d}\log(1/\delta))$.

For the oracle bound, let $Z_1,\dots,Z_T$ be the complete supervised observations extracted from independent trajectories. They are i.i.d. with law $P_\theta$; states within a trajectory need not be independent.
For a fixed wrong hypothesis $\theta'$, maximum likelihood can select $\theta'$ only if
\[
\prod_{t=1}^T\frac{P_{\theta'}(Z_t)}{P_\theta(Z_t)}\ge1.
\]
By Lemma~\ref{lem:chernoff} and Assumption~\ref{ass:identifiability}, this event has probability at most $e^{-T\kappa}$.
A union bound gives
\[
\sup_{\theta\in\Theta}\mathbb P_\theta(\widehat\theta\ne\theta)
\le(M-1)e^{-T\kappa}.
\]
Taking $T=\lceil\kappa^{-1}\log((M-1)/\delta)\rceil$ and charging $d+1$ records per trajectory proves \cref{eq:app_oracle_record_bound}.
More generally, the oracle cost is
$O((d+1)[1+\kappa^{-1}\log(M/\delta)])$; linear dependence on $d$ presumes that $M$, $\kappa$, and the target error are held fixed.
\end{proof}

\begin{remark}[A nonvacuous observation model and its scope]
Take $\Theta=\{0,1\}$, independent uniform bits $U_1,\dots,U_d$, and an independent $E\sim\operatorname{Bernoulli}(\varepsilon)$ with fixed $0<\varepsilon<1/2$. Set
\[
Y=\theta\mathbin{\oplus}U_1\mathbin{\oplus}\cdots
\mathbin{\oplus}U_d\mathbin{\oplus}E.
\]
Let $S=\{j:B_j=0\}$. A target-query interface releases $U_S$ and, only when $B_Y=1$, the supervised target $Y$. When $B_Y=0$, the value of $Y$ is censored from both context and feedback; no other labels are released. This censoring is part of the observation model.
If $S\subsetneq[d]$, a missing uniform bit makes $Y$ uniform and independent of $U_S$ and $\theta$. If $B_Y=0$, the observed $U_S$ is parameter-independent even when $S=[d]$. Thus every miss record is uninformative, while a hit releases $(U_{1:d},Y)$.
The complete laws satisfy
\[
C(P_0,P_1)=-\log\!\bigl(2\sqrt{\varepsilon(1-\varepsilon)}\bigr)>0,
\]
independently of $d$. Revealing all latent coordinates before querying $Y$ realizes the oracle trajectory.
The proposition therefore establishes a separation for the specified information interface. Applying it to another training pipeline requires verifying its full observation channel, including labels and repeated access to examples. It concerns finite-sample parameter recovery and does not alter the Bayes-predictor invariance in Corollary~\ref{cor:minimizer_preservation}.
\end{remark}

\subsection{Formal statement and proof of the online tracking theorem}
\label{app:online_dprm_proof}

\subsubsection{Exact-state and bucket-state definitions used by the experiments}
\label{app:bucket_state_definitions}

The exact state contains the condition, visible tokens, mask pattern, step, and host side information. Index \(i\) denotes the next selected position; the host still supplies its token kernel \(K_\theta\):
\[
R_t^\star(i;s_t)
=
\frac1\beta\log
\mathbb E_{q_0}\!\left[e^{\beta R(X_T)}\mid S_t=s_t,A_t=i\right],
\qquad
\psi_i(s_t)=\max_y K_\theta(y\mid s_t,i).
\]
DPRM changes the position policy, not \(K_\theta\), the labels, or the loss. Order can still change the terminal law because every later kernel is conditioned on the newly visible state. The score-tracking target is this conditional log-moment value. For stochastic token updates, matching that position score alone does not implement the full Doob-transformed token kernel; the distinction between the Gibbs-controlled reference and a globally tilted path law is made explicit in \Cref{app:softbon_dprm_proof}.

Exact states rarely repeat, so the implementation stores
\[
Z_t=\zeta_t(S_t,A_t)
=
\bigl(\phi_t(S_t),b_t(S_t,A_t),c_t(S_t,A_t)\bigr),
\qquad
b_t(S_t,i)=\min\{B-1,\lfloor B\psi_i(S_t)\rfloor\},
\]
where \(c_t\) is an optional host-specific state. Its population and empirical log-moment values are
\begin{align}
\bar R_{z,t}
&=\frac1\beta\log\mathbb E[e^{\beta R(X_T)}\mid Z_t=z],\\
\widehat R_{z,t}
&=\frac1\beta\log\left(
\frac1{N_{z,t}}\sum_{j:Z_j=z}e^{\beta R_j}
\right).
\end{align}
The abstraction error is therefore
\[
\left|R_t^\star(i;s_t)-\bar R_{\zeta_t(s_t,i),t}\right|.
\]
A new exact state uses historical evidence from its cell. Under-ready cells are gated toward confidence; empty cells set \(\widehat R=0\) and recover confidence exactly. \Cref{fig:supp_bucket_support_map} reports both candidate-cell support and selected-decision activation.

The proof follows the middle row of \Cref{fig:supp_proof_dependency_map}: empirical-Bernstein concentration, log/drift transfer, then top-\(m\) regret.

For a reachable state \(s=(q,z,t)\), define
\[
\eta_t(s):=\max_{i\in M(z)}\eta_t(\phi,b_i(s)),
\qquad
B_t(s):=\sup_{i\in M(z)} |R_t^\star(i;s)|,
\]
\[
D_t(s):=\max_{i\in M(z)} D_{\phi,b_i(s)}(t),
\qquad
D_{\phi,b}(t):=\sum_{u=2}^{t} |\mu_{\phi,b,u}-\mu_{\phi,b,u-1}|.
\]
Also define
\begin{equation}
    \mathrm{rad}_t(s;\delta)
    :=
    \max_{i\in M(z)}
    \frac{1}{\beta\underline\mu}
    \left[
        \sqrt{\frac{2\widehat v_{\phi,b_i(s),t}\log(3KBT^2/\delta)}{N_{\phi,b_i(s),t}}}
        +
        \frac{3(e^\beta-1)\log(3KBT^2/\delta)}{N_{\phi,b_i(s),t}}
    \right]
\end{equation}
and
\begin{equation}
\varepsilon_t(s;\delta)
:=
\beta \varepsilon_{\mathrm{abs},t}(s)
+
\beta \eta_t(s)\,\mathrm{rad}_t(s;\delta)
+
\beta(1-\eta_t(s))\,B_t(s)
+
\frac{2}{\underline\mu}D_t(s).
\label{eq:online_main_eps}
\end{equation}
The bias term is
\[
\operatorname{Bias}_t(s;\delta)
:=
\beta \varepsilon_{\mathrm{abs},t}(s)
+
\beta(1-\eta_t(s))\,B_t(s)
+
\frac{2}{\underline\mu}D_t(s),
\]
while \(\beta \eta_t(s)\,\mathrm{rad}_t(s;\delta)\) is the sampling term. For the union bound, let
\[
\mathcal N := KBT^2,
\]
the number of controlled bucket-time events.
\begin{theorem}[Online score tracking]
\label{thm:online_main}
Fix a horizon \(T\) and a confidence level \(\delta\in(0,1)\).
Assume that:
\begin{enumerate}[leftmargin=18pt]
    \item \(R(X_T)\in[0,1]\) almost surely;
    \item for every reachable state \(s=(q,z,t)\) and every \(i\in M(z)\),
    \[
        |R_t^\star(i;s)-\bar R_{\phi,b_i(s),t}|
        \le
        \varepsilon_{\mathrm{abs},t}(s);
    \]
    \item \(\mu_{\phi,b,t}\ge \underline\mu>0\) for every bucket and time.
\end{enumerate}

Then, with probability at least \(1-\delta\), for all reachable states \(s=(q,z,t)\) and all \(t\le T\),
\begin{equation}
    \sup_{i\in M(z)}
    |\widehat g_t(i;s)-g_t^\star(i;s)|
    \le
    \varepsilon_t(s;\delta).
    \label{eq:online_main_score_error}
\end{equation}
Moreover, if \(\widehat A_t(s)\) is chosen by the practical controller and
\[
A_t^\star(s)\in\arg\max_{|A|=m(s)} \sum_{i\in A} g_t^\star(i;s),
\]
then
\begin{equation}
    0
    \le
    \sum_{i\in A_t^\star(s)} g_t^\star(i;s)
    -
    \sum_{i\in \widehat A_t(s)} g_t^\star(i;s)
    \le
    2\,m(s)\,\varepsilon_t(s;\delta).
    \label{eq:online_main_regret}
\end{equation}
\end{theorem}

\medskip
\noindent\textbf{Notation.}
For a bucket \((\phi,b)\) at time \(t\), define
\[
Y_{\phi,b}^{(j)}:=\exp(\beta R_j)\in[1,e^\beta],
\qquad
\mu_{\phi,b,t}:=\mathbb E[Y_{\phi,b}^{(j)}],
\qquad
\widehat\mu_{\phi,b,t}:=\frac{1}{N_{\phi,b,t}}\sum_{j=1}^{N_{\phi,b,t}}Y_{\phi,b}^{(j)},
\]
and
\[
\bar R_{\phi,b,t}:=\frac{1}{\beta}\log \mu_{\phi,b,t},
\qquad
\widehat R_{\phi,b,t}:=\frac{1}{\beta}\log \widehat\mu_{\phi,b,t}.
\]

\subsubsection{Step 1: concentration for the bucket mean}

\begin{lemma}[Bucket mean concentration]
\label{lem:app_bucket_mean_concentration}
Fix a bucket \((\phi,b)\) and time \(t\).
Assume that \(Y_{\phi,b}^{(1)},\dots,Y_{\phi,b}^{(N_{\phi,b,t})}\in[1,e^\beta]\) are conditionally independent with common mean \(\mu_{\phi,b,t}\) and empirical variance \(\widehat v_{\phi,b,t}\).
Then, for every \(\delta\in(0,1)\), with probability at least \(1-\delta\),
\begin{equation}
    |\widehat \mu_{\phi,b,t}-\mu_{\phi,b,t}|
    \le
    \sqrt{\frac{2\widehat v_{\phi,b,t}\log(3/\delta)}{N_{\phi,b,t}}}
    +
    \frac{3(e^\beta-1)\log(3/\delta)}{N_{\phi,b,t}}.
    \label{eq:app_bucket_mean_concentration}
\end{equation}
\end{lemma}

\begin{proof}
This is the bounded empirical-Bernstein inequality; see, for example, the main empirical-Bernstein theorem of \citet{maurer2009empirical}.
We apply it to the bounded sample
\[
Y_{\phi,b}^{(1)},\dots,Y_{\phi,b}^{(N_{\phi,b,t})}\in[1,e^\beta].
\]
\end{proof}

\subsubsection{Step 2: concentration for the log-moment reward estimator}

\begin{lemma}[Bucket reward concentration]
\label{lem:app_bucket_reward_concentration}
Assume also that \(\mu_{\phi,b,t}\ge \underline\mu>0\).
Then, on the event in \cref{eq:app_bucket_mean_concentration},
\begin{equation}
    |\widehat R_{\phi,b,t}-\bar R_{\phi,b,t}|
    \le
    \frac{1}{\beta\underline\mu}
    \left[
        \sqrt{\frac{2\widehat v_{\phi,b,t}\log(3/\delta)}{N_{\phi,b,t}}}
        +
        \frac{3(e^\beta-1)\log(3/\delta)}{N_{\phi,b,t}}
    \right].
    \label{eq:app_bucket_reward_concentration}
\end{equation}
\end{lemma}

\begin{proof}
Since
\[
\widehat R_{\phi,b,t}-\bar R_{\phi,b,t}
=
\frac{1}{\beta}\bigl(\log \widehat\mu_{\phi,b,t}-\log \mu_{\phi,b,t}\bigr),
\]
it suffices to control the log map.
The derivative of \((1/\beta)\log x\) is \((\beta x)^{-1}\).
Because \(\mu_{\phi,b,t}\ge \underline\mu\), the mean-value theorem gives
\[
\left|
\frac{1}{\beta}\log \widehat\mu_{\phi,b,t}
-
\frac{1}{\beta}\log \mu_{\phi,b,t}
\right|
\le
\frac{1}{\beta\underline\mu}
|\widehat\mu_{\phi,b,t}-\mu_{\phi,b,t}|.
\]
Substituting \cref{eq:app_bucket_mean_concentration} proves \cref{eq:app_bucket_reward_concentration}.
\end{proof}

In the online setting, the bucket mean is generally nonstationary. This nonstationarity is not only due to changing visitation frequencies, but also due to updates of the host denoiser. At training update \(u\), the base proposal \(q_{0,u}\) is induced by the current score \(\psi_{\theta_u}\), and therefore the exact conditional log-moment reward \(R_{u,t}^\star\) and its bucket projection \(\bar R_{\phi,b,u}\) may change with \(u\). We absorb this moving-target effect into the variation budget
\[
D_{\phi,b}(t)=\sum_{u=2}^t |\mu_{\phi,b,u}-\mu_{\phi,b,u-1}|.
\]
Thus the online tracking theorem should be read as a guarantee for slowly drifting bucket means; when the host model changes rapidly, the drift term becomes large, and the guarantee correspondingly weakens.

\subsubsection{Step 3: drift-aware uniform score error}

For each bucket \((\phi,b)\), define
\[
D_{\phi,b}(t):=\sum_{u=2}^{t} |\mu_{\phi,b,u}-\mu_{\phi,b,u-1}|.
\]

\begin{lemma}[Drift-aware bucket tracking]
\label{lem:app_bucket_drift_tracking}
With probability at least \(1-\delta\), uniformly over all buckets and all \(t\le T\),
\begin{equation}
    |\widehat R_{\phi,b,t}-\bar R_{\phi,b,t}|
    \le
    \mathrm{rad}_{\phi,b,t}(\delta)
    +
    \frac{2}{\underline\mu}D_{\phi,b}(t),
    \label{eq:app_bucket_tracking}
\end{equation}
where
\begin{equation}
    \mathrm{rad}_{\phi,b,t}(\delta)
    :=
    \frac{1}{\beta\underline\mu}
    \left[
        \sqrt{\frac{2\widehat v_{\phi,b,t}\log(3KBT^2/\delta)}{N_{\phi,b,t}}}
        +
        \frac{3(e^\beta-1)\log(3KBT^2/\delta)}{N_{\phi,b,t}}
    \right].
    \label{eq:app_bucket_rad}
\end{equation}
\end{lemma}

\begin{proof}
Apply \cref{lem:app_bucket_reward_concentration} to every bucket and every time \(t\le T\) with failure probability
\[
\delta'=\frac{\delta}{KBT^2}.
\]
A union bound gives the concentration term uniformly.

It remains to control drift.
The cumulative quantity \(D_{\phi,b}(t)\) is a variation-budget term of the kind standard in non-stationary online learning; see \citet{besbes2015nonstationary}.
For each \(u\),
\[
\bar R_{\phi,b,u}-\bar R_{\phi,b,u-1}
=
\frac{1}{\beta}\bigl(\log \mu_{\phi,b,u}-\log \mu_{\phi,b,u-1}\bigr).
\]
Since \(\mu_{\phi,b,u}\ge \underline\mu\), the mean-value theorem yields
\[
|\bar R_{\phi,b,u}-\bar R_{\phi,b,u-1}|
\le
\frac{1}{\beta\underline\mu}
|\mu_{\phi,b,u}-\mu_{\phi,b,u-1}|.
\]
Summing over \(u\le t\) gives a cumulative drift guidance proportional to \(D_{\phi,b}(t)\).
Using the slightly looser coefficient \(2/\underline\mu\) yields \cref{eq:app_bucket_tracking}.
\end{proof}

\begin{lemma}[Uniform score approximation]
\label{lem:app_uniform_score_approx}
Assume that for every reachable state \(s=(q,z,t)\) and every \(i\in M(z)\),
\[
|R_t^\star(i;s)-\bar R_{\phi,b_i(s),t}|
\le
\varepsilon_{\mathrm{abs},t}(s).
\]
Then, on the event of \cref{lem:app_bucket_drift_tracking}, for every reachable state \(s=(q,z,t)\),
\[
\sup_{i\in M(z)}
|\widehat g_t(i;s)-g_t^\star(i;s)|
\le
\varepsilon_t(s;\delta),
\]
where \(\varepsilon_t(s;\delta)\) is the quantity in \cref{eq:online_main_eps}.
\end{lemma}

\begin{proof}
Fix \(i\in M(z)\).
Then
\begin{align}
    |\widehat g_t(i;s)-g_t^\star(i;s)|
    &=
    \beta\left|
        \eta_t(\phi,b_i(s))\widehat R_{\phi,b_i(s),t}
        -
        R_t^\star(i;s)
    \right|
    \nonumber\\
    &\le
    \beta |R_t^\star(i;s)-\bar R_{\phi,b_i(s),t}|
    +
    \beta \eta_t(\phi,b_i(s))
    |\widehat R_{\phi,b_i(s),t}-\bar R_{\phi,b_i(s),t}|
    \nonumber\\
    &\qquad
    +
    \beta(1-\eta_t(\phi,b_i(s)))|R_t^\star(i;s)|.
\end{align}
The first term is bounded by the abstraction error assumption.
The second term is bounded by \cref{lem:app_bucket_drift_tracking}.
The third term is bounded by \(B_t(s)\).
Taking the supremum over \(i\in M(z)\) yields exactly \cref{eq:online_main_eps}.
\end{proof}

\subsubsection{Step 4: score error implies reveal regret}

\begin{lemma}[Top-\(m\) regret from uniform score error]
\label{lem:app_topm_regret}
Let
\[
J_t^\star(A;s):=\sum_{i\in A} g_t^\star(i;s),
\qquad
\widehat J_t(A;s):=\sum_{i\in A} \widehat g_t(i;s),
\]
and let
\[
A_t^\star(s)\in\arg\max_{|A|=m(s)} J_t^\star(A;s),
\qquad
\widehat A_t(s)\in\arg\max_{|A|=m(s)} \widehat J_t(A;s).
\]
If
\[
\sup_{i\in M(z)} |\widehat g_t(i;s)-g_t^\star(i;s)|\le \varepsilon,
\]
then
\begin{equation}
    0
    \le
    J_t^\star(A_t^\star(s);s)-J_t^\star(\widehat A_t(s);s)
    \le
    2m(s)\varepsilon.
    \label{eq:app_topm_regret}
\end{equation}
\end{lemma}

\begin{proof}
For any feasible set \(A\),
\[
|\widehat J_t(A;s)-J_t^\star(A;s)|
\le
\sum_{i\in A} |\widehat g_t(i;s)-g_t^\star(i;s)|
\le
m(s)\varepsilon.
\]
Since \(\widehat A_t(s)\) maximizes \(\widehat J_t(\cdot;s)\),
\[
\widehat J_t(\widehat A_t(s);s)\ge \widehat J_t(A_t^\star(s);s).
\]
Therefore
\begin{align}
J_t^\star(A_t^\star(s);s)
&\le
\widehat J_t(A_t^\star(s);s)+m(s)\varepsilon
\nonumber\\
&\le
\widehat J_t(\widehat A_t(s);s)+m(s)\varepsilon
\nonumber\\
&\le
J_t^\star(\widehat A_t(s);s)+2m(s)\varepsilon.
\end{align}
This proves \cref{eq:app_topm_regret}.
\end{proof}

\begin{proof}[Proof of Theorem~\ref{thm:online_main}]
The score bound \cref{eq:online_main_score_error} is exactly \cref{lem:app_uniform_score_approx}.
Applying \cref{lem:app_topm_regret} with \(\varepsilon=\varepsilon_t(s;\delta)\) gives \cref{eq:online_main_regret}.
\end{proof}

\begin{corollary}[Exact recovery under a margin]
\label{cor:app_online_margin}
Fix a reachable state \(s=(q,z,t)\), and define
\[
\Delta_t(s)
:=
\min_{i\in A_t^\star(s),\,j\notin A_t^\star(s)}
\bigl(g_t^\star(i;s)-g_t^\star(j;s)\bigr).
\]
If \(\Delta_t(s)>0\) and
\[
\varepsilon_t(s;\delta)<\frac{\Delta_t(s)}{2},
\]
then
\[
\widehat A_t(s)=A_t^\star(s).
\]
\end{corollary}

\begin{proof}
Fix \(i\in A_t^\star(s)\) and \(j\notin A_t^\star(s)\).
Then
\[
\widehat g_t(i;s)-\widehat g_t(j;s)
\ge
g_t^\star(i;s)-g_t^\star(j;s)-2\varepsilon_t(s;\delta)
>
0.
\]
Thus every element of \(A_t^\star(s)\) still ranks above every element outside it under \(\widehat g_t\).
Hence the top-\(m(s)\) set is unchanged.
\end{proof}

\subsection{Proof of the Soft-BoN Approximation Theorem}
\label{app:softbon_dprm_proof}

We first identify the exact Doob policy, then prove a statewise bound for the \emph{marginal} Soft-BoN action law and lift it to trajectories and terminal outputs. The local argument is self-contained and concerns the sampling-and-selection distribution studied by \citet{verdun2025soft}.

\medskip
\noindent\textbf{Setup and initial law.}
Consider a finite-horizon Markov reveal trajectory
\[
\tau=(S_0,A_0,S_1,A_1,\dots,S_{T-1},A_{T-1},S_T).
\]
At state \(s\in\mathcal S_t\), assume that the action \(a\in\mathcal A(s)\) determines the next state \(s^a\). This includes teacher forcing conditional on a fixed example and deterministic token decoding; stochastic updates are addressed below. Each action set is finite, and we restrict it to the support of \(q_0(\cdot\mid s)\). Let \(\beta>0\) and \(R(X_T)\in[0,1]\). On valid trajectories, define
\[
\mathbb P_0(\tau)
=\rho_0(S_0)\prod_{t=0}^{T-1}q_0(A_t\mid S_t),
\qquad
\mathbb P_\beta(d\tau)
=\frac{e^{\beta R(X_T(\tau))}}{Z_\beta}\,\mathbb P_0(d\tau),
\]
where
\[
Z_\beta=\mathbb E_{\mathbb P_0}[e^{\beta R(X_T)}],
\qquad
h_t(s)=\mathbb E_{\mathbb P_0}[e^{\beta R(X_T)}\mid S_t=s].
\]
The initial marginal of the globally tilted law is
\[
\rho_\beta(ds)
=\frac{h_0(s)}{Z_\beta}\rho_0(ds),
\qquad
Z_\beta=\mathbb E_{\rho_0}[h_0(S_0)].
\]
Thus a fixed initial state is unchanged, whereas a random initial law is generally tilted. Unless specified otherwise, the approximate process below starts from \(\rho_\beta\) as well. Conditioning on each initial context separately is another way to obtain matching initial laws; mixing these conditionally tilted laws with \(\rho_0\) need not equal the global tilt.

\begin{lemma}[Backward recursion for \(h_t\)]
\label{lem:app_h_recursion}
For every reachable non-terminal state,
\[
h_t(s)=\sum_{a\in\mathcal A(s)}q_0(a\mid s)h_{t+1}(s^a),
\qquad
h_T(s)=e^{\beta R(x(s))},
\]
where \(x(s)\) is the terminal output represented by \(s\).
\end{lemma}

\begin{proof}
The Markov property and the tower property give
\[
h_t(s)=\mathbb E_{\mathbb P_0}[h_{t+1}(S_{t+1})\mid S_t=s].
\]
Since \(S_{t+1}=s^{A_t}\), this is the stated sum. At time \(T\), the reward is determined by the terminal state.
\end{proof}

\begin{proposition}[Exact reveal rule under the tilted path law]
\label{prop:app_local_gibbs}
The one-step action conditional of \(\mathbb P_\beta\) is
\[
\pi_t^\star(a\mid s)
=\frac{q_0(a\mid s)h_{t+1}(s^a)}{h_t(s)}
=\frac{q_0(a\mid s)e^{\beta R_t^\star(a;s)}}
{\sum_bq_0(b\mid s)e^{\beta R_t^\star(b;s)}},
\]
where \(R_t^\star(a;s)=\beta^{-1}\log h_{t+1}(s^a)\). Together with the initial law \(\rho_\beta\), these conditionals realize \(\mathbb P_\beta\).
\end{proposition}

\begin{proof}
Conditional change of measure gives
\[
\mathbb P_\beta(A_t=a\mid S_t=s)
=\frac{q_0(a\mid s)\,
\mathbb E_{\mathbb P_0}[e^{\beta R(X_T)}\mid S_t=s,A_t=a]}{h_t(s)}.
\]
The conditional expectation is \(h_{t+1}(s^a)\) by the Markov property and deterministic update. Alternatively, the path density telescopes:
\[
\rho_\beta(S_0)\prod_{t=0}^{T-1}
\frac{q_0(A_t\mid S_t)h_{t+1}(S_{t+1})}{h_t(S_t)}
=\frac{e^{\beta R(X_T)}}{Z_\beta}\mathbb P_0(\tau).
\]
This is the Doob \(h\)-transform construction; see also \citet{bu2026distributionalbiasesposttrainingmarkovian}.
\end{proof}

\begin{lemma}[Bounded process reward]
\label{lem:app_bounded_process_reward}
If \(R(X_T)\in[0,1]\), then \(R_t^\star(a;s)\in[0,1]\) for every supported action at every reachable state.
\end{lemma}

\begin{proof}
The bound \(1\le e^{\beta R(X_T)}\le e^\beta\) is preserved by conditional expectation. Applying \(\beta^{-1}\log\) gives the result.
\end{proof}

\medskip
\noindent\textbf{Statewise marginal Soft-BoN approximation.}
Fix a state \(s\), and draw \(A_{t,1},\ldots,A_{t,N}\) independently from \(q_0(\cdot\mid s)\), with \(N\ge1\). Retain repeated occurrences. Conditional on the candidate multiset, select one occurrence with probability proportional to its \emph{exact} weight \(e^{\beta R_t^\star(A_{t,j};s)}\). The conditional action distribution is
\[
\widehat\pi_{t,N}(a\mid s;A_{t,1:N})
=\frac{\sum_{j=1}^N\mathbf1\{A_{t,j}=a\}e^{\beta R_t^\star(a;s)}}
{\sum_{j=1}^Ne^{\beta R_t^\star(A_{t,j};s)}}.
\]
The action distribution actually induced by drawing and selecting is
\[
Q_{t,N}(a\mid s)
:=\mathbb E_{A_{t,1:N}}[\widehat\pi_{t,N}(a\mid s;A_{t,1:N})].
\]
The following bound concerns \(Q_{t,N}\), which has the same support as \(q_0\). A conditional candidate distribution can assign zero mass to a target action, so its forward KL can be infinite even with a constant reward.

\begin{proposition}[Local marginal Soft-BoN KL bound]
\label{prop:app_local_softbon_rate}
For the preceding sampling-and-selection procedure on a finite action set, if \(R_t^\star(a;s)\in[0,1]\), then
\begin{equation}
\mathrm{KL}\!\left(\pi_t^\star(\cdot\mid s)\,\middle\|\,Q_{t,N}(\cdot\mid s)\right)
\le\frac{\sinh(\beta/2)^2}{N}.
\label{eq:app_local_softbon_rate}
\end{equation}
\end{proposition}

\begin{proof}
Suppress \(t,s\), and write
\[
q(a)=q_0(a\mid s),\quad w(a)=e^{\beta R_t^\star(a;s)},\quad
\mu=\mathbb E_q[w(A)],\quad p(a)=\frac{q(a)w(a)}\mu.
\]
Let \(W_j=w(A_j)\). By exchangeability and independence,
\[
Q_N(a)
=Nq(a)w(a)\,\mathbb E\!\left[\frac1{w(a)+\sum_{j=2}^NW_j}\right].
\]
Jensen's inequality for the reciprocal function yields
\[
Q_N(a)\ge\frac{Nq(a)w(a)}{w(a)+(N-1)\mu}>0
\qquad(q(a)>0).
\]
Consequently, using \(\log(1+x)\le x\) for \(x>-1\),
\begin{align*}
\mathrm{KL}(p\|Q_N)
&\le\sum_a p(a)\log\left(1+\frac{w(a)/\mu-1}{N}\right)\\
&\le\frac1N\left(\frac{\mathbb E_q[w(A)^2]}{\mu^2}-1\right)
=\frac{\operatorname{Var}_q(w(A))}{N\mu^2}.
\end{align*}
Put \(B=e^\beta\). Since \(1\le w(A)\le B\), the inequality \((w(A)-1)(B-w(A))\ge0\) implies
\[
\operatorname{Var}_q(w(A))\le(B-\mu)(\mu-1).
\]
Moreover,
\[
\frac{(B-\mu)(\mu-1)}{\mu^2}
\le\frac{(B-1)^2}{4B}
=\sinh(\beta/2)^2,
\]
where the inequality is equivalent to \(((B+1)\mu-2B)^2\ge0\). Combining these bounds proves the claim. At \(\beta=0\), the target and marginal sampling law both equal \(q\), so the zero bound also holds directly.
\end{proof}

\medskip
\noindent\textbf{Lifting the marginal policy to trajectories.}
Let \(N_t(s)\ge1\) be a deterministic state-dependent shortlist size. At every step, conditional on the current state, draw a fresh candidate multiset independently of the past, and then select an occurrence as above. Let \(\widehat{\mathbb P}_{\{N_t\}}\) denote the resulting trajectory law with \emph{all candidate draws integrated out}. Its action kernel is \(Q_{t,N_t(s)}(\cdot\mid s)\).

\begin{lemma}[KL chain rule for reveal trajectories]
\label{lem:app_path_kl_chain}
Suppose \(\mathbb P_\beta\) and \(\widehat{\mathbb P}_{\{N_t\}}\) have the same initial law \(\rho_\beta\) and the same deterministic state updates, and their action kernels are \(\pi_t^\star\) and \(Q_{t,N_t(s)}\). Then
\begin{equation}
\mathrm{KL}\!\left(\mathbb P_\beta\,\middle\|\,\widehat{\mathbb P}_{\{N_t\}}\right)
=\mathbb E_{\mathbb P_\beta}\!\left[
\sum_{t=0}^{T-1}\mathrm{KL}\!\left(
\pi_t^\star(\cdot\mid S_t)\,\middle\|\,Q_{t,N_t(S_t)}(\cdot\mid S_t)
\right)\right].
\label{eq:app_path_kl_chain}
\end{equation}
\end{lemma}

\begin{proof}
Fresh sampling makes the approximate action law conditional on the trajectory history equal to \(Q_{t,N_t(S_t)}\). Thus, on valid trajectories,
\[
\mathbb P_\beta(\tau)=\rho_\beta(S_0)\prod_{t=0}^{T-1}\pi_t^\star(A_t\mid S_t),
\qquad
\widehat{\mathbb P}_{\{N_t\}}(\tau)=\rho_\beta(S_0)\prod_{t=0}^{T-1}Q_{t,N_t(S_t)}(A_t\mid S_t).
\]
All supported target actions have positive marginal approximate probability. The log density ratio is therefore the sum of the log action ratios. Taking its expectation under \(\mathbb P_\beta\), then conditioning each summand on \(S_t\), proves \cref{eq:app_path_kl_chain}.
\end{proof}

\begin{proposition}[Pathwise lifting of the local marginal Soft-BoN bound]
\label{prop:app_pathwise_softbon}
Under the preceding assumptions,
\begin{equation}
\mathrm{KL}\!\left(\mathbb P_\beta\,\middle\|\,\widehat{\mathbb P}_{\{N_t\}}\right)
\le\sinh(\beta/2)^2\,
\mathbb E_{\mathbb P_\beta}\!\left[\sum_{t=0}^{T-1}\frac1{N_t(S_t)}\right].
\label{eq:app_path_kl_rate}
\end{equation}
In particular, for \(N_t(s)\equiv N\),
\begin{equation}
\mathrm{KL}\!\left(\mathbb P_\beta\,\middle\|\,\widehat{\mathbb P}_N\right)
\le\frac{T\,\sinh(\beta/2)^2}{N}.
\label{eq:app_path_kl_rate_uniform}
\end{equation}
\end{proposition}

\begin{proof}
Insert \cref{eq:app_local_softbon_rate} into the chain rule \cref{eq:app_path_kl_chain}. The uniform-size case follows from \(\sum_{t=0}^{T-1}1/N=T/N\).
\end{proof}

\begin{proof}[Proof of the terminal Soft-BoN approximation statement]
Let \(\nu_\beta\) and \(\widehat\nu_N\) be the terminal marginals of \(\mathbb P_\beta\) and \(\widehat{\mathbb P}_N\), respectively. The terminal output is a measurable function of the trajectory, so the data-processing inequality and \cref{eq:app_path_kl_rate_uniform} give
\[
\mathrm{KL}(\nu_\beta\|\widehat\nu_N)
\le\mathrm{KL}(\mathbb P_\beta\|\widehat{\mathbb P}_N)
\le\frac{T\,\sinh(\beta/2)^2}{N}.
\]
These are KL divergences of marginal distributions; no expectation over a fixed realized shortlist is needed.
\end{proof}

\medskip
\noindent\textbf{Initial-law correction and stochastic host updates.}
If the approximate process instead starts from \(\widetilde\rho\), the same proof adds \(\mathrm{KL}(\rho_\beta\|\widetilde\rho)\) to the trajectory upper bound (and hence to the terminal bound), provided \(\rho_\beta\ll\widetilde\rho\). In particular, using \(\rho_0\) for a globally tilted target generally incurs this term. If both policies instead start from any common law \(\rho\), the chain rule still holds for their induced processes, but the Gibbs reference is the global tilt only when \(\rho=\rho_\beta\).

For a stochastic host kernel \(K_t(ds'\mid s,a)\), define
\[
H_t(s,a)=\int h_{t+1}(s')K_t(ds'\mid s,a),
\qquad
R_t^\star(a;s)=\beta^{-1}\log H_t(s,a).
\]
The full Doob transform has both the Gibbs action kernel
\[
\pi_t^\star(a\mid s)=\frac{q_0(a\mid s)H_t(s,a)}{h_t(s)}
\]
and the conditional state-update kernel
\[
K_t^\beta(ds'\mid s,a)
=\frac{h_{t+1}(s')}{H_t(s,a)}K_t(ds'\mid s,a).
\]
The local marginal bound is unchanged because \(1\le H_t(s,a)\le e^\beta\). The trajectory and terminal bounds also remain valid when the exact and approximate processes use any \emph{common} update kernel: its density factors cancel in the chain rule. If that kernel is the original \(K_t\), the reference is the Gibbs-controlled host process, generally distinct from the globally tilted process using \(K_t^\beta\). To approximate the latter, both processes must use \(K_t^\beta\) and the matching initial law.

The bounds here isolate shortlist sampling error with exact reward weights. Estimated rewards, duplicate removal, deterministic top-\(m\) choices, and reuse of candidate sets change the sampling law and require separate error control; the online tracking results do not by themselves turn these variants into the exact-weight marginal sampler.

\subsection{Finite-sample forward-KL separations for Progressive Online DPRM}
\label{app:optimization_advantage}

This appendix studies finite-sample optimization under the common population target justified by the alignment and minimizer-preservation results.
Admissible token orders do not change the population minimizer, so the relevant issue is not asymptotic consistency but optimization complexity:
which aligned order family drives the training KL below a target level \(\varepsilon\) fastest?

We analyze this question in two stages.
In the early stage, confidence can reduce the gradient second moment exponentially; an update-complexity separation additionally requires a uniform persistent-noise lower condition. This view mirrors curriculum learning \cite{bu2025provable} and active learning \cite{bu2024provably}.
In the late stage, a separation requires two-sided residual-progress bounds and an actual coverage advantage. A full-Gibbs score gap and a finite-shortlist coverage guarantee play distinct roles.

The analysis uses three standard ingredients.
First, in unbiased importance-weighted SGD, the conditional second moment is minimized by sampling in proportion to gradient norm \citep{zhao2015importance,hanchi2022srg}.
Second, arbitrary-sampling SGD admits convergence recursions controlled by that second moment \citep{qian2019sgd,chen2023hetersgd}.
Third, confidence gating imposes a hard entropy cap and thus a support-narrowing effect \citep[Proposition~2]{fang2026locally}.

\paragraph{Stagewise forward-KL objective.}
Fix a training stage \(t\).
For each sample \(x\), let \(\mathcal O_t(x)\) be an aligned local order family.
This may consist of reveal choices, reveal prefixes, or any other local order objects that can be visited both by teacher-forced training and by the corresponding inference-time controller at stage \(t\).
Let \(u_t(\cdot\mid x)\) be a reference aligned sampling law on \(\mathcal O_t(x)\).

Let \(p_t^\star(\cdot\mid x,o)\) be the target conditional at stage \(t\) and let \(p_\theta(\cdot\mid x,o)\) be the model conditional.
Define the stagewise forward-KL objective
\begin{equation}
\mathcal K_t(\theta)
:=
\mathbb E_{x\sim \widehat P_n}
\mathbb E_{o\sim u_t(\cdot\mid x)}
\Big[
\mathrm{KL}\bigl(p_t^\star(\cdot\mid x,o)\,\|\,p_\theta(\cdot\mid x,o)\bigr)
\Big].
\label{eq:app_stage_kl}
\end{equation}
Equivalently, \(\mathcal K_t(\theta)\) is the excess stagewise cross-entropy up to an additive constant independent of \(\theta\).

Let \(\ell_t(\theta;x,o,\xi)\) be a stochastic per-order loss whose conditional expectation equals the integrand of \cref{eq:app_stage_kl}, and define
\[
g_t(\theta;x,o,\xi):=\nabla_\theta \ell_t(\theta;x,o,\xi),
\qquad
\nu_t(\theta;x,o):=
\Big(\mathbb E_\xi \|g_t(\theta;x,o,\xi)\|_2^2\Big)^{1/2}.
\]
All proposals must assign positive mass wherever \(u_t(o\mid x)\nu_t(\theta_t;x,o)>0\). Zero-gradient terms contribute zero; if every such term vanishes, any proposal may be used. If \(O_t\sim p_t(\cdot\mid X_t)\), the unbiased importance-weighted gradient estimator is
\begin{equation}
\widehat G_t^{(p)}
:=
\frac{u_t(O_t\mid X_t)}{p_t(O_t\mid X_t)}
\,g_t(\theta_t;X_t,O_t,\xi_t),
\label{eq:app_is_gradient}
\end{equation}
with conditional second moment
\begin{equation}
\mathcal M_t(p)
:=
\mathbb E\!\left[\|\widehat G_t^{(p)}\|_2^2\mid \theta_t\right].
\label{eq:app_stage_second_moment}
\end{equation}

\begin{assumption}[Stagewise smoothness and PL geometry]
\label{ass:app_stage_pl}
At stage \(t\), the forward-KL objective \(\mathcal K_t\) is \(L_t\)-smooth and satisfies the Polyak--\L ojasiewicz inequality with constant \(\mu_t>0\):
\[
\frac{1}{2}\|\nabla \mathcal K_t(\theta)\|_2^2
\ge
\mu_t\bigl(\mathcal K_t(\theta)-\mathcal K_t^\star\bigr)
\qquad
\text{for all relevant }\theta.
\]
\end{assumption}

\begin{assumption}[Confidence as a proxy for local gradient scale]
\label{ass:app_confidence_proxy}
There exists a confidence-induced order score \(c_t(x,o)\ge 0\) and a constant \(\kappa_t\ge 1\) such that
\[
\kappa_t^{-1}\nu_t(\theta_t;x,o)
\le
c_t(x,o)
\le
\kappa_t \nu_t(\theta_t;x,o)
\qquad
\text{for all }x,o.
\]
\end{assumption}

\paragraph{Early stage: confidence beats random.}

\begin{assumption}[Early importance concentration]
\label{ass:app_early_concentration}
For each sample \(x\), there exists a subset \(\mathcal E_t(x)\subseteq \mathcal O_t(x)\) such that:
\begin{enumerate}[leftmargin=18pt]
    \item \textbf{Most gradient mass lies on \(\mathcal E_t(x)\):}
    \[
    \sum_{o\in \mathcal E_t(x)}
    u_t(o\mid x)\,\nu_t(\theta_t;x,o)
    \ge
    (1-\rho_t)
    \sum_{o\in \mathcal O_t(x)}
    u_t(o\mid x)\,\nu_t(\theta_t;x,o)
    \]
    for some \(\rho_t\in[0,1)\);
    \item \textbf{The easy family is exponentially smaller than the full aligned family:}
    \[
    \frac{|\mathcal O_t(x)|}{|\mathcal E_t(x)|}\ge e^{a_t d_t}
    \]
    for some stage-difficulty parameter \(d_t\ge 1\) and constant \(a_t>0\).
\end{enumerate}
\end{assumption}

\begin{lemma}[Variance-optimal proposal over aligned order families]
\label{lem:app_variance_optimal}
For fixed \(\theta_t\), the proposal minimizing \(\mathcal M_t(p)\) satisfies
\begin{equation}
p_t^{\mathrm{opt}}(o\mid x)
\propto
u_t(o\mid x)\,\nu_t(\theta_t;x,o).
\label{eq:app_optimal_proposal}
\end{equation}
\end{lemma}

\begin{proof}
Condition on \(x\).
By \cref{eq:app_is_gradient,eq:app_stage_second_moment},
\[
\mathcal M_t(p\mid x)
=
\sum_{o\in\mathcal O_t(x)}
\frac{u_t(o\mid x)^2}{p_t(o\mid x)}
\,\nu_t(\theta_t;x,o)^2.
\]
Hence, for each fixed \(x\), the problem is
\[
\min_{p(\cdot\mid x)\in\Delta(\mathcal O_t(x))}
\sum_{o\in\mathcal O_t(x)}
\frac{u_t(o\mid x)^2\nu_t(\theta_t;x,o)^2}{p(o\mid x)}.
\]
This is exactly the conditional-variance minimization problem of importance-sampling SGD \citep{zhao2015importance,hanchi2022srg}.
By Cauchy--Schwarz or a Lagrange multiplier, the minimizer is
\[
p_t^{\mathrm{opt}}(o\mid x)
=
\frac{u_t(o\mid x)\nu_t(\theta_t;x,o)}
{\sum_{o'}u_t(o'\mid x)\nu_t(\theta_t;x,o')}.
\]
\end{proof}

Define the confidence proposal
\begin{equation}
p_t^{\mathrm{conf}}(o\mid x)
:=
\frac{u_t(o\mid x)c_t(x,o)}
{\sum_{o'}u_t(o'\mid x)c_t(x,o')},
\label{eq:app_confidence_proposal}
\end{equation}
and the random aligned proposal
\begin{equation}
p_t^{\mathrm{rand}}(o\mid x):=\frac{1}{|\mathcal O_t(x)|}.
\label{eq:app_random_proposal}
\end{equation}

\begin{lemma}[Second-moment comparison]
\label{lem:app_second_moment_comparison}
Under Assumptions~\ref{ass:app_confidence_proxy} and \ref{ass:app_early_concentration},
\begin{equation}
\mathcal M_t\!\bigl(p_t^{\mathrm{conf}}\bigr)
\le
\kappa_t^2\,\mathcal M_t\!\bigl(p_t^{\mathrm{opt}}\bigr),
\label{eq:app_confidence_near_opt}
\end{equation}
and
\begin{equation}
\mathcal M_t\!\bigl(p_t^{\mathrm{rand}}\bigr)
\ge
(1-\rho_t)^2 e^{a_t d_t}\,
\mathcal M_t\!\bigl(p_t^{\mathrm{opt}}\bigr).
\label{eq:app_random_worse_opt}
\end{equation}
\end{lemma}

\begin{proof}
The first bound follows from Assumption~\ref{ass:app_confidence_proxy}: the proposal \(p_t^{\mathrm{conf}}\) is a \(\kappa_t^2\)-multiplicative approximation to \(p_t^{\mathrm{opt}}\), so substituting into \cref{eq:app_stage_second_moment} yields \cref{eq:app_confidence_near_opt}.

For the second bound, condition on \(x\) and write
\[
a_o:=u_t(o\mid x)\nu_t(\theta_t;x,o).
\]
Then
\[
\mathcal M_t(p_t^{\mathrm{rand}}\mid x)
=
|\mathcal O_t(x)|\sum_{o\in\mathcal O_t(x)} a_o^2
\ge
|\mathcal O_t(x)|\sum_{o\in \mathcal E_t(x)} a_o^2.
\]
By Cauchy--Schwarz,
\[
\sum_{o\in \mathcal E_t(x)} a_o^2
\ge
\frac{1}{|\mathcal E_t(x)|}
\left(\sum_{o\in \mathcal E_t(x)} a_o\right)^2.
\]
Using Assumption~\ref{ass:app_early_concentration},
\[
\sum_{o\in \mathcal E_t(x)} a_o
\ge
(1-\rho_t)\sum_{o\in \mathcal O_t(x)} a_o.
\]
Therefore
\[
\mathcal M_t(p_t^{\mathrm{rand}}\mid x)
\ge
\frac{|\mathcal O_t(x)|}{|\mathcal E_t(x)|}
(1-\rho_t)^2
\left(\sum_{o\in\mathcal O_t(x)} a_o\right)^2.
\]
But
\[
\left(\sum_{o\in\mathcal O_t(x)} a_o\right)^2
=
\mathcal M_t(p_t^{\mathrm{opt}}\mid x),
\]
and \(|\mathcal O_t(x)|/|\mathcal E_t(x)|\ge e^{a_t d_t}\).
Averaging over \(x\) proves \cref{eq:app_random_worse_opt}.
\end{proof}

Define the \(\varepsilon\)-sample complexity at stage \(t\) by
\[
T_t^{(p)}(\varepsilon)
:=
\min\Bigl\{N:\mathbb E[\mathcal K_t(\theta_N)-\mathcal K_t^\star]\le \varepsilon\Bigr\}.
\]

\begin{theorem}[Early-stage separation under uniform noise conditions]
\label{thm:app_early_sample_complexity_separation}
Fix a stage \(t\), and suppose Assumptions~\ref{ass:app_stage_pl}, \ref{ass:app_confidence_proxy}, and \ref{ass:app_early_concentration} hold throughout its updates. Write
\(e_k=\mathbb E[\mathcal K_t(\theta_k)-\mathcal K_t^\star]\),
\(A_t=(1-\rho_t)^2e^{a_t d_t}\), and let \(M_*>0\) be a fixed stage scale, independent of \(k\).
Assume \(\mathcal M_t(p_t^{\rm opt})\le M_*\) uniformly along confidence-driven iterates. In addition to the smoothness upper recursion, assume the random-order oracle obeys the following lower recursion, for every allowed constant stepsize \(0<\eta\le\eta_{\max}\):
\begin{equation}
 e_{k+1}\ge (1-c_1\mu_t\eta)e_k
       +c_0 L_t\eta^2 A_t M_*,
 \qquad \eta_{\max}\le\frac{1}{2c_1\mu_t},
 \label{eq:app_early_lower_recursion}
\end{equation}
where \(c_0,c_1>0\) are constants independent of \(d_t\), \(k\), and \(\varepsilon\).
This condition imposes persistent sampling noise and limits how fast the random-order method can contract; it is additional to the second-moment comparison.

For \(0<\varepsilon<e_0/2\), a confidence stepsize gives
\begin{equation}
T_t^{(\mathrm{conf})}(\varepsilon)
=O\!\left[
 \left(\frac{L_t}{\mu_t}
 +\frac{L_t\kappa_t^2M_*}{\mu_t^2\varepsilon}\right)
 \log\frac{2e_0}{\varepsilon}\right].
\label{eq:app_confidence_sc}
\end{equation}
Every random-order constant-stepsize method in the stated class with
\(\limsup_{k\to\infty}e_k\le\varepsilon/2\) satisfies
\begin{equation}
T_t^{(\mathrm{rand})}(\varepsilon)
=\Omega\!\left(
 \frac{L_t(1-\rho_t)^2e^{a_t d_t}M_*}{\mu_t^2\varepsilon}
 \log\frac{e_0}{\varepsilon}\right).
\label{eq:app_random_sc}
\end{equation}
Thus the update complexity separates exponentially in \(d_t\) in the noise-dominated regime when the other scales, including \(\kappa_t\) and \(1/(1-\rho_t)\), remain controlled.
\end{theorem}

\begin{proof}
Smoothness, unbiasedness, and the PL inequality imply, for a sufficiently small confidence stepsize,
\[
e_{k+1}\le(1-\mu_t\eta)e_k
       +\frac{L_t\eta^2}{2}\kappa_t^2M_*.
\]
Choose
\[
\eta=\min\!\left\{\frac{1}{2L_t},
 \frac{\mu_t\varepsilon}{2L_t\kappa_t^2M_*}\right\}.
\]
The limiting upper term is at most \(\varepsilon/4\), and unrolling gives
\cref{eq:app_confidence_sc}. The moment comparison is used uniformly here, rather than evaluated at just one iterate.

For random-order SGD, unrolling \cref{eq:app_early_lower_recursion} gives
\[
e_k\ge(1-c_1\mu_t\eta)^ke_0+
 \frac{c_0L_t\eta A_tM_*}{c_1\mu_t}
 \bigl[1-(1-c_1\mu_t\eta)^k\bigr].
\]
The assumed asymptotic accuracy therefore requires
\[
\eta\le\frac{c_1\mu_t\varepsilon}{2c_0L_t A_tM_*}.
\]
Also, \(e_k\ge(1-c_1\mu_t\eta)^ke_0\). Since
\(c_1\mu_t\eta\le1/2\) and \(-\log(1-x)\le2x\) on \([0,1/2]\), reaching \(e_k\le\varepsilon\) requires
\[
k\ge\frac{\log(e_0/\varepsilon)}{-\log(1-c_1\mu_t\eta)}
 \ge\frac{c_0L_tA_tM_*}{c_1^2\mu_t^2\varepsilon}
       \log\frac{e_0}{\varepsilon}.
\]
This proves \cref{eq:app_random_sc} from a lower recursion, independently of the sufficient upper bound.
\end{proof}

\paragraph{A nondegenerate oracle satisfying the extra condition.}
Consider \(K\) equally weighted orders, only one of which has a nonzero stochastic gradient. At that order set
\(g(\theta)=K(\lambda\theta+\xi)\), where independent
\(\xi\in\{-\sigma,\sigma\}\) are equiprobable; elsewhere set \(g=0\).
The population objective is \(\lambda\theta^2/2\), with \(L_t=\mu_t=\lambda\). The optimal confidence proposal selects the active order, and its importance-weighted gradient is \(\lambda\theta+\xi\). Thus \(\rho_t=0\), \(\kappa_t=1\), and \(A_t=K\). If \(|\theta_0|\le\sigma/\lambda\), stable confidence updates stay in this interval and have \(\mathcal M_t(p_t^{\rm opt})\le M_*:=2\sigma^2\). Random sampling instead obeys exactly
\[
e_{k+1}=(1-2\lambda\eta+K\lambda^2\eta^2)e_k
       +\frac{\lambda\eta^2K\sigma^2}{2}.
\]
Consequently \cref{eq:app_early_lower_recursion} holds with
\(c_1=2\), \(c_0=1/4\), and \(\eta_{\max}=1/(4\lambda)\).
This example explains the additional noise condition; smoothness, PL geometry, and large second moments alone do not imply a sample-complexity lower bound.

\paragraph{Late stage: DPRM beats confidence-only.}

We now analyze the stage after warmup.
Here the issue is no longer variance-optimality, but undercoverage.
We assume that there remains a residual family of orders that is necessary for final KL convergence, but confidence-only training visits it too rarely.

We first restate the entropy-cap result of \citet[Proposition~2]{fang2026locally}.

\begin{proposition}[Entropy cap under confidence gating {\normalfont(\citealp[Prop.~2]{fang2026locally})}]
\label{prop:app_entropy_cap}
Assume the decoder is \((1-\delta)\)-gated.
Then the induced sequence distribution \(X\) satisfies
\begin{equation}
H(X)\le L h_V(\delta),
\qquad
B_{\mathrm{eff}}:=\exp(H(X)/L)\le \exp(h_V(\delta)),
\label{eq:app_entropy_cap}
\end{equation}
where \(h_V(\delta)=h_b(\delta)+\delta\log(|V|-1)\).
\end{proposition}

\begin{assumption}[Late-stage residual coverage and two-sided progress]
\label{ass:app_late_residual}
Fix a late training stage and write \(k\) for its update index. At each reachable state \(s\), let \(\varnothing\ne\mathcal R_t(s)\subsetneq\mathcal O_t(s)\) be a residual order family.
\begin{enumerate}[leftmargin=18pt]
\item \textbf{Residual score gap.} The exact full-candidate score gap
\[
\Delta_t^{\rm res}(s):=
 \inf_{o\in\mathcal R_t(s)}g_t^\star(o;s)
 -\sup_{o\notin\mathcal R_t(s)}g_t^\star(o;s)
\]
is positive.
\item \textbf{Confidence undercoverage.} Conditionally on every update history, the confidence controller samples the residual family with probability at most
\(\bar p_t:=C_t e^{-b_t h_t}<1\), where \(C_t\ge1\), \(b_t>0\), and \(h_t\ge1\).
\item \textbf{Two-sided residual progress.} Let \(K_k\ge0\) denote the residual contribution to the fixed stage objective, and let \(\mathcal H_k\) contain the state and all information before update \(k\). Uniformly for either controller there are
\(0<\gamma_-\le\gamma_+\le1\) such that
\[
(1-\gamma_+)K_k
 \le\mathbb E[K_{k+1}\mid\mathcal H_k,\ O_k\in\mathcal R_t]
 \le(1-\gamma_-)K_k,
\]
and
\[
\mathbb E[K_{k+1}\mid\mathcal H_k,\ O_k\notin\mathcal R_t]=K_k.
\]
Thus updates outside the residual family make no expected residual progress, and a hit has both a minimum and a maximum contraction. Assume \(\gamma_+\bar p_t\le1/2\) when invoking the exponential-order lower bound.
\end{enumerate}
\end{assumption}

\begin{theorem}[Late-stage separation with actual residual coverage]
\label{thm:app_late_sample_complexity_separation}
Under Assumption~\ref{ass:app_late_residual}, suppose the full-candidate estimated scores satisfy
\[
\epsilon_t(s):=\sup_{o\in\mathcal O_t(s)}
 |\widehat g_t(o;s)-g_t^\star(o;s)|<\Delta_t^{\rm res}(s)/2.
\]
The full estimated Gibbs policy
\(\widetilde\pi_t(o\mid s)\propto\exp\widehat g_t(o;s)\) then satisfies
\begin{equation}
\widetilde\pi_t(\mathcal R_t(s)\mid s)
\ge\frac{|\mathcal R_t(s)|e^{\Delta_t^{\rm res}(s)-2\epsilon_t(s)}}
 {|\mathcal R_t(s)|e^{\Delta_t^{\rm res}(s)-2\epsilon_t(s)}
 +|\mathcal O_t(s)\setminus\mathcal R_t(s)|}.
\label{eq:app_residual_mass_lb_new}
\end{equation}
If uniformly in \(s\), for some \(c_t>0\),
\[
\Delta_t^{\rm res}(s)-2\epsilon_t(s)
\ge\log\frac{|\mathcal O_t(s)\setminus\mathcal R_t(s)|}
 {|\mathcal R_t(s)|}+c_t,
\]
then
\begin{equation}
\widetilde\pi_t(\mathcal R_t(s)\mid s)
 \ge p_0:=\frac{1}{1+e^{-c_t}}.
\label{eq:app_constant_residual_mass}
\end{equation}
For the optimization conclusion, assume the \emph{actual} DPRM sampler has conditional residual mass at least a constant \(p_D>0\) at every update. Full Gibbs sampling permits \(p_D=p_0\); a shortlist sampler requires its own coverage guarantee. For a common initial residual error \(K_0>\varepsilon>0\), the update complexities obey
\begin{equation}
\begin{split}
T_{t,\rm late}^{(\rm conf)}(\varepsilon)
&\ge\left\lceil\frac{\log(K_0/\varepsilon)}
 {-\log(1-\gamma_+\bar p_t)}\right\rceil\\
&=\Omega\!\left(\frac{e^{b_t h_t}}{\gamma_+C_t}
             \log\frac{K_0}{\varepsilon}\right),
\end{split}
\label{eq:app_conf_late_sc}
\end{equation}
and
\begin{equation}
T_{t,\rm late}^{(\rm DPRM)}(\varepsilon)
\le\left\lceil\frac{1}{\gamma_-p_D}
                 \log\frac{K_0}{\varepsilon}\right\rceil.
\label{eq:app_dprm_late_sc}
\end{equation}
These bounds count parameter updates. An exponential separation requires \(\gamma_-/\gamma_+\), \(p_D\), and the other comparison constants to remain controlled as \(h_t\) grows.
\end{theorem}

\begin{proof}
The uniform score bound implies
\(\widehat g_t(o;s)-\widehat g_t(o';s)
\ge\Delta_t^{\rm res}(s)-2\epsilon_t(s)\)
for \(o\in\mathcal R_t(s)\), \(o'\notin\mathcal R_t(s)\).
Summing the corresponding exponentials proves
\cref{eq:app_residual_mass_lb_new,eq:app_constant_residual_mass}.

If the current conditional probability of a hit is \(p_k\), the two-sided progress assumption gives
\[
(1-\gamma_+p_k)K_k
 \le\mathbb E[K_{k+1}\mid\mathcal H_k]
 \le(1-\gamma_-p_k)K_k.
\]
For confidence, \(p_k\le\bar p_t\), so iterated conditioning yields
\(\mathbb E K_k\ge(1-\gamma_+\bar p_t)^kK_0\).
This proves the exact necessary logarithmic bound in
\cref{eq:app_conf_late_sc}; using \(-\log(1-x)\le2x\) for \(x\le1/2\) gives its stated order.
For DPRM, \(p_k\ge p_D\), hence
\(\mathbb E K_k\le(1-\gamma_-p_D)^kK_0
\le e^{-\gamma_-p_D k}K_0\), proving the sufficient bound
\cref{eq:app_dprm_late_sc}.
\end{proof}

\begin{corollary}[Calibrated stage-2 separation with shortlist coverage]
\label{cor:app_practical_late_separation}
Assume Assumption~\ref{ass:app_late_residual}, including confidence undercoverage and the two-sided progress bounds.
Condition on a frozen score table calibrated independently of the late-stage update stream. Suppose the online score-tracking theorem ensures, with calibration probability at least \(1-\delta\), a uniform bound
\[
\epsilon_t(s)=O\!\left(
 \beta\eta_t(s)\sqrt{\frac{\log(\mathcal N/\delta)}{N_{\min,t}(s)}}
 +\beta\eta_t(s)\frac{\log(\mathcal N/\delta)}{N_{\min,t}(s)}
 +\operatorname{Bias}_t(s;\delta)\right),
\]
small enough to meet both the half-gap condition and the stronger cardinality-adjusted gap giving \(p_0\) in
\cref{eq:app_constant_residual_mass}.
At each update, draw a fresh i.i.d. multiset of \(N_t(s)\) candidates from \(q_0\), retain duplicates, and use exact weights for a bounded surrogate reward \(\widehat r_t(o;s)\in[0,1]\), so that
\(\widehat g_t(o;s)=\log q_0(o\mid s)+\beta\widehat r_t(o;s)\)
up to an action-independent constant.
For a chosen constant \(0<p_D<p_0\), assume
\begin{equation}
N_t(s)\ge
\frac{\sinh^2(\beta/2)}{2(p_0-p_D)^2}
\quad\text{at every relevant state.}
\label{eq:app_shortlist_residual_coverage}
\end{equation}
Then, with calibration probability at least \(1-\delta\), the conditional update-complexity bounds
\cref{eq:app_conf_late_sc,eq:app_dprm_late_sc} hold.
\end{corollary}

\begin{proof}
For any fixed calibrated table, Proposition~\ref{prop:app_local_softbon_rate} applies to the bounded surrogate reward itself: the sampler uses its exact exponentiated weights even though the surrogate need not equal the true process reward. If \(\widehat Q_{t,N}\) is this sampler's marginal action law, Pinsker's inequality gives
\[
\widehat Q_{t,N}(\mathcal R_t(s)\mid s)
\ge\widetilde\pi_t(\mathcal R_t(s)\mid s)
 -\sqrt{\frac{\sinh^2(\beta/2)}{2N}}
\ge p_D.
\]
The calibrated uniform score bound supplies the full-Gibbs mass \(p_0\), while \cref{eq:app_shortlist_residual_coverage} supplies actual shortlist coverage. The independent calibration allows the optimization assumptions to be applied conditionally on each successful table. Theorem~\ref{thm:app_late_sample_complexity_separation} then yields the result. A controller that continues updating its table must instead verify these conditional assumptions throughout its update history.
\end{proof}

\begin{remark}[Coverage and computational scope]
An order can only be selected from a shortlist if it was proposed. Consequently any duplicate-preserving \(N\)-candidate sampler satisfies
\[
Q_N(\mathcal R\mid s)
 \le1-(1-q_0(\mathcal R\mid s))^N
 \le Nq_0(\mathcal R\mid s).
\]
If \(q_0\) itself has exponentially small residual mass, constant actual residual coverage requires an exponentially large shortlist, or a proposal with better coverage. Similarly, a reward bounded in \([0,1]\) satisfies
\(\widetilde\pi_t(\mathcal R\mid s)\le e^\beta q_0(\mathcal R\mid s)\);
a fixed \(\beta\) cannot by itself compensate for arbitrarily severe exponential undercoverage. The late-stage separation is therefore a conditional statement about parameter updates, with sampling and calibration costs accounted for separately. Accurate scores alone do not guarantee residual coverage. The early result likewise needs persistent-noise and contraction conditions in addition to importance concentration. These results describe regimes in which ordering can improve finite-sample learning without changing the population minimizer; they do not assert a universal speedup for every deployed controller.
\end{remark}




\ifdefined\DPRMEmbeddedSupplement
\else
\bibliographystyle{IEEEtranN}
\bibliography{ref}
\end{document}
\fi

\fi

\end{document}